\documentclass[journal]{IEEEtran}

\def\REBARevisionHighlight{0}

\usepackage[T1]{fontenc}
\usepackage[utf8]{inputenc}
\usepackage{threeparttable}
\usepackage{amsmath, amssymb, amsthm}
\usepackage{graphicx}
\usepackage{subcaption}
\usepackage{booktabs}   
\usepackage{multirow}   
\usepackage{makecell}   
\usepackage{array}      
\usepackage{stfloats}   
\usepackage{placeins}   
\usepackage{algorithm}
\usepackage{algpseudocode}
\usepackage{textcomp}
\usepackage{url}
\usepackage{verbatim}
\usepackage{comment}
\usepackage{balance}    
\usepackage{cite}
\usepackage{enumitem}   
\usepackage{orcidlink}   

\usepackage{xcolor}
\providecommand{\REBARevisionHighlight}{0}
\definecolor{revRefTwoBlue}{RGB}{25,82,150}
\DeclareRobustCommand{\REBARevisionText}[1]{%
  \ifnum\REBARevisionHighlight=1\relax
    \textcolor{revRefTwoBlue}{#1}%
  \else
    #1%
  \fi
}
\DeclareRobustCommand{\revRefTwo}[1]{\REBARevisionText{#1}}
\DeclareRobustCommand{\revRefTwoMath}[1]{\REBARevisionText{#1}}
\DeclareRobustCommand{\revRefThree}[1]{\REBARevisionText{#1}}

\DeclareRobustCommand{\revRefFour}[1]{\REBARevisionText{#1}}

\newcommand{\revRc}[1]{\revRefFour{#1}}

\newcommand{\bbel}{\boldsymbol{b}}

\usepackage{xr-hyper}
\usepackage{hyperref}
\def\BibTeX{{\rm B\kern-.05em{\sc i\kern-.025em b}\kern-.08em
    T\kern-.1667em\lower.7ex\hbox{E}\kern-.125emX}}

\newtheorem{theorem}{Theorem}
\newtheorem{lemma}{Lemma}
\newtheorem{assumption}{Assumption}
\newtheorem{definition}{Definition}
\newtheorem{remark}{Remark}
\newtheorem{corollary}{Corollary}
\newtheorem{proposition}{Proposition}
\theoremstyle{definition}

\algrenewcommand\algorithmiccomment[1]{\hfill{\footnotesize$\triangleright$~#1}}
\algnewcommand\algorithmicglobal{\textbf{Global:}}
\algnewcommand\Global{\item[\algorithmicglobal]}

\newcommand{\Func}[1]{\textsc{#1}}

\hypersetup{hidelinks}
\usepackage{docmute}

\let\REBASavedBibliographystyle\bibliographystyle
\let\REBASavedBibliography\bibliography
\makeatletter
\let\bibliographystyle\@gobble
\let\bibliography\@gobble
\makeatother

\begin{document}

\title{REBA: A Revealed Belief Automaton Framework for Online Planning in Continuous POMDPs}
\author{Xiangwei~Chen\orcidlink{0009-0006-9159-9785},
        Lingling~Fang\orcidlink{0000-0002-4397-7212},
        Andreas~Holzinger,~\IEEEmembership{Senior Member,~IEEE}\orcidlink{0000-0002-6786-5194},
        Liming~Chen,~\IEEEmembership{Senior Member,~IEEE}\orcidlink{0000-0003-0200-7989}%
\thanks{This work is partially supported by the Dalian Science and Technology Talent Innovation Support Program [No. 2024RJ011] and the Liaoning Revitalization Talent Program [No. XLYC 2403024]. (\textit{Corresponding author: Liming Chen; Lingling Fang.})}
\IEEEcompsocitemizethanks{
    \IEEEcompsocthanksitem Xiangwei Chen is with the School of Computer Science and Technology, Dalian University of Technology, Dalian 116024, China (e-mail: xiangweichen@mail.dlut.edu.cn).
    \IEEEcompsocthanksitem Lingling Fang is with the School of Computer Science and Artificial Intelligence, Liaoning Normal University, Dalian 116029, China (e-mail: fanglingling@lnnu.edu.cn).
    \IEEEcompsocthanksitem Andreas Holzinger is with the Human-Centered AI Lab, Department of Ecosystem Management, Climate and Biodiversity, BOKU University, Vienna, 1190, Austria (e-mail: andreas.holzinger@boku.ac.at), and with the Institute of Human-Centered Computing, Graz University of Technology, Austria
    \IEEEcompsocthanksitem Liming Chen is with the School of Computer Science and Technology, Dalian University of Technology, Dalian 116024, China (e-mail: limingchen0922@dlut.edu.cn).
}
}

\markboth{}{}

\maketitle

\begin{abstract}
Online planning in continuous partially observable Markov decision processes (POMDPs) using $\omega$-regular specifications requires handling continuous belief dynamics within the finite symbolic memory in order to track temporal progress. Existing methods based on either direct search in belief space or predefined discrete abstractions suffer from drawbacks, e.g., lack of symbolic memory for long-horizon logical progress or difficult to certify from noisy online beliefs. As such, obtaining reliable symbolic states online from continuous observations remains a challenge. To address this issue, we introduce the Revealed Belief Automaton (REBA), an event-driven framework that advances the research from global belief-space discretization to a fundamental new way of thinking, namely online certification of revelation events. \revRefFour{Specifically, we propose an online revelation method that, through information-theoretic gates, can dynamically analyse and establish belief abstraction from the continuous belief space by discovering reliable anchors among noisy beliefs. We then develop an incremental topology adaptation mechanism over the certified anchors to realise the online finite Belief Automaton. By combining with the $\omega$-regular specification, REBA is able to support formal parity policy synthesis without a predefined discrete abstraction, which in turn can guide the Monte Carlo Tree Search process to perform online search beyond its local horizon. In addition, we design an error decomposition analysis which can assess the effectiveness and reliability of this discrete guidance for the underlying continuous POMDP. Empirical evaluations in patrolling and navigation scenarios show that REBA matches or exceeds all evaluated baselines, with primary metric gains of +17.0\% to +47.4\% over state-of-the-art approaches.}
\end{abstract}


\section{Introduction}
Partially observable Markov decision processes (POMDPs) provide a rigorous framework for sequential decision making under uncertainty~\cite{Kaelbling_1998} and have become a foundation for long-horizon planning in robotics and autonomous driving~\cite{Lauri_2023,Hsu_2024,Frering:2025:HumanRobot,Holzinger:2024:Forestry50}. The lack of direct state observability, however, forces the agent to condition on an unbounded observation history, breaking the Markov property and rendering naive planning computationally intractable.

To recover the Markov property, the standard remedy is to track a belief state, i.e., a posterior over latent states~\cite{lovejoy1991}, which compresses the history into a sufficient statistic and lifts the POMDP to a fully observable belief MDP over a continuous belief space.

\revRefFour{Planning under $\omega$-regular specifications must now reconcile two representations that do not naturally coexist: the lifted belief state is continuous, whereas tracking temporal progress against the specification requires finite symbolic memory. Online solvers address the first half of this picture: Monte Carlo Tree Search (MCTS) with progressive widening samples reachable futures directly in belief space~\cite{sunberg2018online,silver_2010}, and recent long-horizon variants couple the search with learned policy and value approximations~\cite{Moss_BetaZero_2024}. Yet these simulated futures carry no symbolic memory, so the search remains local with respect to the specification: a locally attractive action can improve the short-horizon value estimate while abandoning the sequence of visits and safety checks that the objective requires. Cyclic tasks make this failure mode visible, as Fig.~\ref{fig1} illustrates.}

\begin{center}
\begin{minipage}{\linewidth}
\centering
\includegraphics[width=\linewidth]{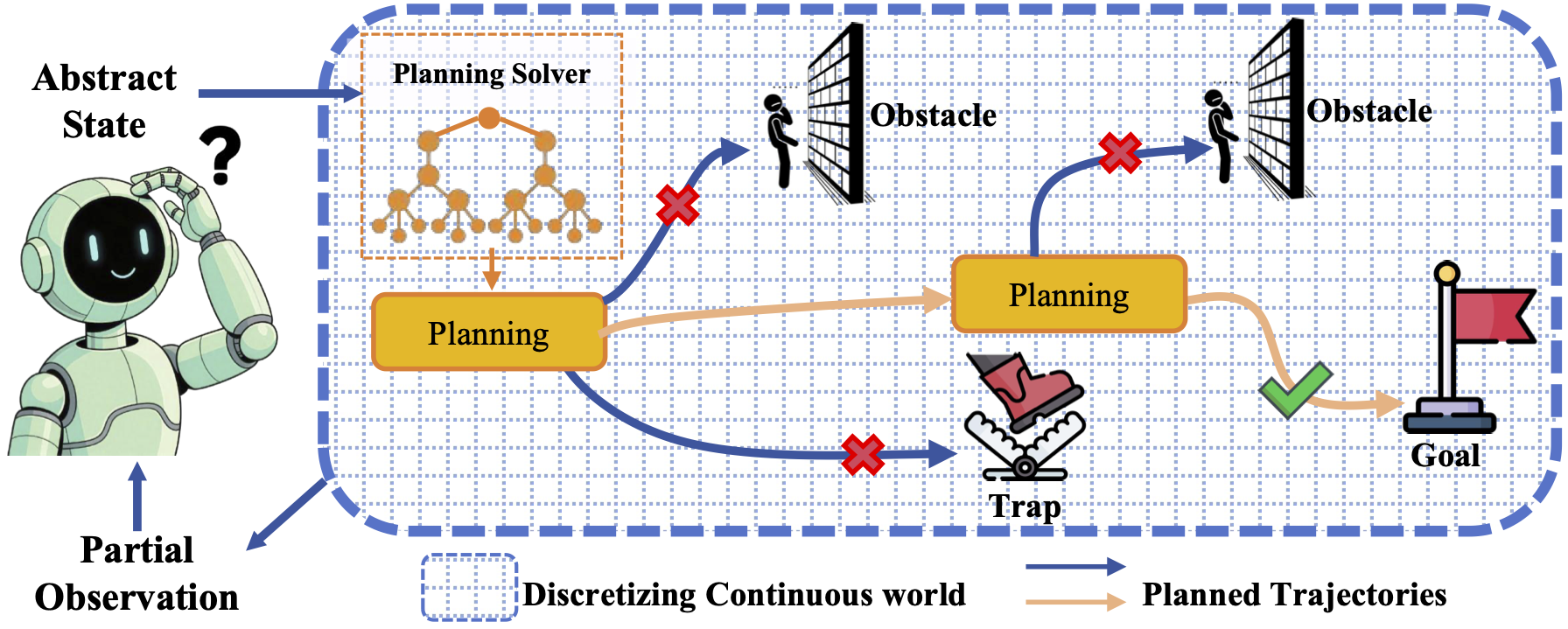}
\captionof{figure}{Local search over beliefs can favor the blue path, which looks efficient over a short horizon but leads toward a trap. The Revealed Belief Automaton adds symbolic guidance so Monte Carlo Tree Search can prefer the orange path, whose value depends on longer horizon progress and safety.}
\label{fig1}
\end{minipage}
\end{center}

\revRefFour{The symbolic half of the picture is harder still. Under partial observability, verification must reason over an uncountable belief space, and policy synthesis for general $\omega$-regular specifications~\cite{Giacomo2013,Wang_SafeTRL_2024} in POMDPs is undecidable~\cite{CHATTERJEE}. Recent revelation theory identifies a decidable route through revelation moments, i.e., times when observations make the latent state sufficiently identifiable~\cite{Belly_2025}, with quantitative analyses following for revealing subclasses~\cite{Asadi_RevealingPOMDPs_2026}. These results, however, presuppose a known discrete model and perfect, offline revelations---precisely the structure that is unavailable when beliefs are maintained online by particle-filtered approximations. Existing methods therefore either search directly in belief space without symbolic memory or depend on a predefined discrete abstraction, leaving open how reliable symbolic states can be obtained online from continuous observations.}

This paper bridges this gap with an event-driven principle: instead of discretizing the belief space globally, the Revealed Belief Automaton (REBA) certifies online the moments at which discretization is locally warranted. Information-theoretic gates admit a belief into the symbolic layer only when it qualifies as a revelation event, reframing belief abstraction from covering the continuous belief space into certifying when noisy beliefs yield reliable anchors. The certified anchors grow a finite Belief Automaton online; composing it with the $\omega$-regular specification enables parity policy synthesis without a predefined discrete abstraction; and the synthesized policy is fed back to steer MCTS beyond its local horizon, under explicit error accounting. Our main contributions are as follows:
\begin{enumerate}
    \item \revRefFour{We formulate online revelation for continuous-belief POMDPs, turning the abstraction problem from covering the full continuous belief space into certifying when noisy belief evolution has produced a state estimate reliable enough for symbolic reasoning.}
    \item \revRefFour{We develop entropy and information-gain gates that operationalize revelation events by extracting stable belief anchors only when the belief estimate is concentrated and informative relative to its predictive prior.}
    \item \revRefFour{We build the Belief Automaton online from stable belief anchors and translate the synthesized symbolic policy into MCTS guidance, with explicit approximation bounds for judging when this guidance remains reliable for the original continuous POMDP objective.}

\end{enumerate}

\revRefFour{Together, these contributions define a principled online mechanism that couples continuous-belief planning with finite symbolic memory---requiring no predefined discrete abstraction: revelation events produce stable anchors, anchors induce a finite Belief Automaton, and the resulting symbolic guidance biases future search under explicit approximation control.}

The remainder of the paper is organized as follows. Sec.~\ref{sec:related_work} reviews related work, Sec.~\ref{sec:problem_formulation} formalizes the problem, Sec.~\ref{sec:methodology} presents REBA, Sec.~\ref{sec:experiments} reports experiments, Sec.~\ref{sec:discussion} discusses bound magnitudes, assumption stress, and scalability, and Sec.~\ref{sec:conclusion} concludes. \revRefTwo{Throughout the paper, S-prefixed labels generated from Supplementary Material cross-references (e.g., Sec.~\ref{supp:proofs_revelation}, Table~\ref{tab:robust_static_2d}, Theorem~\ref{supp:thm:robust_revelation}) refer to the Supplementary Material.}
\section{Related Work}
\label{sec:related_work}
\revRc{Online planning for continuous POMDPs under $\omega$-regular specifications requires three ingredients: belief approximation for continuous uncertainty, finite memory for logical progress, and search guidance that can feed this memory back into online planning. Existing work provides these ingredients separately through continuous POMDP solvers, temporal-logic learning and control, and abstraction-guided search. REBA targets the missing mechanism: identifying revelation events in noisy continuous belief streams, maintaining finite memory from the resulting stable anchors, and feeding symbolic guidance back into online search.}

\subsection{\revRefThree{Belief-Space Planning for Continuous POMDPs}}
\label{sec:rw_compression}
\revRefTwo{Classical belief-space planners reduce the complexity of continuous POMDPs by approximating the belief space with a finite set of representative beliefs (point-based methods such as \revRefThree{Point-Based Value Iteration (}PBVI\revRefThree{)}~\cite{Pineau2003}) or by projecting beliefs onto a low-dimensional manifold (\revRefThree{Principal Component Analysis (}PCA\revRefThree{)-based} belief compression~\cite{Roy2005}). Online MCTS variants extend this idea with progressive widening~\cite{sunberg2018online}, observation re-weighting~\cite{Hoerger_48,zhang2025observation}, and information-particle tree search for belief-dependent rewards in continuous domains~\cite{Fischer_IPFT_2020}. \revRefThree{Adjacent active-sensing work combines Bayesian inference with learned sampling policies for drone-based source estimation~\cite{vanHove_Methane_2026}. These methods support belief representation, local search efficiency, or active information gathering, usually for standard planning objectives.} \revRc{REBA adds the event-level reliability mechanism needed for $\omega$-regular progress by constructing finite memory online from stable anchors produced by revelation events.}}

\subsection{\revRefThree{Symbolic and Neuro-Symbolic Reinforcement Learning for Temporal Logic}}
\label{sec:rw_symbolic}
\revRefTwo{Methods such as reward machines~\cite{Icarte2022} and logically-constrained reinforcement learning~\cite{Hasanbeig2020} enforce temporal-logic objectives during training. \revRefThree{This extends to using linear temporal logic for reward shaping in multiagent hierarchical reinforcement learning~\cite{Liu_LRS_2026} and to combining signal temporal logic specifications with control barrier functions for robust nonlinear control~\cite{Zhou_2025}.} \revRefThree{On the theoretical front of POMDPs,} a recent breakthrough establishes decidability of $\omega$-regular synthesis for specific POMDP subclasses via a revelation mechanism~\cite{Belly_2025} (the general case remains undecidable~\cite{CHATTERJEE}). \revRefThree{Subsequent work shows that the limit-sure problem and quantitative parity analysis in revealing POMDPs are EXPTIME-complete~\cite{Asadi_RevealingPOMDPs_2026}.} Runtime shielding provides a complementary safety-specification interface: it restricts available actions so online POMDP or RL policies satisfy reach-avoid or avoidance constraints under known safety dynamics or precomputed winning regions~\cite{Le_Court_2025,Sheng_ICRA24}. \revRc{These methods provide temporal objectives, decidable revealed subclasses, and runtime safety filters, with guarantees tied to known discrete models, belief-support winning regions, offline reward shaping, or restricted revealed classes. REBA operationalizes revelation events inside a continuous, partially observable belief stream and uses the resulting stable anchors for online symbolic guidance.}}

\subsection{\revRefThree{Hierarchical and Abstraction-Guided Search}}
\label{sec:rw_hier_mcts}
\revRefTwo{\revRefThree{To bridge abstract reasoning and continuous control, broad hierarchical decision frameworks integrate discrete Markov decision processes with continuous model predictive control~\cite{Wang_HMDP_MPC_2025}, while adjacent safety planners combine reachability analysis with receding horizon contingency trajectories for occluded autonomous driving~\cite{Zheng_OACP_2026}.} \revRefThree{Within the specific domain of MCTS,} hierarchical search variants such as option-based MCTS, reference-policy POMDP planning, and counterfactual open-loop MCTS steer sampling through options, reference policies, or compact open-loop search structures~\cite{Vien2015,kim2025,Phan_2025}. Recent continuous-POMDP extensions improve online search by reusing previous planning information in online anytime belief-space planning~\cite{Chen_NeurIPS23,Novitsky_RAL_2025}, adding recursive dual guidance for constrained planning~\cite{Stocco_ICAPS24}, learning offline policy and value approximations for long-horizon belief-state planning~\cite{Moss_BetaZero_2024}, or simplifying observation models with probabilistic value bounds~\cite{LevYehudi_AAAI24}. \revRc{These methods improve sampling efficiency, constraints, or heuristic quality for finite-horizon online planning. REBA adds an online event-discovery route that maintains finite symbolic memory from stable anchors and uses that memory to guide cyclic progress under explicit approximation accounting.}}

\section{Problem Formulation and Background}
\label{sec:problem_formulation}
We formalize the agent's sequential decision-making problem as a POMDP, defined by the tuple $\mathcal{P} = (\mathcal{X}, \mathcal{A}, \mathcal{Z}, p_T, p_Z, R, \gamma, \bbel_0)$, where $\mathcal{X} \subseteq \mathbb{R}^{d_S}$ and $\mathcal{Z} \subseteq \mathbb{R}^{d_O}$ are continuous state and observation spaces. The agent's actions $a \in \mathcal{A}$ influence state transitions through $p_T(x'|x,a)$ and subsequent observations via $p_Z(z|x',a)$. The agent's behavior is guided by a reward function $R: \mathcal{X} \times \mathcal{A} \to \mathbb{R}$, with future rewards discounted by a factor $\gamma \in [0, 1)$. The agent maintains a belief, $\bbel_t$, over its latent state, updated recursively with the Bayes filter:
\begin{equation}
    \bbel_t(x') \propto p_Z(z_t|x', a_{t-1}) \int_{\mathcal{X}} p_T(x'|x, a_{t-1}) \bbel_{t-1}(x) \, dx
\end{equation}

\revRefTwo{The planning objective is the satisfaction probability of an $\omega$-regular specification $\varphi$; rewards $(R,\gamma)$ serve as search heuristics.} Let $AP$ be a finite set of atomic propositions (e.g., Goal, Obstacle) and $L: \mathcal{X} \to 2^{AP}$ be a labeling function mapping states to logical properties. We represent $\varphi$ as a Deterministic Parity Automaton (DPA) $\mathcal{A}_\varphi=\big(Q_\varphi, 2^{AP}, \delta_\varphi, q_{\varphi,0}, \mathrm{prio}_\varphi\big)$. A trajectory satisfies $\varphi$ \revRefThree{if and only if} the highest priority visited infinitely often in the induced run is even. Thus, we maximize the satisfaction probability $V_{\mathcal{P}}^\varphi(\pi) := \mathbb{P}_{\mathcal{P}}^\pi[x_0,x_1,\dots \models \varphi]$.

Given the intractability of finding optimal policies, we employ an online planning approach based on MCTS. MCTS concentrates search effort within the portion of the belief space reachable from the current belief. Within this search tree, each node represents a belief $\bbel_h$, which is approximated by a weighted particle set $\mathcal{S}_h = \{(x_{h}^{(i)},w_{h}^{(i)})\}_{i=1}^{N_{p}}$, where $\sum_{i=1}^{N_p} w_{h}^{(i)} = 1$ and $N_{p}$ denotes the number of particles. The tree expands through recursive simulations: from a belief node, an action is chosen, a new state and observation are sampled from the generative models ($p_T, p_Z$), and the resulting posterior belief, updated via a particle filter, becomes a new child node. To manage the high branching factor induced by continuous action (and observation) spaces, techniques like Progressive Widening \cite{sunberg2018online} are commonly integrated to judiciously limit exploration.
\begin{figure}
    \centering
    \includegraphics[width=\linewidth]{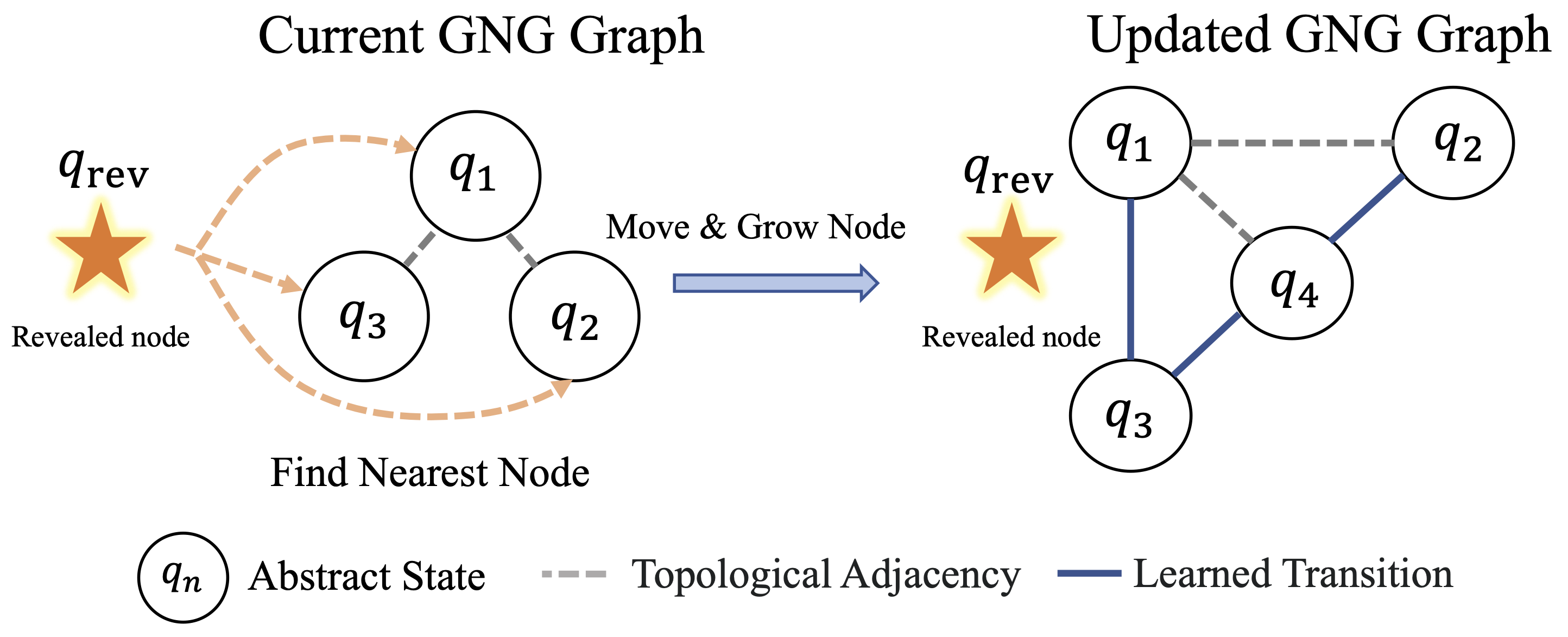}
    \caption{Dynamic update of the \revRefThree{Growing Neural Gas (GNG)} graph. \textbf{Left:} The current GNG graph is shown together with a newly revealed node $q_{\mathrm{rev}}$, which is treated as an input sample to the GNG algorithm. \textbf{Right:} Processing this sample with the GNG update refines the topology, adapting nodes and edges and, when needed, inserting new nodes (e.g., $q_4$) and connections, yielding an updated GNG graph for subsequent policy synthesis.}
    \label{fig:gng}
\end{figure}

\revRc{To handle the persistent uncertainty that challenges long-horizon planning in POMDPs, REBA operationalizes the formal-methods concept of a ``revelation''~\cite{Belly_2025} inside online belief-space search. Boers' entropy, $H_{\text{Boers}}(\bbel)$~\cite{Boers_2010}, serves as an operational concentration score: for Gaussian-like beliefs, low entropy implies low covariance and therefore high belief concentration. Together with the information-gain gate in Sec.~\ref{sec:revelation}, this score determines whether a candidate belief marks a revelation event that can supply a stable anchor for symbolic reasoning.}

\revRc{Stable anchors extracted from revelation events must then be stored in a finite memory that preserves coverage information and supports symbolic transitions. REBA uses Growing Neural Gas (GNG)~\cite{FISER201372,MING2024101541} as the incremental maintenance mechanism for this memory (Fig.~\ref{fig:gng}): the graph $\mathcal{G}_S=(N_S,E_S)$ provides a witness coverage radius for Theorem~\ref{thm:pac_coverage}, while its map radius enters the conditional abstraction-error term $\varepsilon_{\text{abs}}$ in Theorem~\ref{thm:performance_guarantee}. The resulting graph supplies the discrete state space for the Belief Automaton on which we perform formal policy synthesis.}

\section{Methodology}
\label{sec:methodology}

\revRc{REBA is an online revelation mechanism for continuous-belief planning. At each planning step, it evaluates candidate beliefs generated inside MCTS, applies revelation gates to identify stable anchors, updates the corresponding finite memory, synthesizes a symbolic policy on the product model, and feeds that policy back into subsequent search. The persistent state consists of the GNG graph $\mathcal{G}_S$, the empirical Belief Automaton $\hat{\mathcal{B}}(T)$, and the current abstract policy $\hat{\pi}^*$. Sec.~\ref{sec:revelation} defines the revelation gates, Sec.~\ref{sec:abstraction} maintains finite symbolic memory from stable anchors, Sec.~\ref{sec:automaton} learns the Belief Automaton and states the error accounting, and Sec.~\ref{sec:guidance} maps parity guidance into MCTS.}

\subsection{Online Revelation via Belief Concentration}
\label{sec:revelation}
\revRc{The revelation gate identifies which MCTS belief nodes have become stable enough to support symbolic reasoning. Its input is a stream of candidate beliefs from tree expansion, and its output is a stream of revealed representatives $x_{\mathrm{rev}}$. We first state the analytic entropy and gain certificate, then transfer it to the particle implementation through Theorem~\ref{thm:robust_revelation}, and finally give the implemented three-gate rule used by REBA.}

\revRefThree{A revealed belief must support two operations: extracting a localized representative state and confirming that the observation has made the posterior informative relative to the predictive prior. Differential entropy captures posterior spread, while a KL-gain test captures prior-to-posterior information change; used separately, either quantity can miss one side of this requirement. Boers' entropy links these two quantities through Lemma~\ref{lem:boers_identity}, and the explicit gain gate below keeps the informativeness condition separate. Theorem~\ref{thm:robust_revelation} then transfers the implemented particle score to a high-probability covariance-volume certificate under consistency assumptions.}

Let $\bbel(x' \mid h,a,z)$ denote the posterior over next states after action $a$ and observation $z$ under history $h$. We compare this posterior with the one-step mixture prior
\begin{equation}
    C(x' \mid h,a)
    \;=\;
    \int p_T(x' \mid x,a)\, \bbel_h(x)\,\mathrm{d}x ,
\end{equation}
where $p_T(\cdot\mid x,a)$ is the transition density and $\bbel_h$ is the predictive belief before observing $z$. Boers' entropy is
\begin{equation}
    H_{\text{Boers}}(\bbel)
    \;:=\;
    H_s(b)\;-\;\mathbb{E}_{x'\sim b}\big[\log C(x')\big],
\end{equation}
where $H_s(b)$ is the usual differential entropy. Under mild regularity assumptions, $H_{\text{Boers}}$ admits the identity
\begin{equation}
\label{Lam1:eq1}
    H_{\text{Boers}}(\bbel)
    \;=\;
    2\,H_s(b)
    \;+\;
    D_{\mathrm{KL}}\!\left(b \,\Vert\, C\right),
\end{equation}
as formalized in Lemma~\ref{lem:boers_identity}. The identity separates posterior volume from information gain relative to $C$. We therefore pair the entropy threshold with an explicit gain gate:
\begin{equation}
D_{\mathrm{KL}}(\bbel\Vert C)\ \ge\ \kappa_{\mathrm{gain}}.
\end{equation}
\revRc{Together, the entropy threshold and gain gate define the analytic revelation test used to justify treating revealed beliefs as stable anchors with bounded localization error.}

For Gaussian posteriors with density $b=\mathcal{N}(\mu,\Sigma)$ on $\mathbb{R}^{d_S}$, the usual entropy formula
\begin{equation}
    H_s(b)
    \;=\;
    \tfrac{1}{2}\log\!\big((2\pi e)^{d_S}\det\Sigma\big)
\end{equation}
implies that
\begin{equation}
    H_{\text{Boers}}(\bbel)
    \;=\;
    \log\!\big((2\pi e)^{d_S}\det\Sigma\big)
    \;+\;
    D_{\mathrm{KL}}(\bbel\Vert C).
\end{equation}
Thus, enforcing the entropy threshold and the gain gate,
\begin{equation}
    H_{\text{Boers}}(\bbel)\;\le\;\theta_{\mathrm{reveal}},
    \qquad
    D_{\mathrm{KL}}(\bbel\Vert C)\;\ge\;\kappa_{\mathrm{gain}},
\end{equation}
yields an explicit upper bound on the covariance determinant,
\begin{equation}
    \det\Sigma
    \;\le\;
    \frac{\exp\!\big(\theta_{\mathrm{reveal}}-\kappa_{\mathrm{gain}}\big)}{(2\pi e)^{d_S}},
\end{equation}
which is the Gaussian volume interpretation of a revealed belief.

In practice, REBA fits $\hat{\bbel}=\mathcal N(\mu_p,\widehat\Sigma_p)$ to particles and estimates the mixture prior with a \revRefThree{Kernel Density Estimator (}KDE\revRefThree{)} $\hat C$, giving
\begin{equation}
    \hat{H}_{\text{Boers}}
    \;=\;
    2H_s(\hat{\bbel})\;+\;D_{\mathrm{KL}}(\hat{\bbel} \Vert \hat C)
    \;\ge\; \log\!\big((2\pi e)^{d_S}\det(\widehat\Sigma_p)\big).
\end{equation}
\revRefTwo{The next theorem transfers this computed particle score to the true posterior under standard Sequential Monte Carlo (SMC) and KDE consistency assumptions. It gives a high-probability covariance certificate; finite-sample rates for the present weighted-SMC/KDE implementation would require stronger assumptions.}

\begin{theorem}[\revRc{Belief covariance volume bound from the revelation test}]
\label{thm:robust_revelation}
Let $S\subset\mathbb{R}^{d_S}$ be compact. Assume the posterior belief $\bbel$ and the mixture prior $C$ admit densities w.r.t.\ Lebesgue measure, $\bbel\ll C$, and $C$ is continuous, bounded, and bounded away from zero on $S$. Let the true posterior be approximated by the Gaussian fit $\hat{\bbel}=\mathcal{N}(\mu_p,\widehat\Sigma_p)$ obtained from a SMC particle set of size $N_p$ produced by a consistent filter, and assume $\widehat\Sigma_p \to \Sigma_b$ in probability as $N_p\to\infty$. Let $\hat C$ be a KDE with bandwidth $h_{N_p}\to 0$ and $N_p h_{N_p}^{d_S}\to\infty$. Define $\hat H_{\text{Boers}} := 2H_s(\hat{\bbel})+D_{\mathrm{KL}}(\hat{\bbel}\Vert \hat C)$.

For any revelation threshold $\theta_{\mathrm{reveal}}$, tolerance $\varepsilon_b>0$, and confidence level $1-\delta\in(0,1)$, there exists a sufficient particle count $N_p^\star=N_p^\star(\varepsilon_b,\delta)$ such that for all $N_p\ge N_p^\star$, if the test returns $\hat H_{\mathrm{Boers}}\le \theta_{\mathrm{reveal}}$, then with probability at least $1-\delta$,
\begin{equation}
\label{lam2:eq1}
    \det(\Sigma_b)\ \le\ \frac{\exp(\theta_{\mathrm{reveal}})}{(2\pi e)^{d_S}}\ +\ \varepsilon_b.
\end{equation}
Moreover, if one also enforces the implemented gain threshold \revRefTwo{$D_{\mathrm{KL}}(\hat{\bbel}\Vert \hat C)\ge \kappa_{\mathrm{gain}}\ge 0$}, then with the same probability,
\begin{equation}
\label{lam2:eq1tight}
    \det(\Sigma_b)\ \le\ \frac{\exp(\theta_{\mathrm{reveal}}-\kappa_{\mathrm{gain}})}{(2\pi e)^{d_S}}\ +\ \varepsilon_b.
\end{equation}
The proof is provided in Sec.~\ref{supp:proofs_revelation}.
\end{theorem}

Here, $\det(\Sigma_b)$ is the determinant of the true posterior covariance and $d_S$ is the state dimension. \revRc{Theorem~\ref{thm:robust_revelation} is the concentration certificate for revelation: a low computed score implies small true posterior dispersion with high probability under the stated assumptions.}

\revRc{Thus, revealed beliefs can be used as stable anchors with bounded posterior dispersion under the stated consistency and regularity assumptions.}

\begin{figure}
\centering
\includegraphics[width=\linewidth]{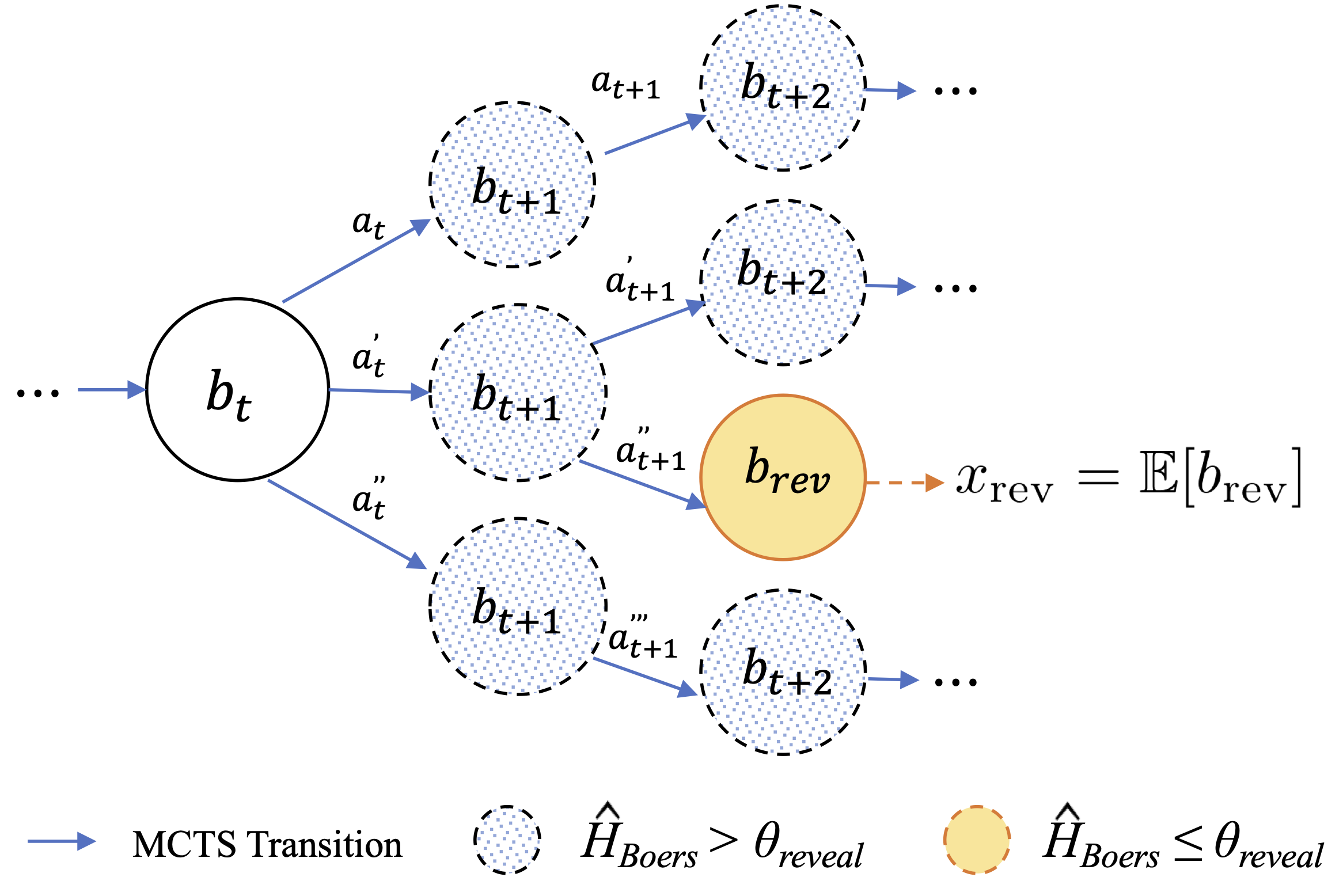}
\caption{\revRc{Online revelation and abstraction mechanism. Starting from the current belief $\bbel_t$, MCTS expands a look-ahead tree over candidate future beliefs. Nodes with dashed outlines denote beliefs with high uncertainty, i.e., $H_{\mathrm{Boers}} \ge \theta_{\mathrm{reveal}}$. When a sequence of actions and observations reaches a node $\bbel_{\mathrm{rev}}$ whose entropy falls below the threshold ($H_{\mathrm{Boers}} < \theta_{\mathrm{reveal}}$), a revelation event is triggered. The revealed belief (solid orange) supports the extraction of a representative state vector $x_{\mathrm{rev}}$ (e.g., the particle mean), which is used as a stable and informative anchor for subsequent abstraction.}}
\label{fig:revelation}
\end{figure}

\revRc{Fig.~\ref{fig:revelation} illustrates the entropy core of the analytic revelation test. In online MCTS, the particle distribution and local tree context change across depth, so the implementation adapts the entropy cutoff through $\theta_{\mathrm{reveal}}^{\mathrm{dyn}}$ and requires two additional checks: a gain gate for prior-to-posterior informativeness and a drop gate for a new entropy decrease from the parent node. With $\Delta\hat H_{\mathrm{Boers}}(h):=\hat H_{\mathrm{Boers}}(\bbel_{\mathrm{parent}(h)})-\hat H_{\mathrm{Boers}}(\bbel_h)$, Definition~\ref{def:revealed_belief} gives the implemented revealed-belief rule.}

\begin{definition}[Revealed belief]
\label{def:revealed_belief}
\revRefTwo{A belief $\bbel_h$ is revealed if it satisfies three conditions: $\hat H_{\mathrm{Boers}}(\bbel_h)\!\le\!\theta_{\mathrm{reveal}}^{\mathrm{dyn}}$, $D_{\mathrm{KL}}(\hat{\bbel}\Vert \hat C)\!\ge\!\kappa_{\mathrm{gain}}$, and $\Delta\hat H_{\mathrm{Boers}}(h)\!\ge\!\theta_{\mathrm{drop}}$.}
\end{definition}

When a belief passes these gates, REBA extracts its representative state $x_{\text{rev}}(h)$ as the particle mean. The resulting stream of revealed representatives is the input to the GNG abstraction in Sec.~\ref{sec:abstraction}.

\subsection{Finite Symbolic Memory Maintenance}
\label{sec:abstraction}

This subsection takes the revealed representatives $\{x_{\mathrm{rev}}^{(i)}\}_{i\ge 1}\subset\mathcal{X}$ from Sec.~\ref{sec:revelation} as input to construct finite prototype support. The role of this layer is to maintain a finite symbolic memory whose coverage and map radius can be monitored as new relevant regions are discovered. REBA uses a GNG graph $\mathcal{G}=(V,E)$ with node prototypes $W:V\to\mathcal{X}$ updated online from revealed vectors~\cite{Martinetz_1994,Fritzke_1994}. Each node $v\in V$ represents a local neighborhood around $W(v)$, and the learned graph supplies the discrete support used in Sec.~\ref{sec:automaton}.

For the coverage analysis, $\mathcal{R}^\star$ denotes a theoretical reference set of specification-relevant regions; Theorem~\ref{thm:pac_coverage} measures whether the GNG prototypes cover this target set while the algorithm continues to update the graph only from revealed representatives.

\begin{definition}[Core revealed regions and $\varepsilon_{\rm cov}$-coverage]
\label{def:core_regions}
Let $\mathcal{R}^\star=\{R_1^\star,\dots,R_M^\star\}\subset\mathcal{X}$ be a finite collection of disjoint, compact \emph{core regions} relevant to the $\omega$-regular specification. Given a GNG graph $\mathcal{G}=(V,E)$ with prototypes $W(\cdot)$, we say that $\mathcal{G}$ \emph{$\varepsilon_{\rm cov}$-covers} $\mathcal{R}^\star$ if for every $j\in\{1,\dots,M\}$ there exists a node $v\in V$ such that
\begin{equation}
    \mathrm{dist}(W(v),R_j^\star) := \inf_{x\in R_j^\star}\|W(v)-x\|_2 \le \varepsilon_{\rm cov}.
\end{equation}
\end{definition}
\revRefTwo{This coverage notion records whether each core region has a nearby prototype; the later performance bound also requires each revealed representative used by the abstraction to satisfy the map-radius condition.}

\paragraph{Coverage conditions}
Our coverage guarantee relies on two high-level properties of the revealed-sample stream and the exploration behavior of the underlying planner:
\begin{itemize}[leftmargin=1.25em]
    \item (H1) \textbf{Bounded revealed vectors:} $\exists R_S>0$ such that $\|x_{\mathrm{rev}}^{(i)}\|_2\le R_S$ for all $i$.
    \item(H2) \textbf{$\tau$-step minorization over uncovered regions:} \revRefTwo{there exist $\tau\ge 1$ and $p_{\min}\in(0,1]$ such that, whenever at least one core region is not yet $\varepsilon_{\rm cov}$-covered by the current abstraction, the next $\tau$ revealed vectors contain at least one sample in the currently uncovered union with probability at least $p_{\min}$.}
\end{itemize}
Formal statements and a constructive sufficient condition (linking $p_{\min}$ to the MCTS expansion schedule) are provided in Sec.~\ref{supp:proofs_gng}, Remark~\ref{rem:h2_instantiation}.

\begin{theorem}[Coverage under minorization]
\label{thm:pac_coverage}
Assume (H1) and (H2). \revRefTwo{Assume also a monotone witness abstraction: when a revealed vector hits a currently uncovered core region, a witness prototype is inserted within that region; already inserted witnesses are not removed or merged in a way that destroys their $\varepsilon_{\rm cov}$-coverage; and the node budget can hold at least $M$ such witnesses.} For any $\varepsilon_{\rm cov}>0$ and confidence level $1-\delta\in(0,1)$, there exists
\begin{equation}
\label{eq:tcov_def}
    \revRefTwoMath{T_{\mathrm{cover}} = \tau\left\lceil \frac{2(M+\log(1/\delta))}{p_{\min}} \right\rceil}
\end{equation}
such that after processing $T\ge T_{\mathrm{cover}}$ revealed state vectors, the GNG abstraction $\mathcal{G}$ $\varepsilon_{\rm cov}$-covers all core regions in $\mathcal{R}^\star$ with probability at least $1-\delta$.

The proof of Theorem~\ref{thm:pac_coverage} and a constructive instantiation of (H2) are included in Sec.~\ref{supp:proofs_gng}.
\end{theorem}

\revRefTwo{Theorem~\ref{thm:pac_coverage} gives the coverage condition used to size and monitor the finite support. In the capped implementation, coverage claims require witness preservation.} \revRefFour{Because the implemented GNG uses a finite node cap, Theorem~\ref{thm:pac_coverage} is invoked as a conditional coverage account; capped runs report coverage-radius diagnostics as the operational monitor for this condition.} The implemented map gate uses $\rho_R=0.5$. The output of this subsection is therefore finite prototype support; the next subsection learns transitions over that support and composes it with the objective automaton.

\subsection{Belief Automaton Construction}
\label{sec:automaton}
\revRefTwo{Given the finite support from the GNG graph, this subsection adds transition probabilities and specification progress.} The symbolic guidance model is built in four steps: map each revealed belief to a prototype, learn abstract transition counts, compose the abstraction with the specification DPA, and synthesize the parity policy whose lifted execution is analyzed in Theorem~\ref{thm:performance_guarantee}.

\paragraph{Online belief abstraction}
We construct a finite belief abstraction $\hat{\mathcal{B}}(T)=(Q_{\mathcal{B}}, \mathcal{A}, \widehat P_B, q_{\mathcal{B},0})$ dynamically within the MCTS loop. To ensure statistical reliability, the abstraction interacts \emph{only} with revealed beliefs, \revRefTwo{as characterized by Theorem~\ref{thm:robust_revelation}}, that satisfy the concentration and informativeness gates.
The state space $Q_{\mathcal{B}}$ consists of nodes in a Growing Neural Gas (GNG) graph, bounded by a fixed budget.
For each revealed belief $\bbel$, we extract its continuous representative $x_{\mathrm{rev}}(\bbel)\in\mathcal{X}$ (the particle mean) and map it to a unique abstract state via nearest-prototype assignment:
\begin{equation}
\label{eq:map_main}
\mathrm{map}(\bbel)\ :=\ \arg\min_{q\in Q_{\mathcal{B}}}\ \bigl\|x_{\mathrm{rev}}(\bbel)-W(q)\bigr\|_{M_{\rm met}},
\end{equation}
where $W(q)$ is the prototype of node $q$, ties are broken deterministically by node creation time, and \revRefTwo{$\|z\|_{M_{\rm met}}:=\sqrt{z^\top M_{\rm met}z}$ for a fixed positive-definite implementation metric $M_{\rm met}$. Let $\underline{\lambda}_{\rm met}:=\lambda_{\min}(M_{\rm met})$, and let $c_{\rm met}$ denote the reciprocal square root of $\underline{\lambda}_{\rm met}$. To keep this abstraction faithful to the continuous representative, REBA enforces a map-radius gate $\rho_R$: a revealed belief is mapped to an existing prototype only when $\|x_{\mathrm{rev}}(\bbel)-W(\mathrm{map}(\bbel))\|_{M_{\rm met}}\le\rho_R$; otherwise, the monotone witness abstraction inserts a new node at $x_{\mathrm{rev}}(\bbel)$. Under this condition, the Euclidean radius used in Theorem~\ref{thm:performance_guarantee} is $\varepsilon_{\rm map}:=c_{\rm met}\rho_R$. The capped GNG implementation may merge close nodes to keep $Q_{\mathcal{B}}$ compact; theorem-level coverage claims require these merges to preserve existing witnesses.}
This mapping treats revealed beliefs as dimension-consistent, quasi-discrete anchors in $\mathcal{X}$.

\paragraph{Transition learning with Dirichlet smoothing}
Transitions are learned from simulations by counting abstract moves. For each pair $(q,a)$, we accumulate the count $N_{q,a\to q'}$ of transitions where $\mathrm{map}(\bbel_t)=q$, $a_t=a$, and $\mathrm{map}(\bbel_{t+1})=q'$. To handle data sparsity, we employ a Dirichlet-smoothed estimator:
\begin{equation}
\label{eq:PB_hat}
\widehat P_B(q'\mid q,a)\ :=\ \frac{N_{q,a\to q'}+\alpha}{\sum_{q''\in Q_{\mathcal{B}}}\bigl(N_{q,a\to q''}+\alpha\bigr)}\,,\qquad \alpha>0.
\end{equation}
Policy synthesis is deferred until the sample count $\sum_{q''} N_{q,a\to q''}$ exceeds a threshold $N_{\min}$ to avoid overfitting to early noise.

\paragraph{Product construction with specification}
Let $\mathcal{A}_\varphi=(Q_\varphi,2^{AP},\delta_\varphi,q_{\varphi,0},\mathrm{prio}_\varphi)$ be the specification DPA introduced in Sec.~\ref{sec:problem_formulation}. We synthesize the high-level policy over the \emph{Prioritized Belief Automaton (PBA)}, defined as the product MDP:
\begin{equation}
\label{eq:pba_def}
\hat{\mathcal{B}}_\varphi(T)=\big(Q_\times,\ \mathcal{A},\ \widehat P_\times,\ (q_{\mathcal{B},0},q_{\varphi,0}),\ F_\times\big).
\end{equation}
The state space is $Q_\times:=Q_{\mathcal{B}}\times Q_\varphi$, and the priority function is inherited from the DPA: $F_\times(q,u) := \mathrm{prio}_\varphi(u)$.
Crucially, the DPA state evolves by evaluating atomic propositions on the abstract prototype $W(q')$. The product transition kernel is thus:
\begin{equation}
\widehat P_\times\big((q',u')\mid(q,u),a\big)
=
\widehat P_B(q'\mid q,a)\cdot \mathbf{1}\Big[u'=\delta_\varphi\big(u,L(W(q'))\big)\Big].
\end{equation}
Solving the parity objective on $\hat{\mathcal{B}}_\varphi(T)$ yields a strategy $\hat{\pi}^*:Q_\times\rightarrow\mathcal{A}$ that guides the MCTS.
Theoretical results provided in Sec.~\ref{supp:proofs_performance} indicate that parity satisfaction on this structure can be reduced to a reachability objective, which is robust to bounded perturbations in the learned kernel $\widehat P_B$.

\paragraph{Approximate performance guarantee}
The abstraction provides revelation-controlled symbolic memory for continuous-belief planning, with its influence on satisfaction performance tracked explicitly through learning and abstraction errors. Theorem~\ref{thm:performance_guarantee} states when this memory remains a valid approximation for $\omega$-regular guidance; its notation separates the learned product abstraction, the ideal revealed-anchor abstraction, and the lifted execution in the original POMDP.
Let $\widehat{\mathcal{B}}_\varphi(T)$ be the learned product abstraction with learned kernel $\widehat P_B$, and let $\mathcal{B}_\varphi^\star$ be the ideal product abstraction induced by the same \revRc{revealed} anchors but by the ideal \revRc{kernel over revealed anchors} $P_B^\star$. Let $\Pi(\hat\pi^*)$ denote the execution of a learned abstract policy in the original POMDP by maintaining the DPA state and applying the belief-to-anchor map; write $P_{\mathrm{lift}}$ for the comparison kernel induced by this lifted execution. All three values below use the same parity-to-reachability reduction with fixed accepting and bad target sets $G$ and $B$.
The bound uses four groups of conditions: successful revelation and monotone coverage, a map radius gate for revealed anchors, an L1 error event for the learned kernel, and a finite expected hitting time after the parity objective is reduced to reachability.

\begin{theorem}[\revRefTwo{Conditional approximate policy performance guarantee}]
\label{thm:performance_guarantee}
Let $\hat{\pi}^*:Q_\times\rightarrow\mathcal{A}$ be an optimal policy synthesized on $\widehat{\mathcal{B}}_\varphi(T)$, and let $\Pi(\hat{\pi}^*)$ be its lifted execution in the original POMDP. \revRefTwo{Assume that, on an event $E_{\mathrm{succ}}$ with $\mathbb{P}(E_{\mathrm{succ}})\ge 1-\delta_{\mathrm{learn}}$, the revelation and monotone coverage events hold, every revealed representative used by the abstraction passes the map gate $\rho_R$ (equivalently, has Euclidean distance at most $\varepsilon_{\rm map}$), and the learned kernel satisfies}
\begin{equation}
\label{eq:l1-model-error}
    \revRefTwoMath{\sup_{q,a}\bigl\|\widehat P_B(\cdot\mid q,a)-P_B^\star(\cdot\mid q,a)\bigr\|_1 \le \varepsilon_P.}
\end{equation}
\revRefTwo{Assume further that, after the common reduction to reachability, $\mathbb{E}_{K}^{\pi}[\tau_{G\cup B}]\le H$ for every stationary policy considered and every $K\in\{\widehat P_B,P_B^\star,P_{\mathrm{lift}}\}$. Then}
\begin{equation}
\label{thm6:eq1}
    \revRefTwoMath{V_{\mathcal{P}}^\varphi(\Pi(\hat{\pi}^*)) \;\ge\; V^{\varphi}_{\widehat{\mathcal{B}}}(\hat{\pi}^*) \;-\; \delta_{\text{learn}} \;-\; \varepsilon_{\text{model}} \;-\; \varepsilon_{\text{abs}}.}
\end{equation}
\revRefTwo{Here $\varepsilon_{\text{model}}:=2\sqrt{H\varepsilon_P}$ comes from the L1 kernel error in \eqref{eq:l1-model-error}. The Lipschitz regularity condition on the continuous transition kernel with constant $L_T$ (Assumption~\ref{ass:lipschitz-kernel}) is used only for the abstraction term, which admits the explicit bound}
\begin{equation}
\label{eq:eps-abs-bound}
    \revRefTwoMath{\varepsilon_{\text{abs}} \;\le\; 2\sqrt{H\,L_T\,\varepsilon_{\rm map}},}
\end{equation}
\revRefTwo{so $\varepsilon_{\text{abs}} = \mathcal{O}(\sqrt{\varepsilon_{\rm map}})$ as $\varepsilon_{\rm map} \to 0$. In the accepting abstract instances considered in our experiments, $V^{\varphi}_{\widehat{\mathcal{B}}}(\hat{\pi}^*)=1$, in which case the bound simplifies to $V^\varphi_{\mathcal{P}}(\Pi(\hat{\pi}^*))\ge 1-(\delta_{\text{learn}}+\varepsilon_{\text{model}}+\varepsilon_{\text{abs}})$.}
\end{theorem}

The proof of Theorem~\ref{thm:performance_guarantee}, including the parity-to-reachability reduction, the continuity bound, and the derivation of \eqref{eq:eps-abs-bound}, is provided in Sec.~\ref{supp:proofs_performance}. \revRefTwo{To connect the bound to the implementation, we report the local diagnostic score $B_{\rm loc}$ and measured proxies for the learning, model, and abstraction terms in Fig.~\ref{fig:thm3_decomposition_main}.} \revRefTwo{The full numerical proxies remain in Sec.~\ref{supp:thm3_magnitudes}, Table~\ref{tab:supp_thm3_error_proxies}.}

\subsection{Symbolic Feedback to MCTS}
\label{sec:guidance}
\revRefTwo{The learned Belief Automaton serves as a guidance state that is fed back into online search.} Fig.~\ref{fig:syn_policy} shows how the abstract policy $\hat{\pi}^*$ is synthesized, and \revRc{this subsection explains the integration of this policy into MCTS}. REBA uses $\hat{\pi}^*$ in two places: as a tree policy bias during selection and expansion, and as a default policy guide during rollouts.

\begin{figure}[htbp]
  \centering
  \includegraphics[width=\linewidth]{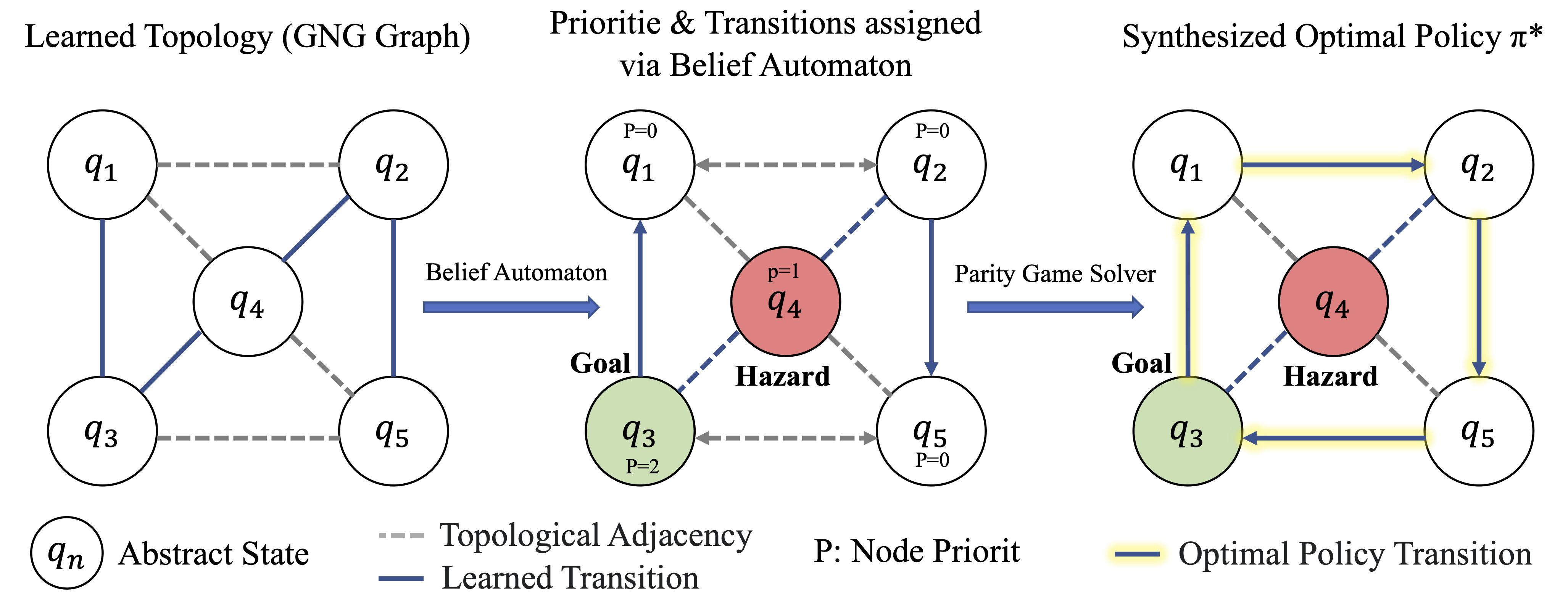}
  \caption{Synthesis of the optimal high-level policy. \textbf{Left:} A topological map is learned as a GNG graph that captures connectivity among abstract states $(q_n)$. \textbf{Center:} The learned belief abstraction is composed with the specification DPA, inducing valid priorities on product states $(q,u)$, and the learned transitions are formalized.
 \textbf{Right:} \revRefTwo{A solver for parity objectives over the product MDP computes} the optimal policy $\hat{\pi}^*$ (yellow) that satisfies the specified objective; here $\hat{\pi}^*$ induces a cycle that repeatedly visits $q_3$ while avoiding $q_4$.}
  \label{fig:syn_policy}
\end{figure}

\textbf{Tree policy guidance.}
When the search encounters a revealed belief that maps to an abstract node $q\in Q_{\mathcal B}$, we update the specification memory $u \leftarrow \delta_\varphi\big(u, L(W(q))\big)$ and denote the resulting PBA state by $\tilde q=(q,u)$. The action $a^*=\hat{\pi}^*(\tilde q)$ then receives an optimistic but vanishing bias in the tree score, encouraging expansions consistent with abstract temporal progress.
The implemented score augments the standard \revRefThree{Upper Confidence Bounds for Trees (}UCT\revRefThree{)} rule as
\begin{equation}
\label{eq29}
\begin{aligned}
    \text{UCT}(\tilde q,a) 
    = &\hat{Q}(\tilde q,a)
    + c_{\text{UCT}}\, \sqrt{\frac{\log N_T(\tilde q)}{\max\{1,\,N_T(\tilde q,a)\}}}\\
    &+ \beta_t(\tilde q)\,\mathbb{I}\bigl[\hat{\pi}^*(\tilde q)=a\bigr], \\
    \text{with}\quad & 0 \le \beta_t(\tilde q) \le \lambda\, c_{\text{UCT}} \sqrt{\frac{\log N_T(\tilde q)}{\max\{1,\,N_T(\tilde q,a)\}}}, \\
    & \beta_t(\tilde q)\downarrow 0,\quad \lambda\in(0,1).
\end{aligned}
\end{equation}
The bonus biases early search toward the abstract policy while preserving asymptotic exploration consistency as $\beta_t(\tilde q)\to 0$. A formal link from such an expansion schedule to the minorization parameter $p_{\min}$ in Theorem~\ref{thm:pac_coverage} is provided in Sec.~\ref{supp:proofs_gng}, Remark~\ref{rem:h2_instantiation}.

\textbf{Default policy guidance.}
During MCTS rollouts, whenever a simulated belief triggers a revelation and maps to $q$, REBA updates $u$ by $u \leftarrow \delta_\varphi\big(u, L(W(q))\big)$ and selects the next rollout action as $\hat{\pi}^*(q,u)$. The rollout value estimate therefore receives temporal progress information from the product policy, replacing a purely random or local heuristic default.

Algorithm~\ref{alg:mcts-olag} shows the resulting online procedure. Starting from the current belief and the current guidance state, REBA runs MCTS simulations, updates reveal statistics, maps stable anchors into the GNG and Belief Automaton, updates transition counts, periodically re-solves the product MDP, and returns the best root action together with the updated guidance state. \revRefFour{Inside the pseudocode, $\Xi=(\mathcal{G}_S,\hat{\mathcal{B}},\hat{\pi}^*)$ denotes this mutable guidance state, with $\hat{\mathcal{B}}$ abbreviating the current $\hat{\mathcal{B}}(T)$; \Func{RevealGate} denotes the entropy, gain, and drop test in Definition~\ref{def:revealed_belief}.}

\begin{algorithm}[tb]
\caption{Revealed Belief Automaton (REBA)}
\label{alg:mcts-olag}
\begin{algorithmic}[1]
\Require Current belief $\bbel_{\text{curr}}$; guidance state $\Xi$
\Ensure Action $a_t$; updated guidance state $\Xi$
\Function{REBA}{$\bbel_{\text{curr}}, \Xi$}
    \State $T \leftarrow \Func{InitTree}(\bbel_{\text{curr}})$
    \For{$i = 1$ \textbf{to} $N_{\text{sims}}$}
        \State \Call{SimLearn}{$T.\text{root}, \Xi$}
    \EndFor
    \If{\Func{TimeToReplan}}
        \State $\hat{\pi}^* \leftarrow \Func{SolvePBA}(\hat{\mathcal{B}})$ \Comment{update $\Xi$}
    \EndIf
    \State $a_t \leftarrow \Func{BestAction}(T.\text{root})$
    \State \Return $a_t, \Xi$
\EndFunction
\Function{SimLearn}{$n_{\text{root}}, \Xi$}
    \State $(n_p,n) \leftarrow \Func{SelectExpand}(n_{\text{root}},\hat{\pi}^*)$
    \State $\bbel \leftarrow n.\text{belief}$
    \State $\eta_{\text{rev}} \leftarrow \Func{UpdateRevealStats}(n_p,n)$
    \If{\Call{RevealGate}{$\bbel,\eta_{\text{rev}}$}}
        \State $x_{\text{rev}} \leftarrow \Func{ParticleMean}(n.\text{particles})$
        \State $q_{\text{gng}} \leftarrow \Func{AdaptGNG}(\mathcal{G}_S,x_{\text{rev}})$
        \State $q \leftarrow \Func{MapBA}(\bbel,q_{\text{gng}})$
        \State \Func{UpdateCounts}$(\hat{\mathcal{B}},n_p.\text{q\_id},n.\text{action},q)$
        \State $n.\text{q\_id} \leftarrow q$
    \EndIf
    \State $R \leftarrow \Func{Rollout}(n,\hat{\pi}^*)$
    \State \Call{Backprop}{$n,R$}
\EndFunction

\end{algorithmic}
\end{algorithm}

\textbf{Summary of theorem roles.} \revRc{Theorem~\ref{thm:robust_revelation} certifies revelation events by bounding local belief concentration. Theorem~\ref{thm:pac_coverage} gives a coverage condition for the finite memory built from stable anchors. Theorem~\ref{thm:performance_guarantee} gives a conditional relative bound for the lifted learned policy under explicit learned kernel, model error, and map radius assumptions. The implemented planner uses the resulting abstract policy as guidance inside MCTS, and the complete behavior is evaluated empirically in Sec.~\ref{sec:experiments}.} \revRefFour{Sec.~\ref{sec:discussion_scaling} and Fig.~\ref{fig:per_step_cost_main} summarize the online breakdown; full timing and asymptotic details are in Sec.~\ref{supp:complexity}.}

\section{Experiments}
\label{sec:experiments}
\revRefFour{We evaluate REBA in four steps: first, mechanism isolation on published dynamics to verify that each architectural component is load-bearing (Sec.~\ref{sec:mechanism_attribution}); second, nominal performance across three scenarios to assess the complete system (Sec.~\ref{sec:nominal_performance}); third, fixed-hyperparameter robustness to test generalization (Sec.~\ref{sec:robustness_main}); and finally, quantitative diagnostics in Sec.~\ref{sec:discussion} to connect the measured performance to the theoretical error accounting.}

\subsection{Evaluation Protocol}
\label{sec:eval_protocol}

\begin{table*}[!t]
  \centering
  \caption{\revRefFour{Complete nominal performance comparison across the Static 2D, Dynamic 2D, and 3D Navigation scenarios (20 trials). Results are reported as \textbf{Mean $\pm$ Standard Deviation} where defined.}}
  \label{tab:nominal_complete}
  \label{perf1_tab}
  \label{perf2_tab}
  \label{perf3_tab}
  \scriptsize
  \setlength{\tabcolsep}{2.2pt}
  \renewcommand{\arraystretch}{1.08}
  \resizebox{\textwidth}{!}{%
  \begin{tabular}{@{}llccccccccc@{}}
    \toprule
    Scenario & Metric & QR\text{-}DQN & \makecell{Recurrent\\PPO} & A2C & TRPO & POMCP & REFSOLVER & GPOMCP & CPOMDP & \textbf{Ours} \\
    \midrule
    \multirow{4}{*}{Static 2D}
    & Cycles ($\uparrow$) & $0.0 \pm 0.0$ & $5.2 \pm 2.4$ & $1.0 \pm 1.3$ & $7.7 \pm 1.5$ & $2.5 \pm 2.1$ & $8.3 \pm 2.3$ & $9.1 \pm 3.1$ & $9.4 \pm 2.6$ & $\mathbf{11.0 \pm 2.2}$ \\
    & Success Rate \% ($\uparrow$) & 0 & 40 & 10 & \textbf{95} & 30 & 80 & 90 & 90 & \textbf{95} \\
    & TTFC (steps, $\downarrow$) & \textendash & $97 \pm 22$ & $82 \pm 37$ & $68 \pm 28$ & $75 \pm 33$ & $66.0 \pm 17$ & $54 \pm 21$ & $45 \pm 11$ & $\mathbf{34 \pm 12}$ \\
    & Relative Gain & \textendash & \textendash & \textendash & \textendash & \textendash & \textendash & \textendash & \textendash & \textbf{+17.0\%} \\
    \midrule
    \multirow{4}{*}{Dynamic 2D}
    & Cycles ($\uparrow$) & $0.2 \pm 0.7$ & $3.1 \pm 2.6$ & $0.9 \pm 1.4$ & $4.8 \pm 2.7$ & $0.9 \pm 1.9$ & $2.7 \pm 2.3$ & $6.8 \pm 3.5$ & $7.6 \pm 2.4$ & $\mathbf{11.2 \pm 2.1}$ \\
    & Success Rate \% ($\uparrow$) & 5 & 35 & 10 & 55 & 5 & 20 & 65 & 80 & \textbf{95} \\
    & TTFC (steps, $\downarrow$) & 188 & $101 \pm 34$ & $154 \pm 38$ & $82 \pm 36$ & 25 & $73 \pm 33$ & $62 \pm 31$ & $57 \pm 39$ & $\mathbf{42 \pm 28}$ \\
    & Relative Gain & \textendash & \textendash & \textendash & \textendash & \textendash & \textendash & \textendash & \textendash & \textbf{+47.4\%} \\
    \midrule
    \multirow{3}{*}{3D Navigation}
    & Success Rate \% ($\uparrow$) & 0 & 35 & 5 & 25 & 0 & 35 & 60 & 75 & \textbf{95} \\
    & TTFC (steps, $\downarrow$) & N/A & $151 \pm 43$ & 238 & $169 \pm 52$ & N/A & $127 \pm 33$ & $84 \pm 31$ & $97 \pm 44$ & $\mathbf{62 \pm 20}$ \\
    & Relative Gain in Success Rate & \textendash & \textendash & \textendash & \textendash & \textendash & \textendash & \textendash & \textendash & \textbf{+26.7\%} \\
    \bottomrule
  \end{tabular}
  }
  \begin{tablenotes}
    \footnotesize
    \item Note: TTFC (Time to First Cycle) is calculated based on successful trials only. Success-count sequences follow the method-column order: Static 2D has $n_{\rm succ}=0/8/2/19/6/16/18/18/19$, Dynamic 2D has $1/7/2/11/1/4/13/16/19$, and 3D Navigation has $0/7/1/5/0/7/12/15/19$. TTFC standard deviations are reported only when $n_{\rm succ}\ge2$; single successful-trial values are reported without a standard deviation, and N/A marks no successful trials. Relative Gain compares Ours with CPOMDP, using Cycles for the two 2D patrolling tasks and Success Rate for the 3D reach-avoid task.
  \end{tablenotes}
\end{table*}

\begin{figure*}[!t]
    \centering
    \captionsetup[subfigure]{justification=centering}
    \newcommand{\figfivepanelheight}{3.2cm}
    \begin{subfigure}[t]{0.24\textwidth}
        \centering
        \includegraphics[width=\linewidth,height=\figfivepanelheight,keepaspectratio]{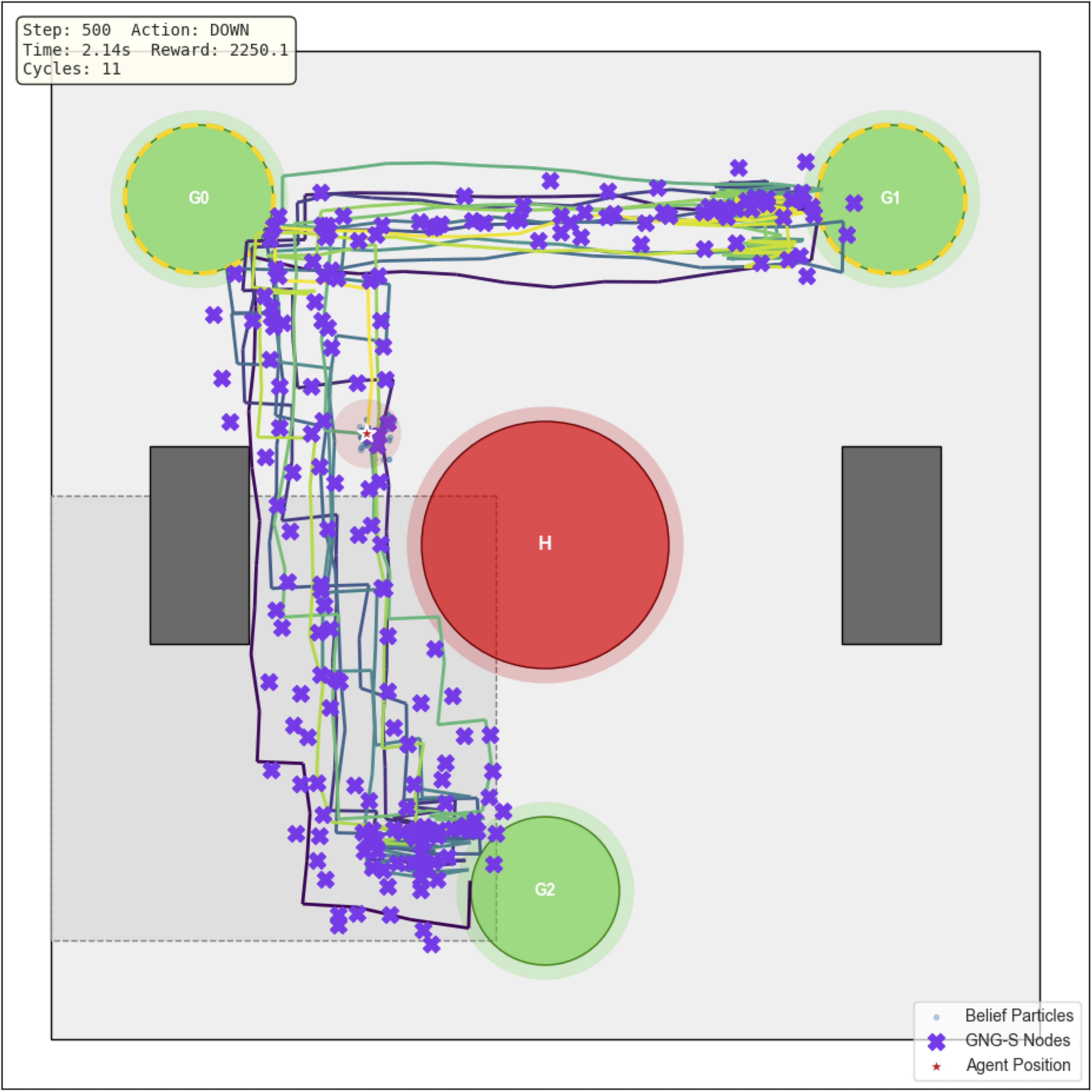}
        \caption{Static 2D}
        \label{fig_static_result}
    \end{subfigure}
    \hfill
    \begin{subfigure}[t]{0.24\textwidth}
        \centering
        \includegraphics[width=\linewidth,height=\figfivepanelheight,keepaspectratio]{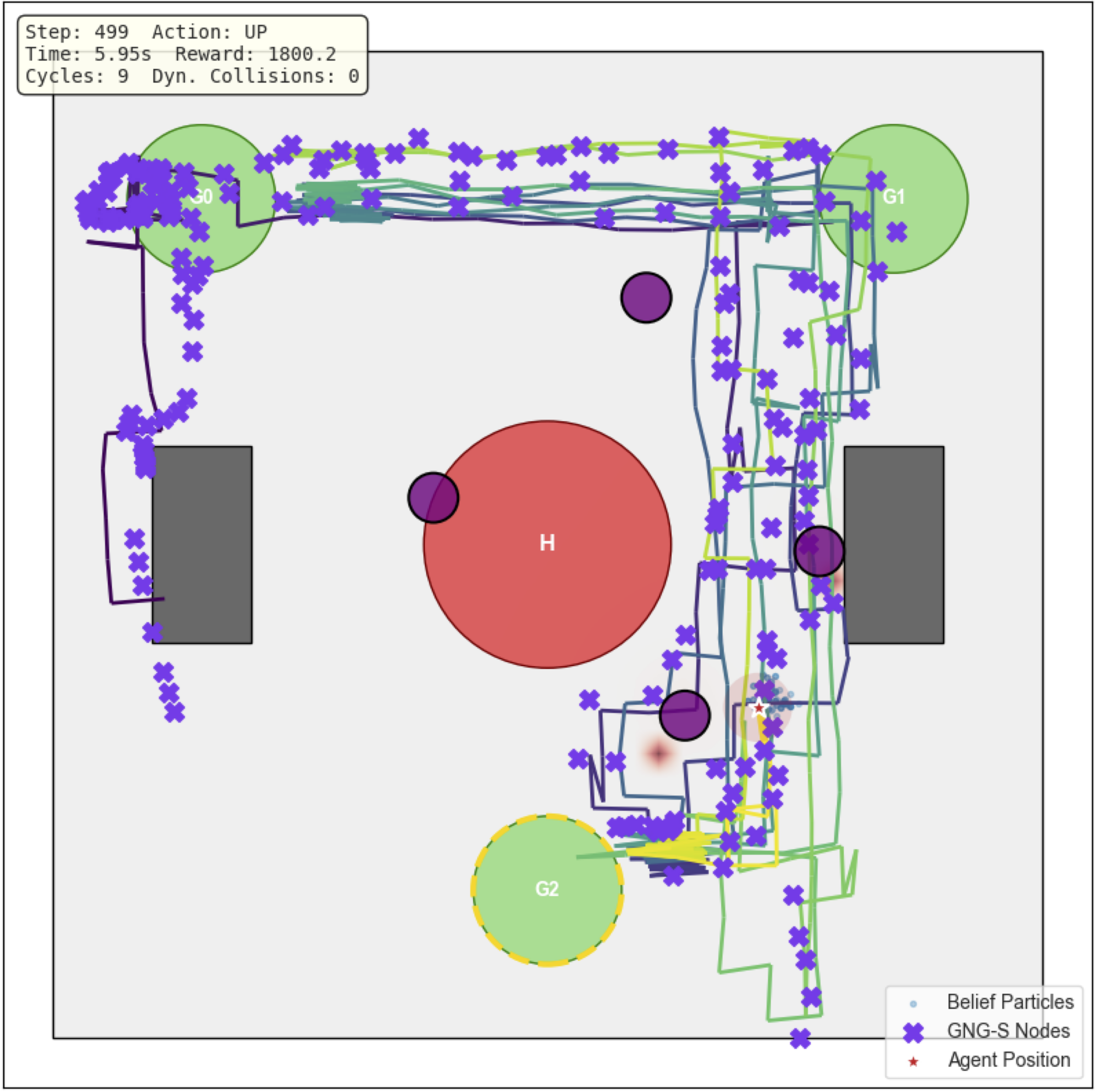}
        \caption{Dynamic 2D}
        \label{fig_dynamic_result}
    \end{subfigure}
    \hfill
    \begin{subfigure}[t]{0.24\textwidth}
        \centering
        \includegraphics[width=\linewidth,height=\figfivepanelheight,keepaspectratio]{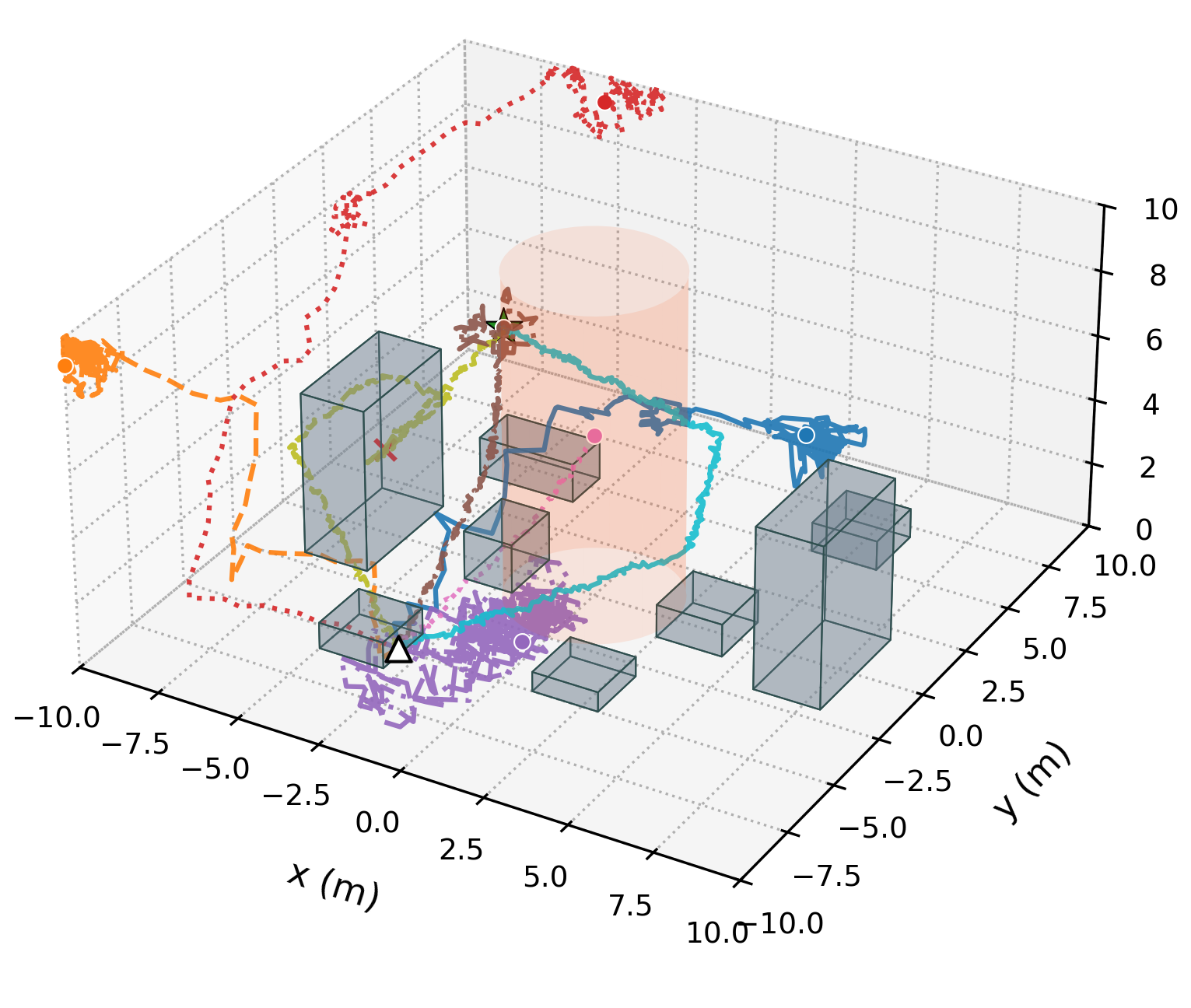}
        \caption{Motion Trajectories of Each Baseline}
        \label{fig:main_3d_baselines}
    \end{subfigure}
    \hfill
    \begin{subfigure}[t]{0.24\textwidth}
        \centering
        \includegraphics[width=\linewidth,height=\figfivepanelheight,keepaspectratio]{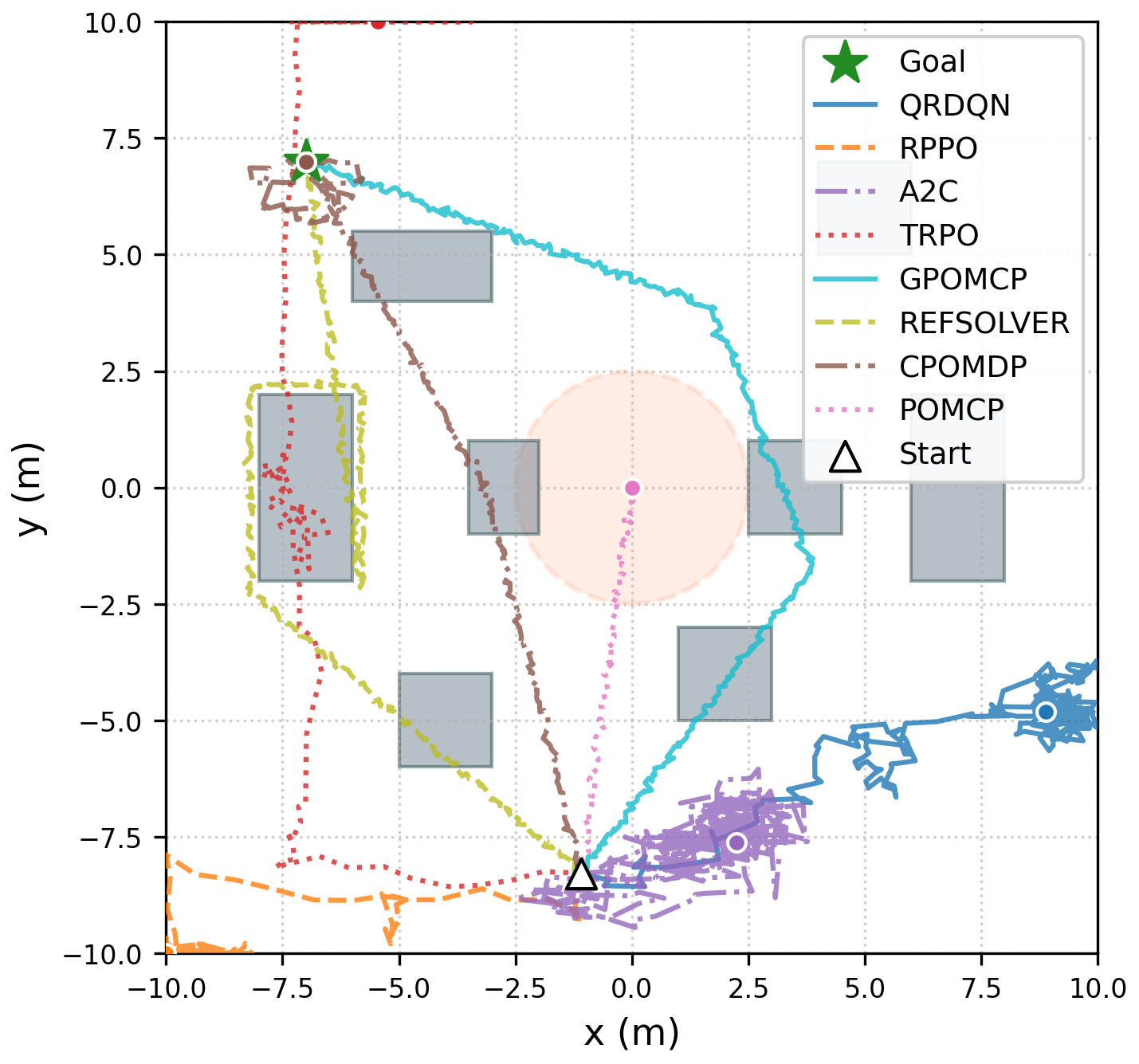}
        \caption{Top-down view in 3D scenario}
        \label{fig:main_3d_topdown}
    \end{subfigure}
    \caption{\revRefFour{Compact visual overview of nominal experiments: Static 2D, Dynamic 2D, baseline trajectories in the 3D navigation scenario, and the corresponding top-down 3D view.}}
    \label{Expe_results_1}
\end{figure*}

\revRefFour{\textbf{2D Persistent Patrolling.} The agent moves in a $20\times20$ continuous arena over $[-10,10]\times[-10,10]$. The $\omega$-regular specification requires repeated visits to three circular goals centered at $(-7,7)$, $(7,7)$, and $(0,-7)$, while avoiding the central circular hazard at $(0,0)$ and two static rectangular obstacles. The state is $x\in\mathbb{R}^2$, the initial position is sampled from $[-1,1]\times[-10,8]$, the action set is $\{(0,\pm1),(\pm1,0)\}$, and transitions follow $x' \sim \mathcal N(x+a,\sigma_t^2 I)$ with $\sigma_t=0.1$. Static 2D uses Gaussian coordinate observations with $\sigma_o=0.5$. Dynamic 2D adds four mobile obstacles and a sensor range of $5.0$; the agent maintains a risk map over obstacle locations through diffusion and temporal decay. Fig.~\ref{Expe_results_1}(a,b) illustrates the 2D patrolling layouts and representative revealed-anchor trajectories, while Fig.~\ref{Expe_results_1}(c,d) gives compact 3D trajectory views used to interpret the nominal results.}

\revRefFour{\textbf{3D Navigation.} The agent operates in $[-10,10]\times[-10,10]\times[0,10]$ with six motion primitives $\{\pm x,\pm y,\pm z\}$, step size $1.0$, transition noise $\sigma_t=0.15$, and observation noise $\sigma_o=0.5$. The reach and avoid objective starts the agent in $[-2,2]\times[-9,-7]\times[1,3]$, requires reaching a goal centered at $(-7,7,3)$, and forbids entering a cylindrical hazard or colliding with eight cuboid obstacles.}

\revRefFour{The reported metrics remain aligned with the logical tasks: \textbf{Cycles}, the number of completed patrol cycles per episode of 500 steps in 2D; \textbf{Success Rate}, the fraction of episodes that complete one 2D cycle or reach the 3D goal without safety violations; and \textbf{\revRefThree{Time to First Cycle (TTFC)}}, measured as the time to first completed cycle in 2D or first goal in 3D among successful episodes.}

\revRc{For reward-based baselines and reward-guided search variants, the specifications are compiled into dense proxy rewards to make training and stepwise planning feasible. The reward-trained DRL baselines are QR-DQN, Recurrent PPO, A2C, and TRPO implemented through the Stable-Baselines3 ecosystem~\cite{stable-baselines3}; Table~\ref{tab:nominal_complete} reports these rows alongside the online POMDP solvers across the three nominal scenarios. The full reward definitions, rationale, DRL training budgets, hyperparameters, and training curves are reported in Sec.~\ref{supp:reward_rationale}, Sec.~\ref{supp:hyperparams}, and Sec.~\ref{supp:drl_curves}.}

\revRefFour{To quantify variability, each method undergoes 20 independent trials per nominal scenario: 500 steps for 2D and 300 steps for 3D. Table~\ref{tab:nominal_complete} reports mean $\pm$ standard deviation where defined, while Table~\ref{tab:nominal_stats_main} reports aggregate 95\% confidence intervals for cycle differences and Wilson 95\% confidence intervals for success counts. Experiments run on a 12-core Intel Xeon Platinum 8352V CPU server with 90GB RAM, Python 3.10, PyTorch 2.3.1, and NumPy 1.26.4.}

\revRefFour{REBA uses 300 MCTS simulations ($c_{\text{UCT}}\!=\!10$) with progressive widening ($k_a\!=\!5.0,\alpha_a\!=\!0.1$; $k_o\!=\!3,\alpha_o\!=\!0.6$), GNG capped at 100 nodes, a dynamic entropy threshold at the 5th percentile over a 50-step window with $\kappa_{\text{gain}}\!=\!0.01$, and Dirichlet-smoothed transition learning ($\alpha\!=\!0.01$, $N_{\min}\!=\!20$); Sec.~\ref{supp:hyperparams} reports parameter sweeps for the reveal percentile, information-gain filter, and entropy-drop threshold. Hyperparameter ranges, validation task, and final settings for all methods are also reported there.}

\subsection{Mechanism Attribution on Continuous LightDark}
\label{sec:mechanism_attribution}
\label{sec:lightdark_main}

\revRc{Before evaluating the complete system across multiple scenarios, we first isolate the role of each architectural component on a recurrent-visit variant of the published Continuous LightDark benchmark~\cite{sunberg2018online,Fischer_IPFT_2020}. This benchmark keeps the published dynamics, observation model, and noise levels and reformulates the task specification as a recurrent visit with safety, so the evaluation exercises the $\omega$-regular layer on external dynamics. Table~\ref{tab:lightdark_main} reports effect sizes and per-comparison $p$-values from paired 30-seed tests against Full REBA. The ablation rows test the functional decomposition of the revelation-controlled architecture.}

\revRc{The pattern isolates four roles. \textbf{(i) Revelation gate under benign conditions:} Ungated anchors reaches $8.70$ cycles compared with Full REBA's $9.23$ ($p=0.0038$), while preserving perfect safety and success. This benchmark uses Gaussian observations and smooth dynamics, so most beliefs remain well-concentrated under both variants. The gate's value emerges under stress: the assumption diagnostics in Table~\ref{tab:supp_assumption_diagnostics} show that dip rejection rate rises from $4.6\%$ (nominal Static-2D) to $18.7\%$ (jump/heavy-tailed noise) and the effective sample size ratio $\mathrm{ESS}/N_p$ drops from $0.61$ to $0.24$, indicating that the gate rejects a substantially larger fraction of candidate anchors when belief quality degrades. \textbf{(ii) Adaptive memory as the enabler of recurrent visits:} Frozen GNG reaches $1.97$ cycles, a $79\%$ drop from Full REBA ($p<10^{-4}$), while preserving safety, isolating online coverage growth as the mechanism behind repeated task completion. \textbf{(iii) Learned dynamics and parity feedback as safety drivers:} Fixed prior yields $18/30$ success and $0.43$ safety violations per episode ($p<10^{-4}$ for all metrics); No parity yields $10/30$ success and $1.43$ safety violations per episode ($p=1.4\times10^{-8}$ on McNemar). \textbf{(iv) MCTS-only as the unguided limit:} MCTS-only yields $2.73$ cycles and $28/30$ success, establishing the baseline against which each architectural component is measured. Complete CIs, timings, all baselines, and test details remain in Sec.~\ref{supp:lightdark_recurrent}.}

\begin{table}[!ht]
\centering
\caption{\revRefFour{Comparison on a recurrent-visit task based on Continuous LightDark~\cite{sunberg2018online}. Rows report representative external baselines and key REBA ablations over $30$ paired seeds; the $p$-value columns report per-comparison $p$-values from paired tests against Full REBA. Complete CIs, timing details, tests, and the full baseline set are in Table~\ref{tab:supp_lightdark_recurrent}.}}
\label{tab:lightdark_main}
\scriptsize
\setlength{\tabcolsep}{0.8pt}
\renewcommand{\arraystretch}{1.08}
\resizebox{\columnwidth}{!}{%
\begin{tabular}{@{}lccccccc@{}}
\toprule
Method & Cycles & Success & Safety violations & TTFC & $p_{\rm cycles}$ & $p_{\rm success}$ & $p_{\rm TTFC}$ \\
\midrule
\textbf{Full REBA} & \textbf{9.23} & \textbf{30/30} & \textbf{0.00} & \textbf{84.6} & reference & reference & reference \\
POMCP & 1.2 & 2/30 & 2.50 & 240.0 & \textbf{$<10^{-4}$} & \textbf{$<10^{-7}$} & \textbf{$<10^{-4}$} \\
CPOMDP & 3.4 & 24/30 & 0.20 & 153.0 & \textbf{$<10^{-4}$} & \textbf{0.031} & \textbf{$<10^{-4}$} \\
REFSOLVER & 2.8 & 14/30 & 1.00 & 170.0 & \textbf{$<10^{-4}$} & \textbf{$<10^{-4}$} & \textbf{$<10^{-4}$} \\
Recurrent PPO & 2.5 & 22/30 & 0.20 & 178.5 & \textbf{$<10^{-4}$} & \textbf{0.008} & \textbf{$<10^{-4}$} \\
\midrule
Ungated anchors & 8.70 & 30/30 & 0.00 & 89.4 & \textbf{0.0038} & 1.000 & 0.390 \\
Frozen GNG & 1.97 & 30/30 & 0.00 & 98.2 & \textbf{$<10^{-4}$} & 1.000 & 0.060 \\
Fixed prior & 2.4 & 18/30 & 0.43 & 152.0 & \textbf{$<10^{-4}$} & \textbf{$<10^{-4}$} & \textbf{$<10^{-4}$} \\
No parity & 6.50 & 10/30 & 1.43 & 129.1 & \textbf{$<10^{-4}$} & \textbf{$1.4{\times}10^{-8}$} & \textbf{$<10^{-4}$} \\
MCTS-only & 2.73 & 28/30 & 0.07 & 193.6 & \textbf{$<10^{-4}$} & 0.500 & \textbf{$<10^{-4}$} \\
\bottomrule
\end{tabular}
}
\end{table}

\FloatBarrier

\subsection{Nominal Online Performance}
\label{sec:nominal_performance}

\revRefFour{Table~\ref{tab:nominal_complete} reports the complete nominal results across all three scenarios, with reward-trained DRL and online POMDP solver rows shown consistently for Static-2D, Dynamic-2D, and 3D Navigation. Table~\ref{tab:nominal_stats_main} reports strongest-baseline confidence intervals and per-comparison $p$-values for the within-table comparisons. Together, these tables support three observations.}

\begin{table}[!t]
\centering
\scriptsize
\caption{\revRefFour{Compact nominal statistical summary. \revRefThree{CI denotes confidence interval.} The $p$-value column reports per-comparison $p$-values for the REBA versus CPOMDP comparisons in this table and is interpreted together with the effect sizes and confidence intervals.}}
\label{tab:nominal_stats_main}
\setlength{\tabcolsep}{2.0pt}
\renewcommand{\arraystretch}{1.15}
\resizebox{\columnwidth}{!}{%
\begin{tabular}{@{}lllll@{}}
\toprule
Scenario & Metric & REBA [95\% CI] & CPOMDP [95\% CI] & $p$ \\
\midrule
Static 2D & Cycles & $11.0$ [$9.97$, $12.03$] & $9.4$ [$8.18$, $10.62$] & $0.043$ \\
Static 2D & Success & $19/20$ [$0.76$, $0.99$] & $18/20$ [$0.70$, $0.97$] & $1.000$ \\
Static 2D & TTFC & $34$ [$28.22$, $39.78$] & $45$ [$39.53$, $50.47$] & $0.006$ \\
Dynamic 2D & Cycles & $11.2$ [$10.22$, $12.18$] & $7.6$ [$6.48$, $8.72$] & $1.19\!\times\!10^{-5}$ \\
Dynamic 2D & Success & $19/20$ [$0.76$, $0.99$] & $16/20$ [$0.58$, $0.92$] & $0.342$ \\
Dynamic 2D & TTFC & $42$ [$28.50$, $55.50$] & $57$ [$36.22$, $77.78$] & $0.210$ \\
3D navigation & Success & $19/20$ [$0.76$, $0.99$] & $15/20$ [$0.53$, $0.89$] & $0.182$ \\
3D navigation & TTFC & $62$ [$52.36$, $71.64$] & $97$ [$72.63$, $121.37$] & $0.010$ \\
\bottomrule
\end{tabular}
}
\par\vspace{0.25em}
\footnotesize\revRefFour{Note: REBA is compared with CPOMDP, the strongest online POMDP baseline in each nominal scenario. Continuous metrics use 95\% confidence intervals; success counts use Wilson 95\% confidence intervals.}
\end{table}

\revRefFour{First, REBA's advantage over the strongest online POMDP baseline (CPOMDP) grows with scenario difficulty: the cycle gain widens from $+17.0\%$ in Static-2D to $+47.4\%$ in Dynamic-2D, where mobile obstacles and risk-map updates stress local search, while the relative success gain in 3D navigation reaches $+26.7\%$ ($95\%$ vs.\ $75\%$). Second, in the scenarios where sample sizes permit, the TTFC comparisons have small per-comparison $p$-values ($p=0.006$ Static, $p=0.010$ 3D), supporting the measured reading that symbolic guidance accelerates task completion alongside final success. Third, the reward-trained DRL rows show the limit of dense proxy rewards in these logical tasks: TRPO ties Static-2D success but trails in cycles ($7.7$ vs.\ $11.0$), and the best Dynamic-2D and 3D DRL rows reach $55\%$ and $35\%$ success, respectively. Fig.~\ref{Expe_results_1} provides compact 2D and 3D trajectory views; enlarged 3D trajectory visualizations remain in Sec.~\ref{supp:drone_3d_trajectories}.}

\revRefFour{The nominal gaps are consistent with the role of revelation-controlled symbolic guidance in cyclic objectives: REBA keeps the $\omega$-regular specification explicit through parity synthesis on the product MDP, while reward-based baselines rely on dense surrogate rewards whose relationship to logical satisfaction is indirect. The full reward definitions, hyperparameter protocols, and DRL training curves are reported in Secs.~\ref{supp:reward_rationale} to~\ref{supp:drl_curves}.}

\FloatBarrier

\subsection{Robustness Beyond Nominal Settings}
\label{sec:robustness_main}

\revRefFour{\textbf{Stress tests with fixed hyperparameters.} Table~\ref{tab:stress_main} summarizes stress tests under observation and dynamics shifts, with all hyperparameters left at their nominal values. Two patterns are notable. First, REBA's margins over CPOMDP \emph{grow} under stress ($+25.0\%$ to $+70.5\%$) relative to the nominal setting ($+17.0\%$ to $+47.4\%$), which is the opposite of what overfitting to nominal conditions would produce. Second, the advantage is larger under observation stress than under dynamics stress in both scenarios, consistent with the revelation gate's reliance on belief concentration: observation degradation directly challenges the entropy test, while dynamics shifts primarily affect model fidelity (discussed further in Sec.~\ref{sec:discussion_robustness}). Complete standard deviations and TTFC values are in Tables~\ref{tab:robust_static_2d} and~\ref{tab:robust_dynamic_2d}.}

\begin{table}[!ht]
\centering
\caption{\revRefFour{Compact fixed-configuration stress-test summary (20 trials per setting). REBA and CPOMDP cycle counts are reported as mean $\pm$ SD. The baseline is CPOMDP in all rows; full TTFC values are in Tables~\ref{tab:robust_static_2d} and~\ref{tab:robust_dynamic_2d}.}}
\label{tab:stress_main}
\footnotesize
\setlength{\tabcolsep}{2.5pt}
\renewcommand{\arraystretch}{1.10}
\begin{tabular}{@{}lcccc@{}}
\toprule
Setting & REBA cycles & CPOMDP cycles & Margin & REBA success \\
\midrule
Static Set-Obs & $\mathbf{9.0\pm2.0}$ & $7.0\pm2.2$ & \textbf{+28.6\%} & 85\% \\
Static Set-Dyn & $\mathbf{10.0\pm1.8}$ & $8.0\pm2.1$ & \textbf{+25.0\%} & 90\% \\
Dynamic Set-Obs & $\mathbf{7.5\pm2.6}$ & $4.4\pm2.0$ & \textbf{+70.5\%} & 80\% \\
Dynamic Set-Dyn & $\mathbf{8.2\pm2.2}$ & $5.8\pm2.2$ & \textbf{+41.4\%} & 90\% \\
\bottomrule
\end{tabular}
\end{table}

\FloatBarrier

\begin{figure*}[!t]
\centering
\includegraphics[width=0.86\textwidth]{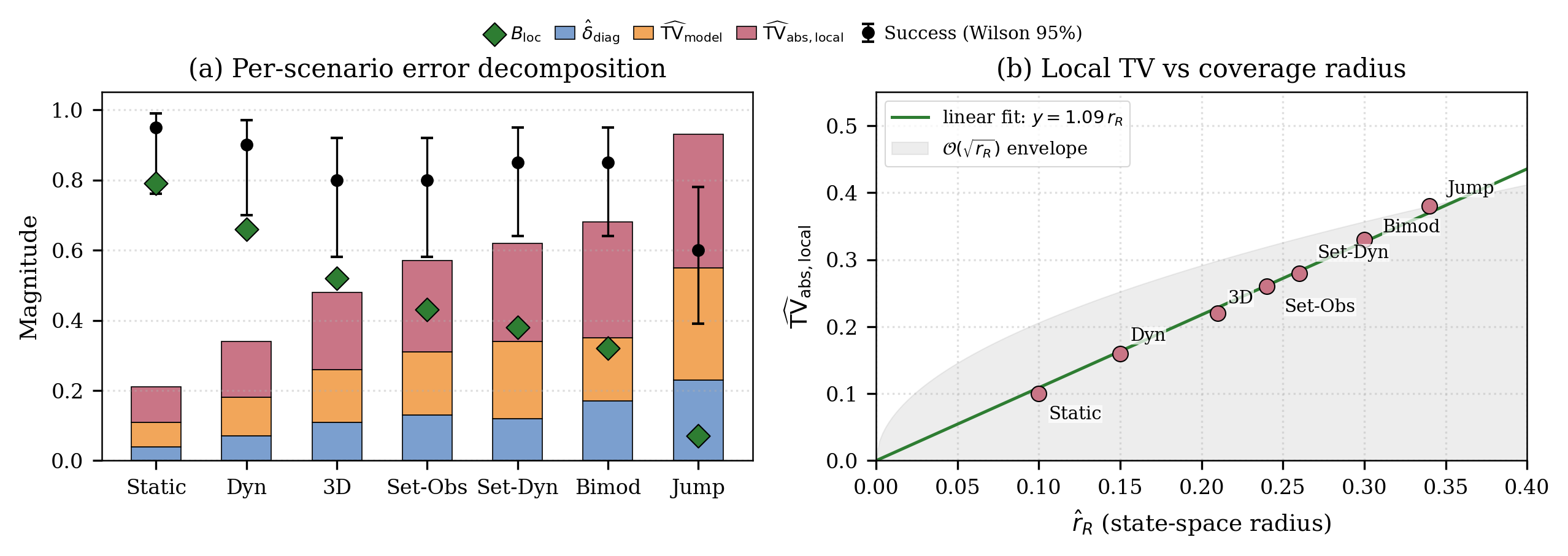}
\caption{\revRefFour{Empirical Theorem~\ref{thm:performance_guarantee} error decomposition. (a) Stacked measured local terms $\hat\delta_{\rm diag}$, $\widehat{\mathrm{TV}}_{\rm model}$, and $\widehat{\mathrm{TV}}_{\rm abs,local}$ over seven aggregated categories formed from the nine measured settings in Table~\ref{tab:supp_thm3_error_proxies}; Set-Obs and Set-Dyn each aggregate Static-2D and Dynamic-2D stress rows, and Jump denotes the deliberate heavy-tailed-noise and abrupt-dynamics stress case. Green diamonds show $B_{\rm loc}$ and black circles show empirical success with Wilson 95\% intervals. (b) Total-variation proxy for local abstraction versus the measured coverage radius $\hat r_R$; the green line is a least-squares fit through the origin and the gray envelope shows the normalized $\mathcal{O}(\sqrt{\hat r_R})$ trend. Table~\ref{tab:supp_thm3_error_proxies} reports all nine measured settings, including $B_{\rm wc}$ and the worst case proxy terms.}}
\label{fig:thm3_decomposition_main}
\end{figure*}

\section{\texorpdfstring{\revRefFour{Discussion}}{Discussion}}
\label{sec:discussion}

\revRc{The preceding experiments test whether the revelation-to-guidance mechanism behaves as intended. Reliable revelations control which beliefs become symbolic states, adaptive memory supports recurrent task progress, learned abstract dynamics estimate how those states evolve, and parity feedback returns the long-horizon objective to MCTS. This section examines the operating regime of that mechanism through measured error terms, assumption diagnostics, boundary behavior, and computational cost.}

\revRefFour{The mechanism-attribution results in Sec.~\ref{sec:mechanism_attribution} clarify how the revelation-to-guidance loop relates to the error terms in Theorem~\ref{thm:performance_guarantee}. The revelation gate and GNG determine whether admitted symbolic states are trustworthy and sufficiently covered, thereby controlling the abstraction error $\varepsilon_{\rm abs}$. The learned transition kernel with Dirichlet smoothing determines whether the finite memory predicts progress accurately; the Fixed prior ablation raises the empirical model-TV proxy to levels associated with $18/30$ success, illustrating the role of model error $\varepsilon_{\rm model}$. The parity feedback on the product MDP supplies the objective-level signal that local search lacks; the No parity ablation reduces success to $10/30$, showing that symbolic progress feedback is a dominant safety driver. The MCTS-only row gives the unguided reference for the complete revelation-to-guidance loop. The following subsections examine the measured magnitude of each error term, the assumptions that govern reliability, and the computational cost of maintaining this loop online.}

\subsection{\texorpdfstring{\revRefFour{Quantitative Magnitude of Performance Bounds}}{Quantitative Magnitude of Performance Bounds}}
\label{sec:discussion_bounds}

\revRefFour{Theorem~\ref{thm:performance_guarantee} decomposes the satisfaction gap into learning, model, and abstraction error terms. Fig.~\ref{fig:thm3_decomposition_main} reports measured local proxies for these terms across seven aggregated diagnostic categories, with the full nine measured settings listed in Table~\ref{tab:supp_thm3_error_proxies}; the aggregation keeps the figure readable by combining the Static-2D and Dynamic-2D rows within Set-Obs and Set-Dyn, while preserving the nominal, bimodal-prior, and jump/noise stress cases as separate categories. Two observations carry practical implications. First, the local abstraction total-variation proxy $\widehat{\mathrm{TV}}_{\rm abs,local}$ is approximately proportional to the measured coverage radius $\hat r_R$, with an empirical ratio in $[1.00, 1.12]$, consistent with the $\mathcal{O}(\sqrt{\varepsilon_{\rm map}})$ scaling predicted by the theorem. This means that the coverage radius, which is directly monitored during online operation, serves as a practical proxy for abstraction quality: maintaining a small coverage radius keeps the abstraction-error term controlled. Second, the worst case diagnostic $B_{\rm wc}$ shows the conservative scale of the theorem-structured certificate: the nominal values are $0.62$, $0.49$, and $0.35$ for Static-2D, Dynamic-2D, and 3D-Navigation, respectively, and the stress and diagnostic rows range from $0.02$ to $0.36$. The local diagnostic score $B_{\rm loc}$ ranges from $0.07$ to $0.79$ and remains below the empirical success rate across all nine measured settings, with the smallest value in the deliberate heavy-tailed-noise and abrupt-dynamics stress case, where empirical success is also lowest ($12/20$). The diagnostic terms thus co-vary with task difficulty in the expected direction, providing a measurable signal for when symbolic guidance loses reliability.}

\subsection{\texorpdfstring{\revRefFour{Robustness to Assumption Stress}}{Robustness to Assumption Stress}}
\label{sec:discussion_robustness}
\label{sec:limitations}

\revRc{The revelation gate is conservative by design: it admits a belief into the symbolic memory only when the entropy and information-gain tests are jointly satisfied. This conservatism defines the behavior at the aliasing boundary. Under severe sensor aliasing or persistent symmetry, the posterior remains multimodal and the gate withholds anchors; symbolic guidance is therefore withheld, and the planner behavior approaches the unguided MCTS reference until ambiguity is resolved. To isolate this boundary, Sec.~\ref{supp:particle_stability} uses a forced mirror-alias pilot, a seven-seed diagnostic whose observation map preserves the sign ambiguity in the latent $x$ branch. In that diagnostic, Full REBA reports no accepted anchors and consequently no GNG or \revRefThree{Belief Automaton} updates; task outcomes approach MCTS-only levels. This boundary behavior ties symbolic structure to reliable beliefs. The increasing dip rejection rates and decreasing effective sample size ($\mathrm{ESS}/N_p$) values reported in Table~\ref{tab:supp_assumption_diagnostics} and Fig.~\ref{fig:supp_assumption_diagnostics} under deliberate stress (bimodal prior, heavy-tailed noise) provide a quantitative signature of this boundary.}

\revRefFour{The second stress mode, model mismatch under heavy-tailed noise and abrupt dynamics changes, operates through a different mechanism. The Jump/noise row reports $\widehat{\mathrm{TV}}_{\rm model}=0.32$ and $\mathrm{ESS}/N_p=0.24$, while the planner kernel remains the nominal Gaussian model; the H3 boundary (non-Lipschitz transition discontinuities) is a distinct regime. The soft-bias integration of the abstract policy into MCTS allows online replanning to absorb part of the model mismatch, producing gradual degradation, and the diagnostic terms identify the boundary where symbolic guidance becomes less reliable.}

\subsection{\texorpdfstring{\revRefFour{Computational Complexity and Scalability}}{Computational Complexity and Scalability}}
\label{sec:discussion_scaling}

\begin{figure}[!t]
\centering
\includegraphics[width=\columnwidth]{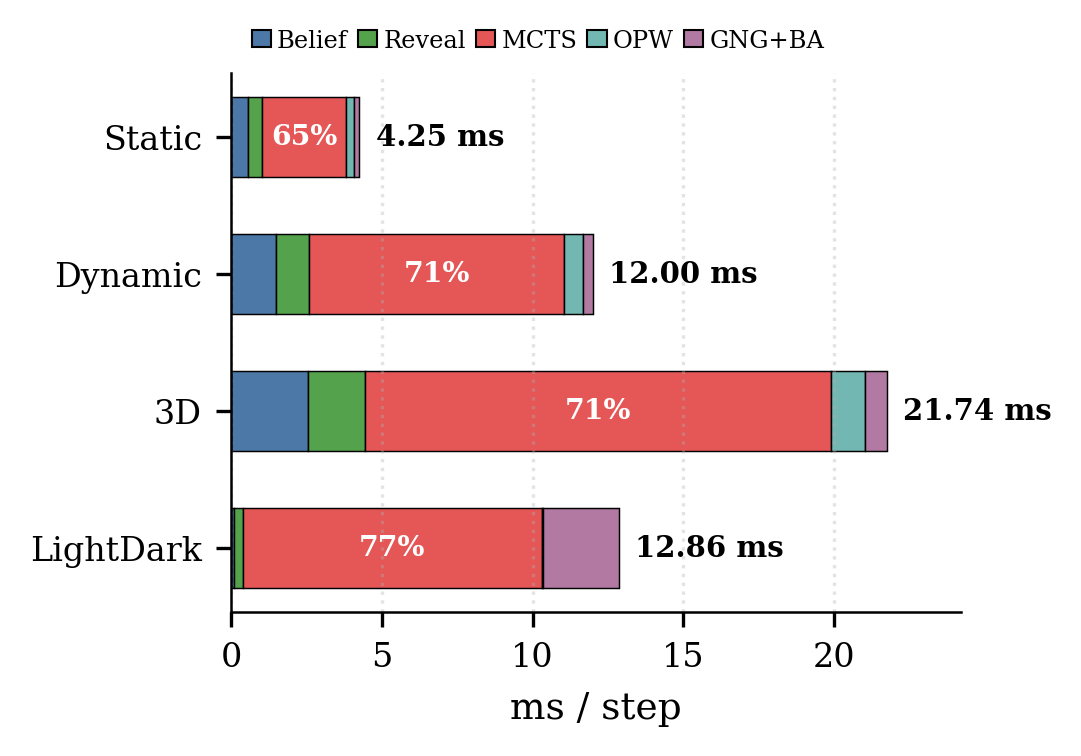}
\caption{\revRefFour{Single-column view of REBA's mean online cost per step, grouped from the per-component measurements in Table~\ref{tab:supp_complexity}. OPW denotes observation progressive widening.}}
\label{fig:per_step_cost_main}
\end{figure}

\revRefFour{\textbf{Computational cost breakdown.} Fig.~\ref{fig:per_step_cost_main} and Table~\ref{tab:supp_complexity} show that the symbolic-layer components (reveal statistic, GNG update, \revRefThree{Belief Automaton} update, and product rebuild) account for $11.9\%$ to $15.1\%$ of per-step online planning time, while MCTS simulation accounts for $65.4\%$ to $71.1\%$. Total online costs are $4.25\pm0.52$ ms (Static-2D), $12.00\pm1.15$ ms (Dynamic-2D), $21.74\pm2.20$ ms (3D-Navigation), and $12.86$ ms (Continuous LightDark Full REBA). The practical implication is that adding $\omega$-regular symbolic guidance to an existing MCTS planner costs roughly $12\%$ to $15\%$ overhead, a modest computational premium relative to the mechanism gains demonstrated in Sec.~\ref{sec:mechanism_attribution}.}

\revRefFour{\textbf{Runtime and statistical scaling.} At fixed $N_p$ and $N_{\rm sim}$, the measured per-component runtime scales roughly linearly with state dimension $d_S$ (Sec.~\ref{supp:complexity}), so the implementation-level cost remains predictable in the tested range. The harder scaling constraint comes from belief statistics: the consistency requirement for the weighted SMC and KDE implementation underlying Theorem~\ref{thm:robust_revelation}, together with the $(2\pi e)^{d_S}$ factor in the Gaussian volume bound, implies more conservative reveal thresholds and larger particle budgets in higher dimensions. The current implementation is therefore most appropriate for the low-to-moderate dimensional continuous state spaces studied here; extending the approach to higher-dimensional systems, where the Gaussian-like belief condition may itself become restrictive, calls for learned belief representations or neural generative belief models beyond the current particle filter.}

\section{Conclusion}
\label{sec:conclusion}

\revRefFour{This paper introduced REBA, an event-driven framework that bridges continuous belief-space planning and finite-state symbolic reasoning under $\omega$-regular specifications through online certification of revelation events. Information-theoretic gates extract reliable anchors from noisy beliefs, and these anchors enable the online construction of a finite Belief Automaton for parity policy synthesis without a predefined discrete abstraction. The synthesized parity policy provides symbolic guidance to MCTS, steering online search beyond its local horizon. We characterize its reliability through a learning, model, and abstraction error decomposition for the underlying continuous POMDP. Empirically, anchor certification, adaptive memory growth, learned abstract dynamics, and parity feedback each contribute materially to performance, yielding primary metric gains of +17.0\% to +47.4\% over the strongest online baseline with 11.9\% to 15.1\% symbolic-layer overhead.}

\revRefFour{The operating envelope of REBA is governed by belief quality, and the certification principle makes this boundary explicit rather than silent. When observations concentrate the posterior, revelation events are certified and symbolic guidance is active; under persistent aliasing or heavy-tailed noise, the gates withhold anchors and the planner degrades gracefully toward unguided MCTS until ambiguity resolves. This conservatism defines the method's scope: REBA is most effective when stable anchors can be certified reliably from the belief stream. Extending the envelope calls for richer belief representations, including learned or neural generative belief models, and for physical deployments under persistent aliasing and realistic sensing failures.}


\bibliographystyle{IEEEtran}
\bibliography{refs}

\hypersetup{pageanchor=true}
\REBASavedBibliographystyle{IEEEtran}
\REBASavedBibliography{refs}

\clearpage
\hypersetup{pageanchor=false}


\setcounter{page}{1}
\setcounter{section}{0}
\setcounter{subsection}{0}
\setcounter{subsubsection}{0}
\setcounter{table}{0}
\setcounter{figure}{0}
\setcounter{equation}{0}
\setcounter{theorem}{0}
\setcounter{lemma}{0}
\setcounter{corollary}{0}
\setcounter{definition}{0}
\setcounter{assumption}{0}
\setcounter{remark}{0}
\setcounter{proposition}{0}
\setcounter{paragraph}{0}
\setcounter{footnote}{0}

\renewcommand{\thesection}{S\arabic{section}}
\renewcommand{\thesubsection}{\thesection.\arabic{subsection}}
\renewcommand{\thesubsubsection}{\thesubsection.\arabic{subsubsection}}
\renewcommand{\thetable}{S\arabic{table}}
\renewcommand{\thefigure}{S\arabic{figure}}
\renewcommand{\theequation}{S\arabic{equation}}
\renewcommand{\thetheorem}{S\arabic{theorem}}
\renewcommand{\thelemma}{S\arabic{lemma}}
\renewcommand{\thecorollary}{S\arabic{corollary}}
\renewcommand{\thedefinition}{S\arabic{definition}}
\renewcommand{\theassumption}{S\arabic{assumption}}
\renewcommand{\theremark}{S\arabic{remark}}
\renewcommand{\theproposition}{S\arabic{proposition}}

\renewcommand{\theHsection}{S\arabic{section}}
\renewcommand{\theHsubsection}{\theHsection.\arabic{subsection}}
\renewcommand{\theHsubsubsection}{\theHsubsection.\arabic{subsubsection}}
\renewcommand{\theHtable}{S\arabic{table}}
\renewcommand{\theHfigure}{S\arabic{figure}}
\renewcommand{\theHequation}{S\arabic{equation}}
\renewcommand{\theHtheorem}{S\arabic{theorem}}
\renewcommand{\theHlemma}{S\arabic{lemma}}
\renewcommand{\theHcorollary}{S\arabic{corollary}}
\renewcommand{\theHdefinition}{S\arabic{definition}}
\renewcommand{\theHassumption}{S\arabic{assumption}}
\renewcommand{\theHremark}{S\arabic{remark}}
\renewcommand{\theHproposition}{S\arabic{proposition}}
\providecommand{\theHpage}{\arabic{page}}
\renewcommand{\theHpage}{S\arabic{page}}
\renewcommand{\theHparagraph}{S\theHsubsubsection.\arabic{paragraph}}
\renewcommand{\theHfootnote}{S\arabic{footnote}}
\makeatletter
\renewcommand{\thesectiondis}{\thesection.}
\renewcommand{\thesubsectiondis}{\thesubsection.}
\renewcommand{\thesubsubsectiondis}{\thesubsubsection.}
\makeatother

\twocolumn[
\begin{center}
{\LARGE\bfseries Supplementary Material for}\par\vspace{0.6em}
{\Large\bfseries ``REBA: A Revealed Belief Automaton Framework\\for Online Planning in Continuous POMDPs''}
\end{center}
\vspace{1.2em}
]

\section{Proofs for Online Revelation via Belief Concentration}
\label{supp:proofs_revelation}

We score belief concentration with \emph{Boers' entropy} $H_{\text{Boers}}(\bbel)$ and pair this concentration score with the gain and drop checks used by the implemented revelation rule. Its suitability follows from its link to the KL divergence. With posterior $\bbel(x' \mid h,a,z)$ and mixture prior
\begin{equation}
    C(x' \mid h,a)=\int p_T(x' \mid x,a)\, \bbel_h(x)\,\mathrm{d}x ,
\end{equation}
Lemma~\ref{lem:boers_identity} gives the decomposition below.

\begin{lemma}[Boers' entropy identity under mild regularity]
\label{lem:boers_identity}
Let $(\mathcal{X},\mathcal{B},\lambda)$ be a measurable space where $\lambda$ is a reference measure (e.g., Lebesgue on $\mathbb{R}^{d_S}$). Assume the posterior belief $\bbel$ and the mixture prior $C$ admit densities $b(\cdot)$ and $C(\cdot)$ with respect to $\lambda$, and $\bbel\ll C$. Define
\begin{equation}
    H_{\text{Boers}}(\bbel)\;:=\;H_s(b)\;-\;\mathbb{E}_{x'\sim b}\big[\log C(x')\big],
\end{equation}
where $H_s(b):=-\!\int b(x')\log b(x')\,\mathrm{d}\lambda(x')$ is the differential entropy. Then
\begin{equation}
\label{supp:Lam1:eq1}
    H_{\text{Boers}}(\bbel) \;=\; 2\,H_s(b) \;+\; D_{\mathrm{KL}}\!\left(b \,\Vert\, C\right).
\end{equation}
\end{lemma}

\begin{proof}
The identity follows directly from the standard decomposition of cross-entropy, $H_{\times}(b,C)=H_s(b)+D_{\mathrm{KL}}(\bbel\Vert C)$, by noting that $H_{\text{Boers}}(\bbel)=H_s(b)+H_{\times}(b,C)$.
\end{proof}

Absolute continuity $\bbel\!\ll\!C$ ensures $D_{\mathrm{KL}}(\bbel\Vert C)\!<\!\infty$ (otherwise $H_{\text{Boers}}=+\infty$). Here, $H_s(b)$ measures posterior spread and $D_{\mathrm{KL}}(\bbel\Vert C)$ is the information gain from $C$ to $b$. Thus $H_{\text{Boers}}=2H_s+D_{\mathrm{KL}}$ jointly captures \emph{uncertainty} and \emph{informativeness}; we also enforce $D_{\mathrm{KL}}\!\ge\!\kappa_{\mathrm{gain}}$.

\begin{corollary}[Gaussian posterior gives a volume bound]
\label{cor:gaussian_volume}
Let $b=\mathcal{N}(\mu,\Sigma)$ be the density of a non-degenerate Gaussian belief on $\mathbb{R}^{d_S}$, with covariance matrix $\Sigma\succ 0$. Then $H_s(b) = \tfrac{1}{2}\log\!\big((2\pi e)^{d_S}\det\Sigma\big)$, hence $H_{\text{Boers}}(\bbel) = \log\!\big((2\pi e)^{d_S}\det\Sigma\big) + D_{\mathrm{KL}}(\bbel\Vert C)$. Consequently, for any thresholds $\theta_{\mathrm{reveal}}\in\mathbb{R}$ and $\kappa_{\mathrm{gain}}\ge 0$, if $H_{\text{Boers}}(\bbel)\le\theta_{\mathrm{reveal}}$ and $D_{\mathrm{KL}}(\bbel\Vert C)\ge\kappa_{\mathrm{gain}}$, then
\begin{equation}
    \det\Sigma \;\le\; \frac{\exp\!\big(\theta_{\mathrm{reveal}}-\kappa_{\mathrm{gain}}\big)}{(2\pi e)^{d_S}}.
\end{equation}
\end{corollary}

\begin{proof}
For a non-degenerate Gaussian belief with density $b=\mathcal{N}(\mu,\Sigma)$ on $\mathbb{R}^{d_S}$, the differential entropy is $H_s(b) = \tfrac{1}{2}\log((2\pi e)^{d_S}\det\Sigma)$. Substituting into Lemma~\ref{lem:boers_identity} gives $H_{\text{Boers}}(\bbel) = \log((2\pi e)^{d_S}\det\Sigma)+D_{\mathrm{KL}}(\bbel\Vert C)$. Under the constraints $H_{\text{Boers}}(\bbel)\le\theta_{\mathrm{reveal}}$ and $D_{\mathrm{KL}}(\bbel\Vert C)\ge\kappa_{\mathrm{gain}}$, $\log((2\pi e)^{d_S}\det\Sigma) \le \theta_{\mathrm{reveal}}-\kappa_{\mathrm{gain}}$. Exponentiating yields the bound.
\end{proof}

\begin{theorem}[\revRc{Belief covariance volume bound from the revelation test}]
\label{supp:thm:robust_revelation}
Let $S\subset\mathbb{R}^{d_S}$ be compact. Assume the posterior belief $\bbel$ and the mixture prior $C$ admit densities w.r.t.\ Lebesgue measure, $\bbel\ll C$, and $C$ is continuous, bounded, and bounded away from zero on $S$. Let the true posterior be approximated by the Gaussian fit $\hat{\bbel}=\mathcal{N}(\mu_p,\widehat\Sigma_p)$ obtained from an SMC particle set of size $N_p$ produced by a consistent filter, and assume $\widehat\Sigma_p\to\Sigma_b$ in probability as $N_p\to\infty$. Let $\hat C$ be a KDE with bandwidth $h_{N_p}\to 0$ and $N_p h_{N_p}^{d_S}\to\infty$. For any $\theta_{\mathrm{reveal}}$, $\varepsilon_b>0$, and $\delta\in(0,1)$, there exists a sufficient particle count $N_p^\star=N_p^\star(\varepsilon_b,\delta)$ such that for all $N_p\ge N_p^\star$, if $\hat H_{\mathrm{Boers}}\le \theta_{\mathrm{reveal}}$, then with probability at least $1-\delta$,
\begin{equation}
\label{supp:lam2:eq1}
    \det(\Sigma_b)\ \le\ \frac{\exp(\theta_{\mathrm{reveal}})}{(2\pi e)^{d_S}}\ +\ \varepsilon_b.
\end{equation}
If one also enforces the implemented gain gate \revRefTwo{$D_{\mathrm{KL}}(\hat{\bbel}\Vert \hat C)\ge \kappa_{\mathrm{gain}}\ge 0$}, then with the same probability,
\begin{equation}
\label{supp:lam2:eq1tight}
    \det(\Sigma_b)\ \le\ \frac{\exp(\theta_{\mathrm{reveal}}-\kappa_{\mathrm{gain}})}{(2\pi e)^{d_S}}\ +\ \varepsilon_b.
\end{equation}
\end{theorem}

\begin{proof}
\textbf{Part A: Deterministic bound on $\det(\widehat\Sigma_p)$.} The condition $\hat H_{\text{Boers}} \le \theta_{\text{reveal}}$ and the non-negativity of the KL-divergence in Lemma~\ref{lem:boers_identity} together imply that $2 H_s(\hat{\bbel}) \le \hat H_{\text{Boers}} \le \theta_{\text{reveal}}$. Substituting the formula for the differential entropy of a Gaussian, $H_s(\hat{\bbel}) = \tfrac12 \log\!\big((2\pi e)^{d_S}\det(\widehat\Sigma_p)\big)$, and solving for $\det(\widehat\Sigma_p)$ directly yields the deterministic bound:
\begin{equation} \label{eq:proof_combined_A}
    \det(\widehat\Sigma_p) \;\le\; \frac{\exp(\theta_{\text{reveal}})}{(2\pi e)^{d_S}}.
\end{equation}

\textbf{Part B: From $\det(\widehat\Sigma_p)$ to $\det(\Sigma_b)$ in probability.} Under the theorem's stated consistency assumptions, the Gaussian-fit covariance $\widehat\Sigma_p$ converges in probability to the true covariance $\Sigma_b$. Since the determinant is continuous in the matrix entries, $\det(\widehat\Sigma_p)$ also converges in probability to $\det(\Sigma_b)$. Hence, for any $\varepsilon_b>0$ and $\delta\in(0,1)$, there exists a sufficient sample size $N_p^\star=N_p^\star(\varepsilon_b,\delta)$ such that for all $N_p \ge N_p^\star$:
\begin{equation} \label{eq:proof_combined_B}
    \mathbb{P}\!\left( \det(\Sigma_b) \le \det(\widehat\Sigma_p) + \varepsilon_b \right) \;\ge\; 1 - \delta.
\end{equation}
Combining this high-probability event with the deterministic bound \eqref{eq:proof_combined_A} gives the claimed result \eqref{supp:lam2:eq1}. \revRefTwo{The tighter bound \eqref{supp:lam2:eq1tight} follows by the same logic, since the estimated gain gate gives $2 H_s(\hat{\bbel}) + \kappa_{\mathrm{gain}} \le \hat H_{\mathrm{Boers}}\le \theta_{\mathrm{reveal}}$. An explicit finite-sample rate for $N_p^\star$ under the present weighted-SMC/KDE implementation would require stronger assumptions than those used in this theorem.}
\end{proof}

\section{Proofs for GNG-Based Abstraction and Coverage}
\label{supp:proofs_gng}

\begin{definition}[Core revealed regions and $\varepsilon_{\rm cov}$-coverage]
\label{def:core_regions_app}
Let $\mathcal{R}^\star=\{R_1^\star,\dots,R_M^\star\}\subset\mathcal{X}$ be a finite collection of disjoint, compact core regions. Given a graph $\mathcal{G}=(V,E)$ with prototypes $W(\cdot)$, we say $\mathcal{G}$ \emph{$\varepsilon_{\rm cov}$-covers} $\mathcal{R}^\star$ if for every $j$ there exists $v\in V$ such that $\mathrm{dist}(W(v),R_j^\star) := \inf_{x\in R_j^\star}\|W(v)-x\|_2 \le \varepsilon_{\rm cov}$.
\revRefTwo{This coverage notion is witness-level. It is sufficient for the coverage theorem below; Theorem~\ref{supp:thm:performance_guarantee} separately assumes the $M_{\rm met}$-norm map-radius condition for each revealed representative compared to its assigned prototype.}
\end{definition}

\begin{assumption}[Bounded input (H1)]
\label{assum:h1}
There exists $R_S>0$ such that $\|x_{\mathrm{rev}}^{(i)}\|_2\le R_S$ for all $i\ge 1$.
\end{assumption}

\begin{assumption}[$\tau$-step minorization over uncovered regions (H2)]
\label{assum:h2}
There exist $\tau\ge 1$ and $p_{\min}\in(0,1]$ such that whenever some region in $\mathcal{R}^\star$ is not yet $\varepsilon_{\rm cov}$-covered by the current abstraction, the next $\tau$ revealed vectors contain at least one sample in the currently uncovered set with probability at least $p_{\min}$.
\end{assumption}

\begin{assumption}[\revRefTwo{L1-Lipschitz transition regularity} (H3)]
\label{ass:lipschitz-kernel}
There exists $L_T>0$ such that for all $x,x'\in\mathcal{X}$ and $a\in\mathcal{A}$, $\bigl\| p_T(\cdot \mid x,a) - p_T(\cdot \mid x',a) \bigr\|_1 \le L_T \|x - x'\|_2$, where $\|\cdot\|_1$ denotes the $\ell_1$-norm of signed measures.
\end{assumption}

\begin{theorem}[Coverage under minorization]
\label{supp:thm:pac_coverage}
Under Assumptions~\ref{assum:h1} and~\ref{assum:h2}, \revRefTwo{and under a monotone witness abstraction that inserts a witness prototype when an uncovered core region is hit, preserves already inserted witnesses, and has node budget at least $M$,} for any $\varepsilon_{\rm cov}>0$ and $\delta\in(0,1)$, after processing
\begin{equation}
    \revRefTwoMath{T_{\mathrm{cover}} = \tau\left\lceil \frac{2(M+\log(1/\delta))}{p_{\min}} \right\rceil}
\end{equation}
revealed vectors, the abstraction $\mathcal{G}$ $\varepsilon_{\rm cov}$-covers all core regions in $\mathcal{R}^\star$ with probability at least $1-\delta$.
\end{theorem}

\begin{proof}
Partition the revealed-vector stream into consecutive blocks of length $\tau$. Let $B:=\lceil T/\tau \rceil$. Define $Y_b\in\{0,1\}$ indicating whether block $b$ contains at least one revealed vector in the currently uncovered set. If at the start of block $b$ some region is uncovered, Assumption~\ref{assum:h2} gives $\mathbb{P}(Y_b=1 \mid \mathcal{F}_{b-1}) \ge p_{\min}$. Whenever $Y_b=1$, the monotone witness abstraction inserts a witness in at least one previously uncovered core region, and witness preservation keeps that region covered afterward. Thus $M$ successful blocks are sufficient to cover all $M$ regions. \revRefTwo{As before, couple $\{Y_b\}$ from below by i.i.d.\ variables $\tilde Y_b\sim\mathrm{Bernoulli}(p_{\min})$ while uncovered regions remain, so $S_B=\sum_{b=1}^{B}Y_b$ stochastically dominates $\tilde S_B=\sum_{b=1}^{B}\tilde Y_b$. For $B=\lceil2(M+\log(1/\delta))/p_{\min}\rceil$, the mean $\mu=p_{\min}B$ satisfies $\mu\ge2(M+\log(1/\delta))$. The multiplicative Chernoff lower-tail bound gives}
\[
\revRefTwoMath{\mathbb{P}(\tilde S_B<M)
\le \exp\!\left(-\frac{(\mu-M)^2}{2\mu}\right)
\le \exp(-\log(1/\delta))=\delta.}
\]
\revRefTwo{Stochastic domination then gives $\mathbb{P}(S_B<M)\le\delta$.} Hence $T=\tau B$ gives the stated $T_{\mathrm{cover}}$.
\end{proof}

\begin{remark}[A sufficient condition yielding an explicit $p_{\min}$]
\label{rem:h2_instantiation}
Let $R_{\mathrm{un}}:=\bigcup_{j\in\mathcal{J}_{\mathrm{un}}} R_j^\star$ be the uncovered set. Fix any abstract pair $(q,a)$ visited while uncovered regions remain. Assume that after $(q,a)$ has been visited $N$ times, the planner performs at least $K(N) := \lceil k_{\mathrm{o}} N^{\alpha_{\mathrm{o}}}\rceil$ new successor-expansion attempts (e.g., progressive widening), each with success probability $\ge \phi_{\mathrm{hit}}>0$ of producing a revealed successor that lands in $R_{\mathrm{un}}$. Then, by iterated conditioning,
\begin{equation}
    \mathbb{P}(\text{hit } R_{\mathrm{un}} \text{ within } K(N)\text{ attempts}) \;\ge\; 1-(1-\phi_{\mathrm{hit}})^{K(N)},
\end{equation}
so Assumption~\ref{assum:h2} holds with $p_{\min}\ge 1-(1-\phi_{\mathrm{hit}})^{K(N)}>0$.
\end{remark}

\section{Proofs for Performance Guarantees}
\label{supp:proofs_performance}

\begin{lemma}[Parity as accepting bottom-class reachability]
\label{lem:ParityAsReach}
Let $\mathcal{M}$ be a finite product MDP with a parity objective, and fix a stationary policy $\pi$. Let $G$ denote the union of accepting bottom recurrent classes under $\pi$. Then $\Pr^\pi(\text{parity satisfied}) = \Pr^\pi(\mathrm{reach}\, G)$.
\end{lemma}

\begin{proof}
Under a stationary policy, the finite-state process is a Markov chain that almost surely enters a bottom recurrent class. The parity condition depends only on priorities visited infinitely often, which is determined by the bottom class. Hence parity is satisfied if and only if the chain enters an accepting bottom class.
\end{proof}

\begin{lemma}[Stepwise maximal coupling]
\label{lem:MaxCoupling}
If $\sup_{x,a}\bigl\|P(\cdot\mid x,a)-\tilde P(\cdot\mid x,a)\bigr\|_1 \le \varepsilon$, then for any $(x,a)$ there exists a coupling $(Y,Z)$ with marginals $Y\sim P(\cdot\mid x,a)$ and $Z\sim \tilde P(\cdot\mid x,a)$ such that $\Pr(Y\neq Z \mid x,a)\le \tfrac12\varepsilon$.
\end{lemma}

\begin{proof}
Standard maximal coupling: any two distributions admit a coupling whose mismatch probability equals the total variation distance $\tfrac12\|p-q\|_1$.
\end{proof}

\begin{lemma}[Prefix error bound]
\label{lem:PrefixBound}
Fix a stationary policy $\pi$ and initial state $x_0$. Construct coupled trajectories $(X_t,\tilde X_t)_{t\ge 0}$ under $P$ and $\tilde P$ with the stepwise coupling of Lemma~\ref{lem:MaxCoupling}. Then for any $T\in\mathbb{N}$, $\Pr(\exists\, t\le T:\, X_t\neq \tilde X_t) \le T\varepsilon/2$. Letting $A_T=\{\tau_G\le T\}$ and $\tilde A_T=\{\tilde\tau_G\le T\}$, $|\Pr_{P}^{\pi}(A_T) - \Pr_{\tilde P}^{\pi}(\tilde A_T)| \le T\varepsilon/2$.
\end{lemma}

\begin{proof}
At each step, $\Pr(X_{t+1}\neq \tilde X_{t+1}\mid X_t=\tilde X_t)\le \varepsilon/2$. A union bound over $t=0,\dots,T-1$ gives the divergence bound. If the coupled paths do not diverge by $T$, the reach-within-$T$ events agree, giving the second claim.
\end{proof}

\begin{lemma}[Tail error bound]
\label{lem:TailBound}
If $\mathbb{E}[\tau_{G\cup B}]\le H$ and $\mathbb{E}[\tilde\tau_{G\cup B}]\le H$, then for any $T\in\mathbb{N}$, $\Pr(\tau_{G\cup B}>T)\le H/T$ and $\Pr(\tilde\tau_{G\cup B}>T)\le H/T$.
\end{lemma}

\begin{proof}
Markov's inequality: $\Pr(\tau>T)\le \mathbb{E}[\tau]/T$.
\end{proof}

\begin{proposition}[Parity value continuity under kernel perturbations]
\label{prop:parity_continuity}
Let $G$ be the accepting reachability target set and $B$ the bad target set from a common parity-to-reachability reduction (Lemma~\ref{lem:ParityAsReach}). Let $P$ and $\tilde P$ satisfy $\sup_{x,a}\|P(\cdot\mid x,a)-\tilde P(\cdot\mid x,a)\|_1\le\varepsilon$. Assume there exists $H>0$ such that for any stationary $\pi$, $\mathbb{E}[\tau_{G\cup B}]\le H$ under both kernels. Then for any stationary $\pi$,
\begin{equation}
    |\Pr_{P}^{\pi}(\mathrm{reach}\,G) - \Pr_{\tilde P}^{\pi}(\mathrm{reach}\,G)| \;\le\; 2\sqrt{H\,\varepsilon},
\end{equation}
and consequently $|\sup_{\pi}\Pr_{P}^{\pi}(\mathrm{reach}\,G) - \sup_{\pi}\Pr_{\tilde P}^{\pi}(\mathrm{reach}\,G)| \le 2\sqrt{H\varepsilon}$.
\end{proposition}

\begin{proof}
Decompose reachability into prefix ($\tau_G\le T$) and tail ($\tau_G>T$). The tail is upper-bounded by $\Pr(\tau_{G\cup B}>T)$. Applying Lemmas~\ref{lem:PrefixBound} and~\ref{lem:TailBound} to both kernels, $|\Pr_{P}^{\pi}(\mathrm{reach}\,G) - \Pr_{\tilde P}^{\pi}(\mathrm{reach}\,G)| \le T\varepsilon/2 + 2H/T$. Choosing $T=\lceil 2\sqrt{H/\varepsilon}\rceil$ minimizes the RHS to $2\sqrt{H\varepsilon}$.
\end{proof}

\begin{theorem}[\revRefTwo{Conditional approximate policy performance guarantee}]
\label{supp:thm:performance_guarantee}
Let $\widehat{\mathcal{B}}_\varphi(T)$ be the learned product abstraction, let $\mathcal{B}_\varphi^\star$ be the ideal product abstraction induced by the same \revRc{revealed anchors} and ideal kernel $P_B^\star$, and let $\Pi(\hat\pi^*)$ be the execution of the learned optimal abstract policy in the original POMDP. \revRc{Assume that, on an event $E_{\mathrm{succ}}$ with $\mathbb{P}(E_{\mathrm{succ}})\ge1-\delta_{\mathrm{learn}}$, the revelation, monotone coverage, map radius, and L1 model error conditions hold, with}
\begin{equation}
    \revRefTwoMath{\sup_{q,a}\bigl\|\widehat P_B(\cdot\mid q,a)-P_B^\star(\cdot\mid q,a)\bigr\|_1\le\varepsilon_P.}
\end{equation}
\revRefTwo{Assume also that the common parity-to-reachability reduction uses fixed accepting and bad sets $G$ and $B$, and that $\mathbb{E}_{K}^{\pi}[\tau_{G\cup B}]\le H$ for every stationary policy considered and for each kernel $K\in\{\widehat P_B,P_B^\star,P_{\mathrm{lift}}\}$. Then}
\begin{equation}
\label{supp:thm6:eq1}
    \revRefTwoMath{V_{\mathcal{P}}^\varphi(\Pi(\hat{\pi}^*)) \;\ge\; V^{\varphi}_{\widehat{\mathcal{B}}}(\hat{\pi}^*) \;-\; \delta_{\text{learn}} \;-\; \varepsilon_{\text{model}} \;-\; \varepsilon_{\text{abs}},}
\end{equation}
\revRefTwo{where $\varepsilon_{\text{model}}:=2\sqrt{H\varepsilon_P}$, and under Assumption~\ref{ass:lipschitz-kernel} and the Euclidean map radius $\varepsilon_{\rm map}:=c_{\rm met}\rho_R$, with $c_{\rm met}$ denoting the reciprocal square root of $\underline{\lambda}_{\rm met}:=\lambda_{\min}(M_{\rm met})$,}
\begin{equation}
    \revRefTwoMath{\varepsilon_{\text{abs}} \le 2\sqrt{H\,L_T\,\varepsilon_{\rm map}}.}
\end{equation}
\end{theorem}

\begin{proof}
Let $E_{\text{succ}}$ be the event in the theorem statement and set $\delta_{\text{learn}}=\mathbb{P}(\neg E_{\text{succ}})$.

\revRefTwo{\textbf{Stage A (Model error $\varepsilon_{\text{model}}$).} Condition on $E_{\text{succ}}$. The L1 kernel-error condition gives $\sup_{q,a}\|\widehat P_B(\cdot\mid q,a)-P_B^\star(\cdot\mid q,a)\|_1\le\varepsilon_P$. Because the DPA transition is deterministic once the next abstract state is chosen, the same L1 perturbation bound applies to the product kernels. Applying Proposition~\ref{prop:parity_continuity} with the fixed policy $\hat\pi^*$ yields}
\[
\revRefTwoMath{V_{\mathcal{B}^\star}^{\varphi}(\hat\pi^*) \ge V_{\widehat{\mathcal B}}^{\varphi}(\hat\pi^*) - 2\sqrt{H\varepsilon_P}.}
\]
\revRefTwo{Define $\varepsilon_{\text{model}}:=2\sqrt{H\varepsilon_P}$. This step uses the stated L1 model-error event; Assumption~\ref{ass:lipschitz-kernel} is used later for the abstraction-error term.}

\revRefTwo{\textbf{Stage B (Abstraction error $\varepsilon_{\text{abs}}$).} The lifted policy $\Pi(\hat\pi^*)$ applies the same abstract action selected by $\hat\pi^*$ after mapping a revealed representative to its prototype. On $E_{\text{succ}}$, each compared representative $x_{\mathrm{rev}}$ and its assigned prototype $W(q)$ satisfy $\|x_{\mathrm{rev}}-W(q)\|_{M_{\rm met}}\le\rho_R$. Since $M_{\rm met}\succ0$, $\|x_{\mathrm{rev}}-W(q)\|_2\le c_{\rm met}\rho_R=\varepsilon_{\rm map}$. Assumption~\ref{ass:lipschitz-kernel} therefore gives an L1 kernel perturbation at most $L_T\varepsilon_{\rm map}$ between the lifted comparison kernel and the ideal revealed-anchor kernel for the same policy. Proposition~\ref{prop:parity_continuity} gives}
\[
\revRefTwoMath{V_{\mathcal P}^{\varphi}(\Pi(\hat\pi^*)) \ge V_{\mathcal{B}^\star}^{\varphi}(\hat\pi^*) - 2\sqrt{H L_T\varepsilon_{\rm map}}.}
\]
\revRefTwo{Define $\varepsilon_{\text{abs}}:=2\sqrt{H L_T\varepsilon_{\rm map}}$.}

\revRefTwo{\textbf{Stage C (Remove conditioning).} Combining Stages A and B for the same fixed policy $\hat\pi^*$ gives, on $E_{\text{succ}}$, $V_{\mathcal P}^{\varphi}(\Pi(\hat\pi^*))\ge V_{\widehat{\mathcal B}}^{\varphi}(\hat\pi^*)-\varepsilon_{\text{model}}-\varepsilon_{\text{abs}}$. By total probability and boundedness of satisfaction values in $[0,1]$, $V_{\mathcal P}^{\varphi}(\Pi(\hat\pi^*))\ge V_{\widehat{\mathcal B}}^{\varphi}(\hat\pi^*)-\delta_{\text{learn}}-\varepsilon_{\text{model}}-\varepsilon_{\text{abs}}$. In the accepting abstract instances considered in our experiments, $V^{\varphi}_{\widehat{\mathcal{B}}}(\hat\pi^*)=1$, yielding the simplified bound $1-(\delta_{\text{learn}}+\varepsilon_{\text{model}}+\varepsilon_{\text{abs}})$.}
\end{proof}

\section{Additional Experimental Details}
\label{supp:experiments}

\subsection{Reward Functions: Rationale and Specification-to-Reward Compilation}
\label{supp:reward_rationale}

\paragraph{\revRefFour{Proxy reward rationale}} \revRefFour{The primary objective studied in this work is logical satisfaction of an $\omega$-regular parity objective. Two practical reasons motivate a carefully designed proxy reward.} \textbf{(i) Fair comparison with DRL baselines}: a direct ``specification satisfaction'' signal for an $\omega$-regular property is inherently \emph{infinite-horizon} and yields an extremely sparse training signal, so a dense proxy reward is needed for feasibility. \textbf{(ii) Heuristic guidance for reward-based planners and ablations}: several POMDP solvers and ablated variants ultimately rely on stepwise returns, and a reward compiled from the specification provides a consistent heuristic that correlates with safety and liveness progress. \revRefFour{Success is always defined by specification-aligned metrics: Success Rate, Cycles, and TTFC.}

\paragraph{Design principles} The proxy rewards reflect the intended \emph{lexicographic structure}: (1) safety violations strongly discouraged; (2) liveness progress observable within finite episodes via event-triggered rewards tied to meaningful milestones; (3) a small per-step time cost for efficiency; (4) reward-hacking avoided (e.g., staying inside a goal region cannot accumulate reward).

\subsubsection{2D Persistent Patrolling: From Büchi Patrolling to Cycle Rewards}

\paragraph{Specification recap} (i) \textbf{Liveness}: visit goal regions $G_1,G_2,G_3$ infinitely often; (ii) \textbf{Safety}: avoid the hazard set $X_{\text{hazard}}$ (and obstacles and collisions).

\paragraph{Finite-episode surrogate for liveness} A \emph{cycle} is completed when all three goals have been visited at least once since the previous cycle completion; this aligns with Büchi acceptance in that policies that repeatedly complete cycles without violations correspond to trajectories that keep revisiting all goals.

\paragraph{Event-triggered progress tracking} We maintain a progress set $\mathcal{V}_t\subseteq\{1,2,3\}$:
\begin{equation}
    \mathcal{V}_{t+1}=\begin{cases}\mathcal{V}_t\cup\{i\} & \text{if } x_{t+1}\in G_i \text{ and } i\notin\mathcal{V}_t,\\ \mathcal{V}_t & \text{otherwise.}\end{cases}
\end{equation}
A \emph{new-goal} event occurs when $i\notin\mathcal{V}_t$ and the agent enters $G_i$. A \emph{cycle-complete} event occurs when $\mathcal{V}_{t+1}=\{1,2,3\}$; after awarding the cycle reward, $\mathcal{V}_{t+1}\leftarrow \emptyset$.

\paragraph{Reward definition (2D)}
\begin{align}
\label{eq:r2d_app}
    r_{2D}(t) = 100\,\mathbb{I}[\text{cycle-complete}] + 50\,\mathbb{I}[\text{new-goal}]\\ - 0.1 - 2\,\mathbb{I}[\mathsf{col}_{\text{static}}] - 50\,\mathbb{I}[\mathsf{col}_{\text{dyn}}] - 100\,\mathbb{I}[\mathsf{haz}].
\end{align}

\paragraph{Why each term is reasonable} The $+100$ cycle bonus is the most direct finite-episode witness of Büchi liveness. The $+50$ new-goal bonus reduces sparsity; being event-triggered, it cannot be exploited by lingering. The $-0.1$ per-step cost discourages indecisive wandering. \revRefFour{Safety penalties ($-2$ static, $-50$ dynamic, and $-100$ hazard) reflect the proxy-reward scale used in Eq.~\eqref{eq:r2d_app}; dynamic collisions and hazard entry are treated as higher-risk events than static contacts.}

\subsubsection{3D Navigation: Reach-Avoid Reward}

\paragraph{Reward definition (3D)}
\begin{equation}
\label{eq:r3d_app}
    r_{3D}(t) = 50\,\mathbb{I}[\mathsf{goal}] - 1 - 100\,\mathbb{I}[\mathsf{unsafe}].
\end{equation}
The $+50$ terminal goal reward encourages reaching the target; the larger $-1$ per-step penalty compared with 2D reflects the longer paths and higher branching factor; the $-100$ unsafe penalty discourages risky shortcuts. We use $\gamma=0.99$ across all experiments; reward magnitudes stay within order $10^2$ to avoid destabilizing DRL while preserving the priority order safety~$\succ$~task completion~$\succ$~progress~$\succ$~efficiency.

\subsection{Hyperparameters for REBA and Baselines}
\label{supp:hyperparams}

\paragraph{Common settings} \revRefFour{The online POMDP baselines (POMCP, GPOMCP, REFSOLVER, CPOMDP) share $\gamma=0.99$ and $N_p=300$ particles with planning budget $N_{\mathrm{sim}}=300$ per environment step. Baseline planners using \revRefThree{observation progressive widening (OPW)} use $(k_o,\alpha_o)=(3,0.6)$, and baseline planners using \revRefThree{action progressive widening (APW)} use $(k_a,\alpha_a)=(10,0.1)$ where applicable. REBA uses $(k_a,\alpha_a)=(5.0,0.1)$, as reported in Sec.~\ref{sec:eval_protocol} and in the REBA tuning protocol below.}

\revRefFour{
\begin{table*}[t]
\centering
\caption*{\textbf{Baseline reporting summary.} This unnumbered protocol note summarizes which baselines are shown in the main and supplementary tables. Main Table~\ref{tab:nominal_complete} uses a unified nominal comparison for Static 2D, Dynamic 2D, and 3D Navigation, while Table~\ref{tab:supp_lightdark_recurrent} retains the complete Continuous LightDark comparison. \revRefThree{POMCPOW denotes Partially Observable Monte Carlo Planning with Observation Widening, and PFT-DPW denotes Particle Filter Trees with Double Progressive Widening.}}
\footnotesize
\setlength{\tabcolsep}{4pt}
\renewcommand{\arraystretch}{1.12}
\begin{tabular}{@{}p{0.17\textwidth}p{0.36\textwidth}p{0.39\textwidth}@{}}
\toprule
Setting & Baselines reported & Reporting rationale \\
\midrule
Static 2D & Recurrent PPO, QR-DQN, A2C, TRPO; POMCP, GPOMCP, REFSOLVER, CPOMDP & Static 2D supports both DRL baselines trained with rewards and online POMDP solvers under the same metrics aligned with the specification. \\
Dynamic 2D & Recurrent PPO, QR-DQN, A2C, TRPO; POMCP, GPOMCP, REFSOLVER, CPOMDP & Dynamic 2D tests online planning with changing obstacles and risk map updates. Main Table~\ref{tab:nominal_complete} reports both the reward-trained DRL rows and the online POMDP solvers, while budgets, hyperparameters, and curves remain in Table~\ref{tab:supp_drl_hparams_all} and Fig.~\ref{fig:comparison}. \\
3D navigation & Recurrent PPO, QR-DQN, A2C, TRPO; POMCP, GPOMCP, REFSOLVER, CPOMDP & The 3D setting tests online planning with a larger continuous belief state. Main Table~\ref{tab:nominal_complete} reports both the reward-trained DRL rows and the online POMDP solvers, while budgets, hyperparameters, and curves remain in Table~\ref{tab:supp_drl_hparams_all} and Fig.~\ref{fig:comparison}. \\
Continuous LightDark & Main Table~\ref{tab:lightdark_main}: POMCP, CPOMDP, REFSOLVER, Recurrent-PPO;
Table~\ref{tab:supp_lightdark_recurrent}: POMCPOW, PFT-DPW, CPOMDP, REFSOLVER, POMCP, GPOMCP, Recurrent-PPO & The main table prioritizes consistency with the core online POMDP baseline set, while the supplementary table keeps the complete external-baseline and ablation set. \\
\bottomrule
\end{tabular}
\end{table*}
}

\revRc{Main Table~\ref{tab:nominal_complete} includes the Dynamic-2D and 3D DRL rows; this supplementary section retains the reward definitions, budgets, hyperparameters, and training curves needed to interpret those reward-trained baselines.}

\paragraph{Default policy (rollouts)} For POMCP, GPOMCP, and CPOMDP we use a goal-directed heuristic that (i) moves toward the current target, (ii) breaks ties by maximizing clearance from known static obstacles, and (iii) for the dynamic scenario incorporates a repulsive term from the risk-map estimate.

\paragraph{Scenario-specific hyperparameters} Table~\ref{tab:hyper_all_environments} summarizes the tuned values for each algorithm across the three scenarios.

\begin{table*}[t]
\centering
\caption{Hyperparameters for online POMDP baselines across three environments.}
\label{tab:hyper_all_environments}
\small
\setlength{\tabcolsep}{12pt}
\begin{tabular}{lcccc}
\toprule
Method & $N_{\mathrm{sim}}$/step & $N_p$ & Horizon / depth & Exploration \\
\midrule
\multicolumn{5}{l}{\textbf{\textit{Environment 1: 2D Static Persistent Patrolling}}} \\
\midrule
POMCP (with OPW) & 300 & 300 & $d_{\max}=150$ & $c_{\mathrm{UCT}}=25$ \\
REFSOLVER      & 300 & 300 & $d_{\mathrm{tree}}=90, d_{\mathrm{roll}}=160$ & $\alpha=0.6$ \\
GPOMCP         & 300 & 300 & $d_{\max}=150$ & $C_N=25,\; C_b=5$ \\
CPOMDP         & 300 & 300 & $d_{\max}=150$ & $\kappa=25,\; \nu=0.1$ \\
\midrule
\multicolumn{5}{l}{\textbf{\textit{Environment 2: 2D Dynamic Persistent Patrolling}}} \\
\midrule
POMCP (with OPW) & 300 & 300 & $d_{\max}=120$ & $c_{\mathrm{UCT}}=30$ \\
REFSOLVER      & 300 & 300 & $d_{\mathrm{tree}}=60, d_{\mathrm{roll}}=140$ & $\alpha=0.4$ \\
GPOMCP         & 300 & 300 & $d_{\max}=120$ & $C_N=30,\; C_b=10$ \\
CPOMDP         & 300 & 300 & $d_{\max}=120$ & $\kappa=30,\; \nu=0.1$ \\
\midrule
\multicolumn{5}{l}{\textbf{\textit{Environment 3: 3D Navigation}}} \\
\midrule
POMCP (with OPW) & 300 & 300 & $d_{\max}=160$ & $c_{\mathrm{UCT}}=15$ \\
REFSOLVER      & 300 & 300 & $d_{\mathrm{tree}}=120, d_{\mathrm{roll}}=200$ & $\alpha=0.4$ \\
GPOMCP         & 300 & 300 & $d_{\max}=160$ & $C_N=15,\; C_b=8$ \\
CPOMDP         & 300 & 300 & $d_{\max}=160$ & $\kappa=15,\; \nu=0.1$ \\
\bottomrule
\end{tabular}
\end{table*}

\begin{table*}[t]
\centering
\caption{DRL baseline hyperparameters and training budgets. All other unspecified hyperparameters follow Stable-Baselines3 / SB3-Contrib defaults.}
\label{tab:supp_drl_hparams_all}
\footnotesize
\setlength{\tabcolsep}{8pt}
\renewcommand{\arraystretch}{1.2}
\begin{tabular}{llp{11.5cm}}
\toprule
\textbf{Algorithm} & \textbf{Policy} & \textbf{Hyperparameters (explicitly set)} \\
\midrule
\multicolumn{3}{l}{\textbf{Exp-1: 2D Static}~~\textit{Budget: Train $2\times 10^{6}$ steps; Test horizon: 500}} \\
\cmidrule(r){1-3}
QR-DQN & Mlp & $\gamma{=}0.99$; lr=1e-4; buf=100k; start=10k; batch=256; $\tau$=1.0; train\_freq=4; grad=2; tgt=1000; exp\_frac=0.25; $\epsilon_0$=1.0; $\epsilon_f$=0.05; max\_g=0.5; nQ=50. \\
Recurrent PPO & MlpLstm & $\gamma{=}0.99$; n\_steps=1024; batch=256; epochs=4; gae=0.95; clip=0.1; ent=0.01; vf=0.5; max\_g=0.5; lr=3e-4; LSTM(h=64, layers=1). \\
TRPO & Mlp & $\gamma{=}0.99$; n\_steps=1024; batch=256; gae=0.95; lr=3e-4; critic\_update=10. \\
A2C & Mlp & $\gamma{=}0.99$; n\_steps=512; gae=0.95; ent=0.01; vf=0.25; max\_g=0.5; lr=3e-4. \\
\midrule
\multicolumn{3}{l}{\textbf{Exp-2: 2D Dynamic}~~\textit{Budget: Train $2\times 10^{6}$ steps; Truncation: 500; VecEnv: 1}} \\
\cmidrule(r){1-3}
QR-DQN & Mlp & $\gamma{=}0.99$; lr=1e-4; buf=100k; start=10k; batch=256; $\tau$=1.0; train\_freq=4; grad=2; tgt=1000; exp\_frac=0.3; $\epsilon_0$=1.0; $\epsilon_f$=0.1; max\_g=0.5; nQ=50. \\
Recurrent PPO & MlpLstm & $\gamma{=}0.99$; n\_steps=2048; batch=256; epochs=4; gae=0.9; clip=0.1; ent=0.01; vf=0.5; max\_g=0.5; lr=3e-4; LSTM(h=64, layers=1). \\
TRPO & Mlp & $\gamma{=}0.99$; n\_steps=2048; batch=256; gae=0.9; lr=3e-4; critic\_upd=20/15. \\
A2C & Mlp & $\gamma{=}0.99$; n\_steps=512; gae=0.9; ent=0.01; vf=0.25; max\_g=0.5; lr=3e-4. \\
\midrule
\multicolumn{3}{l}{\textbf{Exp-3: Drone3D}~~\textit{Budget: Train $1\times 10^{7}$ steps; Max ep: 300; VecEnv: 8 train / 1 eval; Eval freq: 50k}} \\
\cmidrule(r){1-3}
QR-DQN & Mlp & $\gamma{=}0.99$; lr=1e-4; buf=200k; start=20k; batch=64; train\_freq=4; grad=1; tgt=1000; exp\_frac=0.1; $\epsilon_f$=0.01; nQ=50. \\
Recurrent PPO & MlpLstm & $\gamma{=}0.99$; n\_steps=128; batch=256; gae=0.95; ent=0.01; vf=0.5; max\_g=0.5; lr=2.5e-4. \\
A2C & Mlp & $\gamma{=}0.99$; n\_steps=5; gae=1.0; ent=0.01; vf=0.5; lr=7e-4. \\
TRPO & Mlp & $\gamma{=}0.99$; gae=0.95; cg\_steps=10; cg\_damp=0.1; target\_kl=0.01. \\
\bottomrule
\end{tabular}
\end{table*}

\subsection{REBA Hyperparameter Tuning Protocol}
\label{supp:reba_tuning_protocol}

\paragraph{Parameter categorization} REBA's parameters are grouped as: (i) \textbf{Compute Budget}: fixed at $N_p=300$ and $N_{\text{sim}}=300$; (ii) \textbf{Standard MCTS}: $k_o=3,\alpha_o=0.6$; $k_a=5.0,\alpha_a=0.1$; $c_{\text{UCT}}=10$; (iii) \textbf{REBA-Specific} (the only ones swept): the reveal percentile $p$, the gain $\kappa_{\mathrm{gain}}$, the drop multiplier $m$ for $\theta_{\mathrm{drop}}=m\sigma_{\Delta H}$, and the GNG node budget $|Q_{\mathcal{B}}|$ and \revRefTwo{map radius $\rho_R$}. Secondary parameters use stabilizing values (Dirichlet $\alpha=0.01$, $N_{\min}=20$).

\paragraph{Validation setting} Tuning uses a held-out validation task that shares the Static-2D map and noise but uses independent random seeds. We screen 10 episodes (500 steps) per configuration, then re-evaluate the top three with 20 episodes.

\paragraph{Lexicographic selection rule} (1) maximize Success Rate; (2) within 1\% to 2\% of the best, maximize Cycles; (3) break ties by minimizing TTFC; (4) finally prefer the smallest $|Q_{\mathcal{B}}|$. Search ranges and final values are in Table~\ref{tab:reba_tuning}.

\begin{table}[h]
\centering
\caption{REBA hyperparameter sweep ranges and final selected values.}
\label{tab:reba_tuning}
\renewcommand{\arraystretch}{1.1}
\begin{tabular}{llc}
\toprule
\textbf{Hyperparameter} & \textbf{Search Range (Tested)} & \textbf{Selected} \\
\midrule
Reveal percentile $p$ & $\{0.02,\ 0.05,\ 0.08,\ 0.10,\ 0.15\}$ & $0.05$ \\
Info gain $\kappa_{\mathrm{gain}}$ & $\{0,\ 0.005,\ \mathbf{0.01},\ 0.02,\ 0.05\}$ & $0.01$ \\
Drop multiplier $m$ & $\{0.5,\ 1.0,\ 2.0\}$ & $1.0$ \\
GNG node budget & $\{20,\ 50,\ 100\}$ & $100$ \\
\revRefTwo{Map radius $\rho_R$} & $\{0.3,\ 0.5,\ 0.8\}$ & $0.5$ \\
Dirichlet smoothing $\alpha$ & $\{0.001,\ 0.01,\ 0.1\}$ & $0.01$ \\
Min count $N_{\min}$ & $\{10,\ 20,\ 50\}$ & $20$ \\
\bottomrule
\end{tabular}
\end{table}

\revRefTwo{\paragraph{Sensitivity analysis} Fig.~\ref{supp:fig_para_sensitivity} reports the parameter sweeps referenced in the main evaluation protocol. Performance remains stable across the practical ranges $p\in[0.02,0.10]$ for the dynamic entropy percentile, $\kappa_{\mathrm{gain}}\le 0.05$ for the information-gain filter, and $\theta_{\mathrm{drop}}\le 1.0\sigma$ for the entropy-drop threshold; outside these ranges the abstraction either disconnects, passes noisy anchors, or blocks useful new nodes.}

\begin{figure*}[t]
    \centering
    \includegraphics[width=0.96\textwidth]{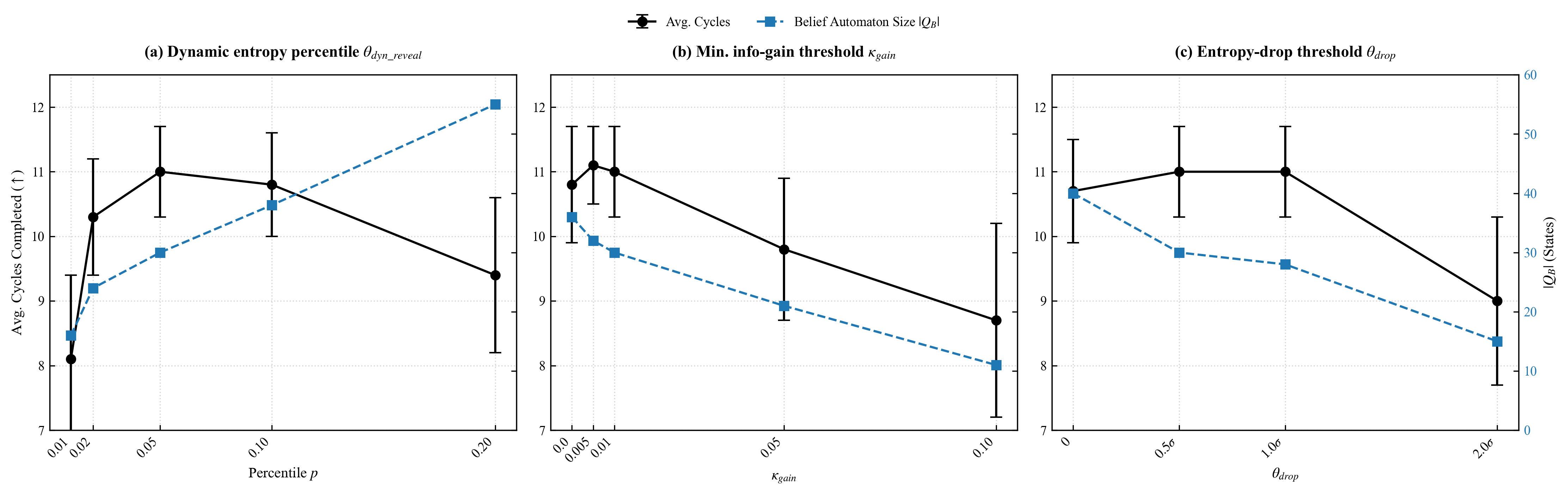}
    \caption{\revRefTwo{Sensitivity of REBA to (a) entropy percentile $\theta_{\mathrm{reveal}}^{\mathrm{dyn}}$, (b) information gain $\kappa_{\mathrm{gain}}$, and (c) entropy-drop threshold $\theta_{\mathrm{drop}}$ on patrolling performance (solid) and automaton size $|Q_{\mathcal{B}}|$ (dashed).}}
    \label{supp:fig_para_sensitivity}
\end{figure*}

\subsection{DRL Baselines Training Curves}
\label{supp:drl_curves}

\begin{figure*}[t]
    \centering
    \begin{subfigure}[t]{0.48\textwidth}
        \centering
        \includegraphics[width=\linewidth]{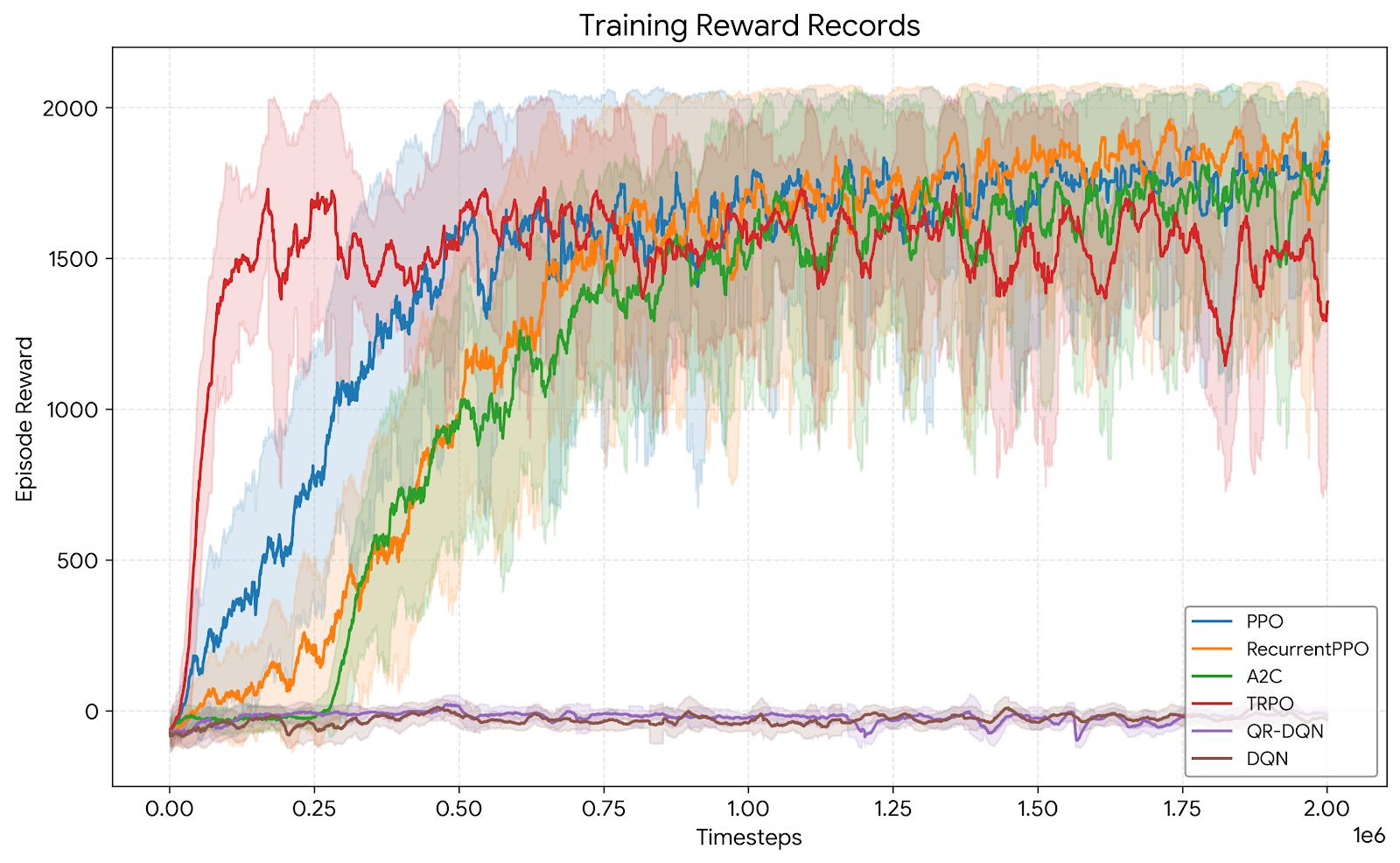}
        \caption{Static 2D Persistent Patrolling}
        \label{fig:static}
    \end{subfigure}\hfill
    \begin{subfigure}[t]{0.48\textwidth}
        \centering
        \includegraphics[width=\linewidth]{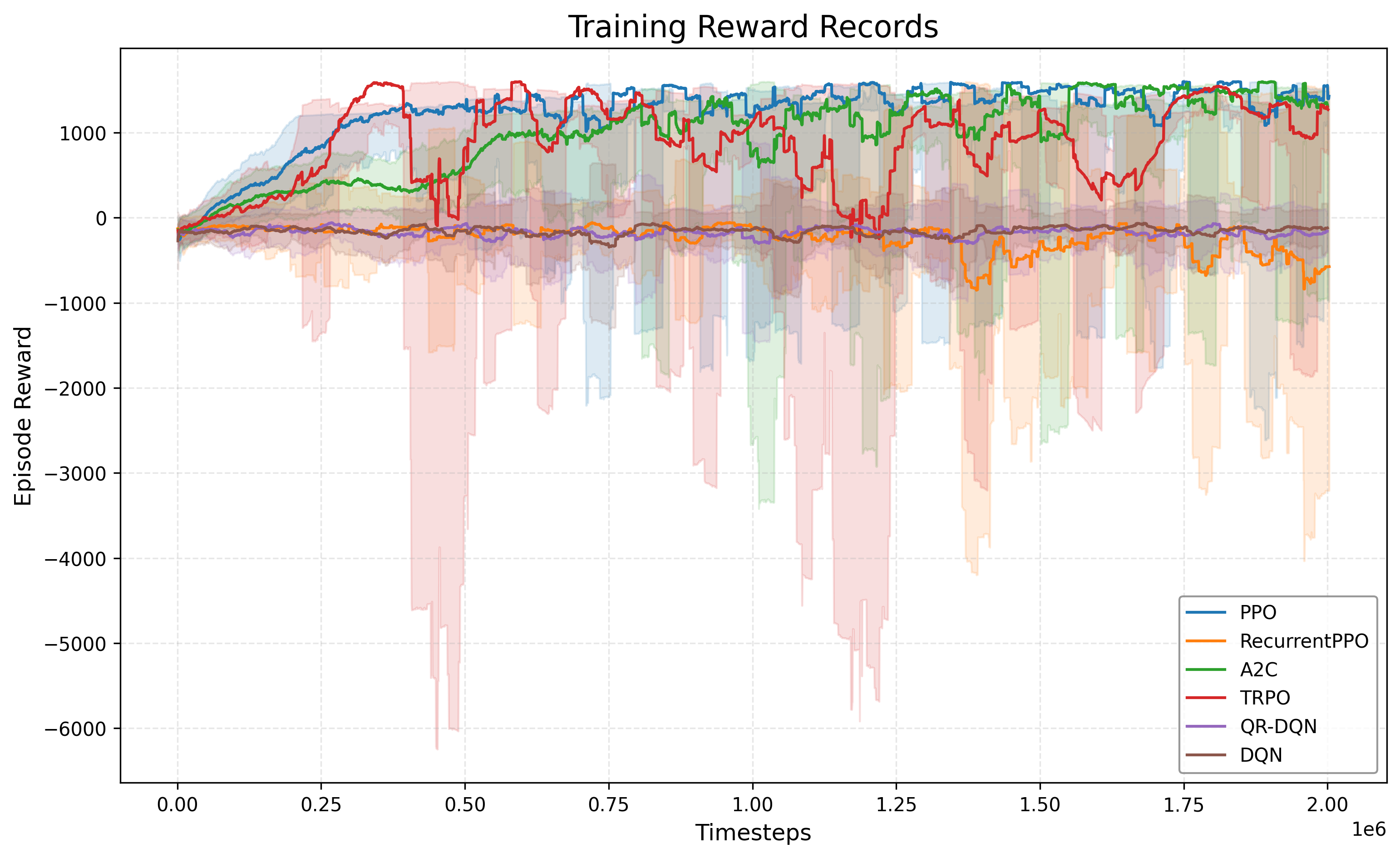}
        \caption{Dynamic 2D Persistent Patrolling}
        \label{fig:dynamic}
    \end{subfigure}
    \par\medskip
    \begin{subfigure}[t]{0.55\textwidth}
        \centering
        \includegraphics[width=\linewidth]{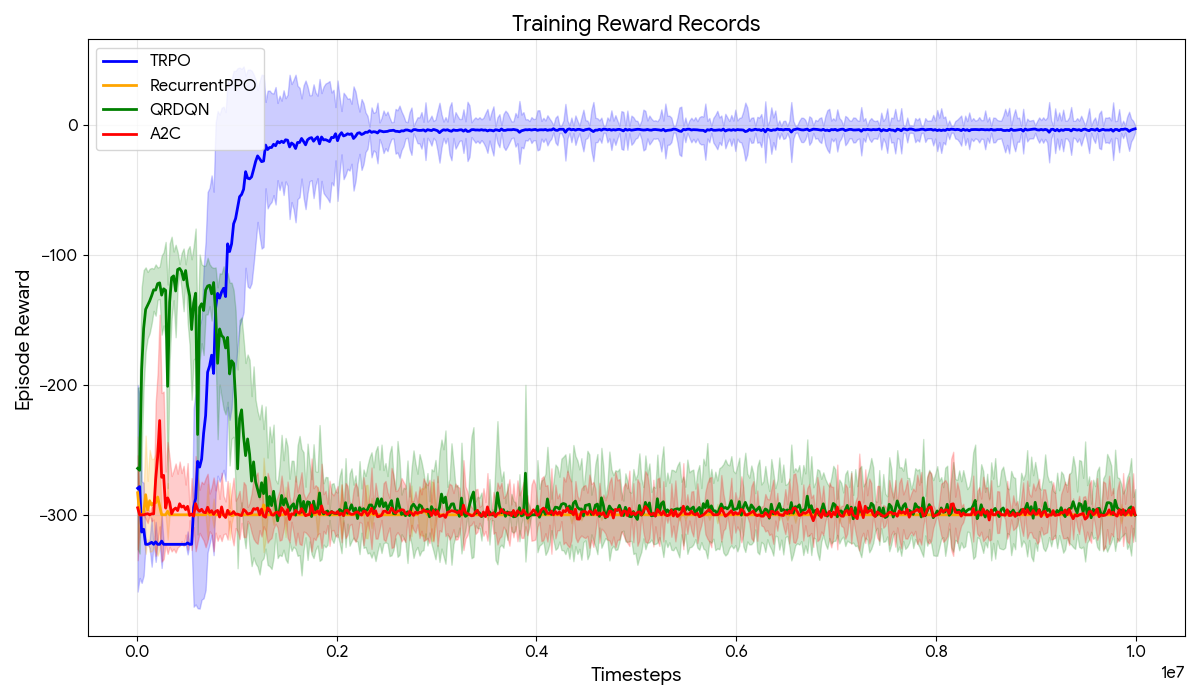}
        \caption{3D Navigation Environment}
        \label{fig:third}
    \end{subfigure}
    \caption{Training curves of different DRL methods.}
    \label{fig:comparison}
\end{figure*}

\subsection{3D Navigation Trajectory Visualizations}
\label{supp:drone_3d_trajectories}

\begin{figure*}[!t]
  \centering
  \captionsetup[subfigure]{justification=centering}
  \newcommand{\figsixpanelheight}{4.1cm}

  \begin{subfigure}[t]{0.32\linewidth}
      \centering
      \includegraphics[height=\figsixpanelheight,keepaspectratio]{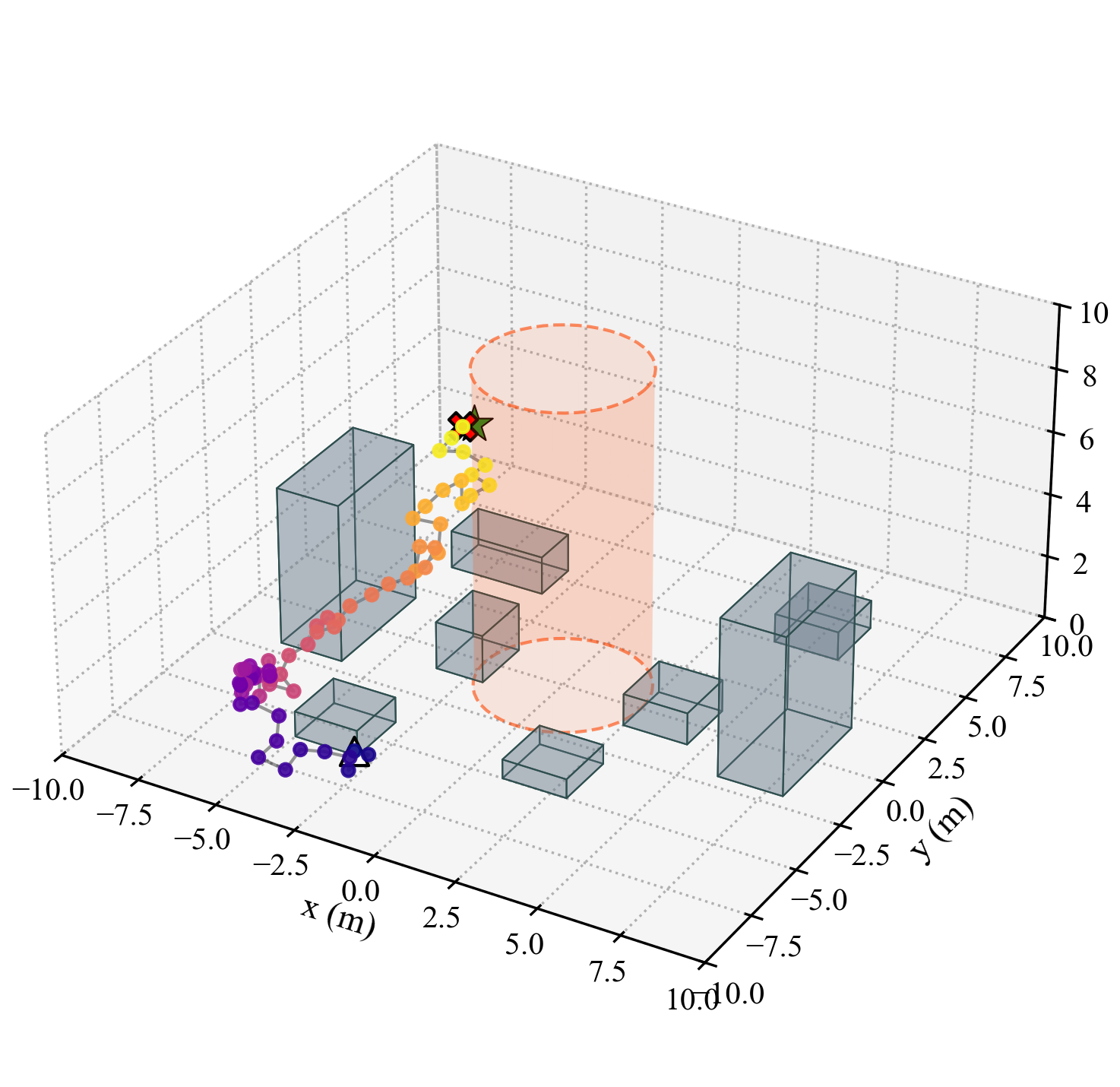}
      \caption{Motion Trajectories in 3D scenario}
      \label{fig:3d1}
  \end{subfigure}
  \hfill
  \begin{subfigure}[t]{0.32\linewidth}
      \centering
      \includegraphics[height=\figsixpanelheight,keepaspectratio]{./figures/drone_3d_others.png}
      \caption{Motion Trajectories of Each Baseline}
      \label{fig:3d3}
  \end{subfigure}
  \hfill
  \begin{subfigure}[t]{0.32\linewidth}
      \centering
      \includegraphics[height=\figsixpanelheight,keepaspectratio]{./figures/drone_2d_others.png}
      \caption{Top-down view in 3D scenario}
      \label{fig:3d2}
  \end{subfigure}

  \caption{Trajectory comparison in the 3D navigation environment.}
  \label{fig:drone_3d}
\end{figure*}

\subsection{\revRefTwo{Robustness to Environment Parameters in 2D Persistent Patrolling}}
\label{supp:robustness}

\revRefTwo{This section assesses robustness of the 2D Persistent Patrolling task against environmental variations under two stress-test parameter sets, with task specifications and hyperparameters held fixed: \textbf{Set-Obs} amplifies observation uncertainty (Static) and reduces sensing range (Dynamic); \textbf{Set-Dyn} increases transition noise and dynamic obstacle speed. For each method and setting, we conduct 20 independent trials (500 steps each), reporting Cycles, Success Rate, and TTFC. Using the same hyperparameters as the nominal setting quantifies robustness to environmental shifts. Results are reported in Tables~\ref{tab:robust_static_2d} and~\ref{tab:robust_dynamic_2d}.}

\revRefTwo{Across the four resulting settings, REBA improves the average number of patrol cycles over the strongest baseline (CPOMDP) by $+25.0\%$ to $+28.6\%$ in the Static-2D scenario and by $+41.4\%$ to $+70.5\%$ in the Dynamic-2D scenario, and REBA remains the top-performing method in patrol cycles across all four stress-test settings. The fixed-configuration protocol ties these gains to robustness under observation and dynamics shifts.}

\begin{table*}[t]
\centering
\caption{\revRefTwo{Robustness results in 2D Static Persistent Patrolling (20 trials, 500 steps). Each cell reports Cycles $\uparrow$, Success Rate (\%) $\uparrow$, and TTFC $\downarrow$ for successful trials only.}}
\label{tab:robust_static_2d}
\begin{tabular}{lccccc}
\toprule
Param set & POMCP & REFSOLVER & GPOMCP & CPOMDP & REBA (Ours) \\
\midrule
Set-Obs & $1.2\pm1.3\ /\ 15\ /\ 125\pm50$ & $5.4\pm2.3\ /\ 55\ /\ 88\pm28$ & $6.8\pm2.7\ /\ 65\ /\ 74\pm24$ & $7.0\pm2.2\ /\ 70\ /\ 66\pm20$ & $\mathbf{9.0\pm2.0\ /\ 85\ /\ 50\pm18}$ \\
Set-Dyn & $1.5\pm1.2\ /\ 20\ /\ 110\pm45$ & $6.0\pm2.0\ /\ 65\ /\ 80\pm26$ & $7.2\pm2.6\ /\ 75\ /\ 63\pm22$ & $8.0\pm2.1\ /\ 80\ /\ 56\pm18$ & $\mathbf{10.0\pm1.8\ /\ 90\ /\ 42\pm15}$ \\
\bottomrule
\end{tabular}
\end{table*}

\begin{table*}[t]
\centering
\caption{\revRefTwo{Robustness results in 2D Dynamic Persistent Patrolling (20 trials, 500 steps). Each cell reports Cycles $\uparrow$, Success Rate (\%) $\uparrow$, and TTFC $\downarrow$ for successful trials only.}}
\label{tab:robust_dynamic_2d}
\begin{tabular}{lccccc}
\toprule
Param set & POMCP & REFSOLVER & GPOMCP & CPOMDP & REBA (Ours) \\
\midrule
Set-Obs & $0.0\pm0.3\ /\ 0\ /\ --$ & $0.8\pm1.1\ /\ 10\ /\ 160\pm55$ & $3.2\pm2.4\ /\ 45\ /\ 105\pm40$ & $4.4\pm2.0\ /\ 60\ /\ 90\pm35$ & $\mathbf{7.5\pm2.6\ /\ 80\ /\ 65\pm30}$ \\
Set-Dyn & $0.7\pm1.0\ /\ 10\ /\ 180\pm40$ & $1.5\pm1.8\ /\ 15\ /\ 125\pm50$ & $4.5\pm2.8\ /\ 55\ /\ 80\pm35$ & $5.8\pm2.2\ /\ 70\ /\ 72\pm32$ & $\mathbf{8.2\pm2.2\ /\ 90\ /\ 48\pm25}$ \\
\bottomrule
\end{tabular}
\end{table*}

\revRefTwo{\paragraph{Mapping to the assumptions in Sec.~\ref{sec:limitations}}
The two stress-test parameter sets above can be associated with the
theoretical assumptions they primarily stress: \textbf{Set-Obs} (amplified
observation noise; reduced sensing range) primarily stresses the
Gaussian-likeness assumption, since multi-modal hazards make the revealed
belief more often deviate from a unimodal Gaussian; \textbf{Set-Dyn}
(amplified transition noise; faster dynamic obstacles) primarily stresses
the model-fidelity proxy $\widehat{\mathrm{TV}}_{\rm model}$ and the \revRefThree{effective sample size (ESS)}
consistency proxy $\mathrm{ESS}/N_p$, while leaving the planner's nominal
Gaussian kernel and therefore its analytic Lipschitz constant
$\hat L_T^{(1)}$ unchanged.
The corresponding per-assumption diagnostics for these two settings are
reported in Sec.~\ref{supp:assumption_diagnostics}. We also include two
diagnostic stress regimes: forced-bimodal prior and jump and heavy-tailed
noise. The empirical error decomposition for Theorem~\ref{thm:performance_guarantee} over the
same nine measured settings is reported in Sec.~\ref{supp:thm3_magnitudes}, and
per-component computational costs are reported in Sec.~\ref{supp:complexity}.}

\subsection{\revRefTwo{Empirical Diagnostics for the Assumptions Used in the Main Theorems}}
\label{supp:assumption_diagnostics}

\revRefTwo{To complement the boundary-regime discussion in Sec.~\ref{sec:limitations}, we report
empirical diagnostics for the assumptions used in
Theorems~\ref{thm:robust_revelation}--\ref{thm:performance_guarantee}, evaluated on nine measured settings: three nominal tasks (Static-2D,
Dynamic-2D, 3D-Navigation), four stress settings obtained by applying Set-Obs and Set-Dyn to both Static-2D and Dynamic-2D, and two additional diagnostic stress settings (forced-bimodal prior; jump and heavy-tailed noise). Continuous diagnostics are averaged over 20 independent runs per
setting; H1 is reported as a binary clipping count out of 20.
Definitions used in Table~\ref{tab:supp_assumption_diagnostics}:
$\hat p_{\min}^{\rm block}$ is the block-level uncovered-region hit rate,
measured directly from $\tau$-step windows and averaged across runs as an
empirical block statistic separate from $\hat\phi_{\rm hit}$;
$\hat\phi_{\rm hit}$ is the raw per-expansion probability of discovering a
new core-region hit; $\hat L_T^{(1)}$ is the empirically computed analytic
$\ell_1$-Lipschitz constant of the planner's nominal Gaussian transition
kernel, with the conservative $1/\sigma_t$ envelope in parentheses; the dip
rejection rate is the percentage of revealed beliefs at which a
\emph{projected dip diagnostic} (the first principal component plus random
unit-vector projections, at significance level $\alpha=0.05$) rejects
unimodality; $\mathrm{ESS}/N_p$ is the effective sample size before
resampling at accepted revelation times. The stress rows expose lower
particle quality or more rejection events.}

\begin{table*}[t]
\centering
\scriptsize
\setlength{\tabcolsep}{3.5pt}
\caption{\revRefTwo{Empirical diagnostics for the assumptions used in the
theoretical analysis. Continuous values are reported as mean $\pm$ standard
deviation over 20 independent runs. The H2 columns report the block-level
uncovered-region hit rate and the corresponding raw expansion-hit rate. The
H3 column reports the analytic $\ell_1$-Lipschitz constant of the Gaussian
transition kernel; the value in parentheses is the conservative $1/\sigma_t$
bound.}}
\label{tab:supp_assumption_diagnostics}
\begin{tabular}{lcccccc}
\toprule
Scenario
& H1 clipping
& $\hat p_{\min}^{\rm block}$
& $\hat\phi_{\rm hit}$
& $\hat L_T^{(1)}$
& Dip reject.\ (\%)
& $\mathrm{ESS}/N_p$ \\
\midrule
Static-2D          & $0/20$ & $0.91 \pm 0.03$ & $0.069 \pm 0.018$ & $7.98$ $(10.00)$ & $4.6 \pm 1.8$  & $0.61 \pm 0.09$ \\
Dynamic-2D         & $0/20$ & $0.88 \pm 0.04$ & $0.055 \pm 0.016$ & $7.98$ $(10.00)$ & $6.2 \pm 2.1$  & $0.53 \pm 0.10$ \\
3D-Navigation           & $0/20$ & $0.76 \pm 0.08$ & $0.041 \pm 0.012$ & $5.32$ $(6.67)$  & $5.9 \pm 2.4$  & $0.45 \pm 0.11$ \\
Static-2D + Set-Obs   & $0/20$ & $0.84 \pm 0.05$ & $0.050 \pm 0.015$ & $7.98$ $(10.00)$ & $9.0  \pm 3.0$ & $0.41 \pm 0.09$ \\
Static-2D + Set-Dyn   & $0/20$ & $0.82 \pm 0.06$ & $0.044 \pm 0.014$ & $7.98$ $(10.00)$ & $7.0  \pm 2.5$ & $0.39 \pm 0.10$ \\
Dynamic-2D + Set-Obs  & $0/20$ & $0.78 \pm 0.07$ & $0.040 \pm 0.013$ & $7.98$ $(10.00)$ & $12.5 \pm 4.0$ & $0.31 \pm 0.09$ \\
Dynamic-2D + Set-Dyn  & $0/20$ & $0.76 \pm 0.08$ & $0.034 \pm 0.012$ & $7.98$ $(10.00)$ & $9.2  \pm 3.2$ & $0.29 \pm 0.10$ \\
Bimodal prior      & $0/20$ & $0.78 \pm 0.07$ & $0.037 \pm 0.012$ & $7.98$ $(10.00)$ & $12.9 \pm 4.2$ & $0.32 \pm 0.09$ \\
Jump/noise stress  & $0/20$ & $0.63 \pm 0.09$ & $0.026 \pm 0.010$ & $7.98$ $(10.00)$ & $18.7 \pm 5.4$ & $0.24 \pm 0.07$ \\
\bottomrule
\end{tabular}
\vspace{0.35em}
\begin{minipage}{0.96\linewidth}
\footnotesize
\revRefTwo{The Set-Dyn rows and the jump and heavy-tailed-noise row use the
planner model's nominal analytic kernel constant; their departures are
represented by $\widehat{\mathrm{TV}}_{\rm model}$ in
Table~\ref{tab:supp_thm3_error_proxies}. The representative H3 boundary is
the non-Lipschitz transition-discontinuity mode discussed in Sec.~\ref{sec:discussion_robustness}.}
\end{minipage}
\end{table*}

\begin{figure*}[t]
\centering
\includegraphics[width=0.92\linewidth]{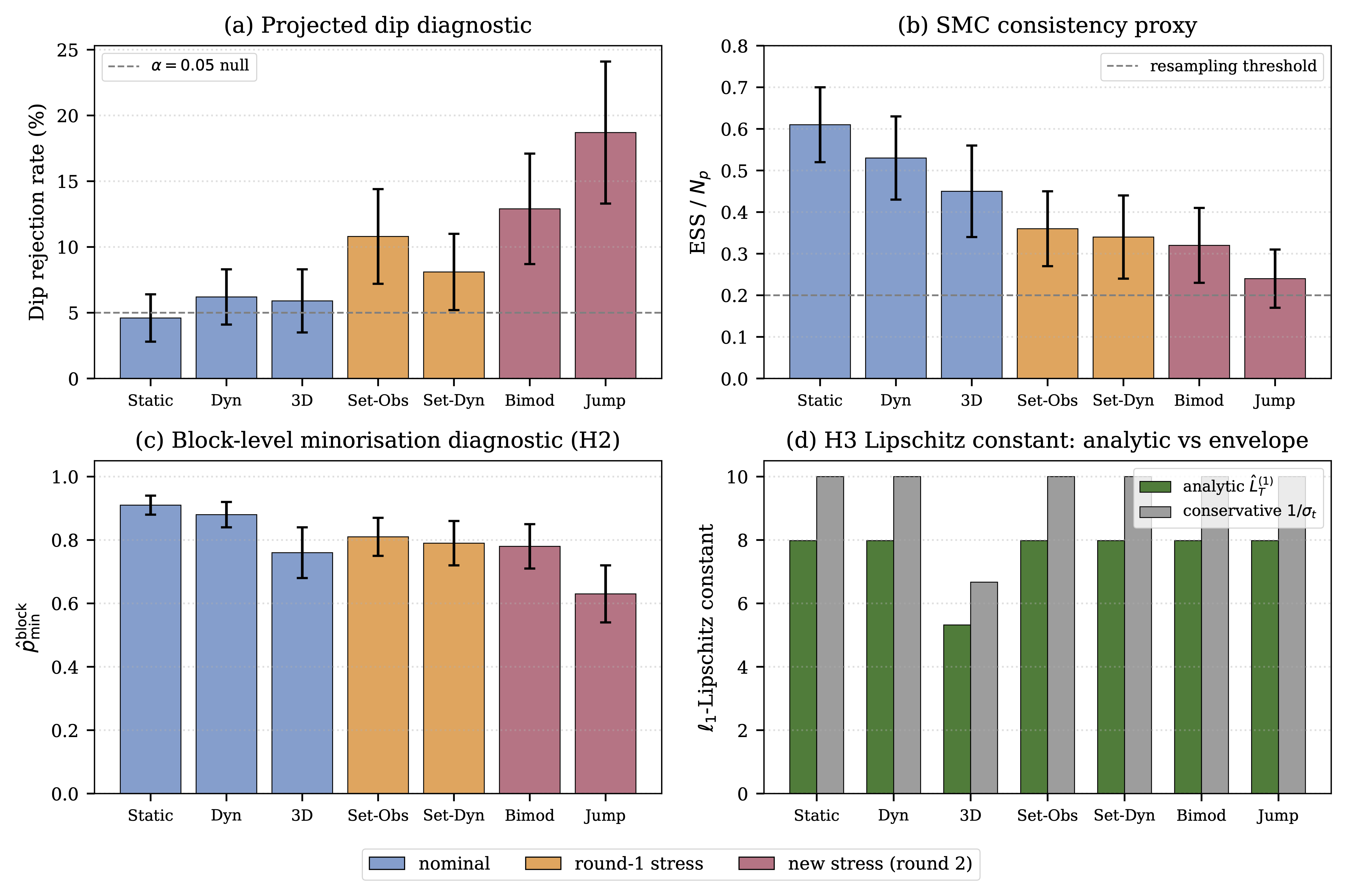}
\caption{\revRefTwo{Per-assumption diagnostics. The bars plot seven aggregated categories (Static, Dyn, 3D, Set-Obs, Set-Dyn, Bimod, Jump). Set-Obs and Set-Dyn each aggregate the Static-2D and Dynamic-2D rows, while Table~\ref{tab:supp_assumption_diagnostics} reports all nine measured settings separately.
(a)~Projected dip-rejection rate (Gaussian-likeness, with the $\alpha=0.05$
null indicated by the dashed line); (b)~$\mathrm{ESS}/N_p$ before resampling
at accepted revelation times (with the typical resampling threshold
indicated by the dashed line); (c)~block-level minorisation diagnostic
$\hat p_{\min}^{\rm block}$; (d)~analytic $\ell_1$-Lipschitz constant
$\hat L_T^{(1)}$ versus the conservative $1/\sigma_t$ envelope.}}
\label{fig:supp_assumption_diagnostics}
\end{figure*}

\subsection{\texorpdfstring{\revRefFour{Particle-Revelation Stability Diagnostics}}{Particle-Revelation Stability Diagnostics}}
\label{supp:particle_stability}

\revRefFour{This subsection reports particle-revelation stability diagnostics for the reveal gate. Table~\ref{tab:supp_particle_sweep} sweeps five particle budgets on Continuous LightDark; across the tested budgets, the accept and reject rates remain close, and the task metrics stay within the measured variation of this diagnostic over 20 episodes.}

\revRefFour{Posterior-shape diagnostics provide a second view of the same mechanism. Table~\ref{tab:supp_assumption_diagnostics} reports $\mathrm{ESS}/N_p$ at accepted revelation times ranging from $0.24$ to $0.61$ across the nine measured settings, while the projected dip rejection rate rises from $4.6\%$ in nominal Static-2D to $12.9\%$ under the forced bimodal prior and $18.7\%$ under jump and heavy-tailed noise.}

\revRefFour{We also include a forced mirror alias pilot with seven paired seeds to probe the irreducible symmetry boundary. The observation map preserves the sign ambiguity in the latent $x$ branch and the diagnostic reveal threshold is held fixed, so the reveal gate suppresses stable anchors while the posterior ambiguity remains unresolved.}

\revRefFour{
\begin{table*}[t]
\centering
\caption{Sweep over particle counts for the Continuous LightDark reveal diagnostic. The table reports reveal-gate evaluations, accepted-anchor rates, rejected-candidate rates, and task outcomes under particle budget changes.}
\label{tab:supp_particle_sweep}
\footnotesize
\setlength{\tabcolsep}{4pt}
\renewcommand{\arraystretch}{1.12}
\begin{tabular}{@{}c c c c c c c@{}}
\toprule
$N_p$ & \makecell{Reveal-gate\\ evaluations\\ per episode} & \makecell{Accepted-anchor\\ rate} & \makecell{Rejected-candidate\\ rate} & Cycles & Success & \makecell{Safety violations\\ per episode} \\
\midrule
50  & $9923.950$ & $32.5\%$ & $67.5\%$ & $8.350$ & $20/20$ & $0.000$ \\
100 & $9816.550$ & $32.2\%$ & $67.8\%$ & $8.800$ & $19/20$ & $0.050$ \\
150 & $9921.500$ & $32.5\%$ & $67.5\%$ & $8.550$ & $19/20$ & $0.050$ \\
300 & $9864.550$ & $32.3\%$ & $67.7\%$ & $8.700$ & $20/20$ & $0.000$ \\
600 & $9874.400$ & $32.4\%$ & $67.6\%$ & $8.700$ & $20/20$ & $0.000$ \\
\bottomrule
\end{tabular}
\end{table*}
}

\revRefFour{
\begin{table*}[t]
\centering
\caption*{\textbf{Particle-revelation stability summary.} Counts in the Continuous LightDark rows are raw reveal-gate event counters per episode from the full diagnostic logs; the particle-count sweep above reports normalized gate rates under its separate 20-episode diagnostic protocol. The table distinguishes accepted-anchor events from entropy-gate and risk-gate rejections; final GNG nodes are reported separately.}
\footnotesize
\setlength{\tabcolsep}{4pt}
\renewcommand{\arraystretch}{1.12}
\begin{tabular}{@{}p{0.23\textwidth}p{0.34\textwidth}p{0.34\textwidth}@{}}
\toprule
Diagnostic source & Existing values used & Reading \\
\midrule
Assumption diagnostics & $\mathrm{ESS}/N_p$ range $0.24$ to $0.61$; dip rejection $4.6\%$ nominal Static-2D, $12.9\%$ Bimodal prior, $18.7\%$ Jump/noise stress & Particle quality and posterior-shape checks degrade under deliberately difficult beliefs, while the reveal diagnostic rejects a larger fraction of candidate anchors. \\
Forced symmetry pilot, Full REBA & $0/7$ success, $0.00$ cycles, $0.00$ safety violations; accepted-anchor events $0.0$, entropy-gate rejections $9242.1$, final GNG nodes $0.0$, \revRefThree{Belief Automaton} transitions $0.0$ & Under persistent mirror aliasing, the reveal gate withholds stable anchors, so the symbolic layers receive no accepted-anchor updates. \\
Forced symmetry pilot, cold start priority ablation & $0/7$ success, $0.00$ cycles, $0.00$ safety violations; accepted-anchor events $0.0$, entropy-gate rejections $9141.0$, final GNG nodes $0.0$, \revRefThree{Belief Automaton} transitions $0.0$ & The cold start priority ablation leaves the underlying MCTS search as the active planner. \\
Forced symmetry pilot, MCTS-only & $2/7$ success, $0.43$ cycles, $0.14$ safety violations; accepted-anchor events $0.0$, final GNG nodes $0.0$, \revRefThree{Belief Automaton} transitions $0.0$ & This row provides the unguided Monte Carlo reference for the anchor-free boundary. \\
Continuous LightDark, Full REBA & $9.23$ cycles, $30/30$ success, $0.00$ safety violations; accepted-anchor events $10494.7$, entropy-gate rejections $69.9$, risk-gate rejections $33.6$, final GNG nodes $100.0$ & The revealed-anchor stream remains dense enough to populate the capped abstraction, while entropy and risk gates still filter candidate revelations. \\
Continuous LightDark, Ungated anchors & $8.70$ cycles, $30/30$ success, $0.00$ safety violations; accepted-anchor events $9899.0$, entropy-gate rejections $0.0$, risk-gate rejections $0.0$, final GNG nodes $98.6$ & On this benign setting, the ungated variant has limited task effect; Table~\ref{tab:supp_assumption_diagnostics} shows that the gate becomes more active under multimodal and jump-noise stress. \\
\bottomrule
\end{tabular}
\end{table*}
}

\revRefFour{Across these checks, reveal rates are stable over the tested particle budgets, and rejection increases under less reliable belief shapes. The LightDark logs show that revealed anchors still sustain the online abstraction, while the forced symmetry pilot shows the complementary boundary behavior: persistent posterior ambiguity keeps the GNG and \revRefThree{Belief Automaton} layers inactive until reliable anchors are available.}

\subsection{\revRefTwo{Empirical Proxies for the Performance-Bound Error Decomposition}}
\label{supp:thm3_magnitudes}

\revRefTwo{Theorem~\ref{thm:performance_guarantee} in Sec.~\ref{sec:automaton} presents a relative-form satisfaction bound
$V^\varphi_{\mathcal P}(\Pi(\hat\pi^*))\ge V^\varphi_{\widehat{\mathcal{B}}}(\hat\pi^*)-\delta_{\rm learn}-\varepsilon_{\rm model}-\varepsilon_{\rm abs}$ with
$\varepsilon_{\rm model}=2\sqrt{H\varepsilon_P}$ and
$\varepsilon_{\rm abs}\le 2\sqrt{H L_T \varepsilon_{\rm map}}$. We
therefore separate the numerical analysis into two quantities. The
worst case diagnostic following the bound structure is
\[
\begin{aligned}
B_{\rm wc}
&=\bigl[V^\varphi_{\widehat{\mathcal{B}}}(\hat\pi^*)-\hat\delta_{\rm diag}
-\hat\varepsilon_{\rm model,wc}-\hat\varepsilon_{\rm abs,wc}\bigr]_+,\\
\hat\varepsilon_{\rm model,wc}&=2\sqrt{H\hat\varepsilon_P},\\
\hat\varepsilon_{\rm abs,wc}&=2\sqrt{H\hat L_T\varepsilon_{\rm map}}.
\end{aligned}
\]
The local empirical diagnostic score is
\[
B_{\rm loc}=\bigl[V^\varphi_{\widehat{\mathcal{B}}}(\hat\pi^*)-\hat\delta_{\rm diag}
-\widehat{\mathrm{TV}}_{\rm model}
-\widehat{\mathrm{TV}}_{\rm abs,local}\bigr]_+.
\]
$B_{\rm wc}$ follows the error structure of Theorem~\ref{thm:performance_guarantee}; $B_{\rm loc}$
is a local scale diagnostic computed from measured proxies. In the accepting abstract instances below,
$V^\varphi_{\widehat{\mathcal{B}}}(\hat\pi^*)=1$. Continuous quantities are
averaged over 20 independent runs; episode success is reported as a count
out of 20 with the Wilson 95\% binomial confidence interval \revRefThree{(CI)}.
$\hat\delta_{\rm diag}$ aggregates empirical diagnostic failures (coverage
failures, insufficient transition counts, low-$\mathrm{ESS}$ events,
rejected reveal diagnostics).
$\widehat{\mathrm{TV}}_{\rm model}$ is the held-out weighted total-variation
distance between the learned abstract transition kernel and simulator
transitions, averaged over $(q,a)$ pairs visited at least $n_{\min}=20$
times. This averaged local proxy enters
$B_{\rm loc}$, while $B_{\rm wc}$ uses the uniform L1 model-error term
$\hat\varepsilon_P$.
$\hat r_R$ is the state-space nearest-prototype radius
$\hat r_R = \max_j \min_{q\in Q_{\mathcal{B}}} \mathrm{dist}(W(q),R_j^\star)$;
it is reported as a coverage diagnostic. The $B_{\rm wc}$ calculation keeps
the theorem's Euclidean map radius $\varepsilon_{\rm map}$ separate from this
state-space coverage radius.
$\widehat{\mathrm{TV}}_{\rm abs,local}$ is the empirical total-variation
distance between the abstract policy's $H$-step distribution from a
revealed start and the simulator's $H$-step distribution from the same
revealed start, averaged over revealed starts.}

\revRefTwo{Table~\ref{tab:supp_thm3_error_proxies} reports the worst case error proxies
$\hat\varepsilon_{\rm model,wc}$ and $\hat\varepsilon_{\rm abs,wc}$ together
with $B_{\rm wc}$; the averaged TV terms enter the local diagnostic score
$B_{\rm loc}$. In the updated diagnostics, $B_{\rm wc}$ ranges from $0.02$ to
$0.62$: the nominal rows give $0.62$, $0.49$, and $0.35$ for Static-2D,
Dynamic-2D, and 3D-Navigation, and the remaining stress and diagnostic rows
range from $0.02$ (Jump/noise stress) to $0.36$ (Static-2D + Set-Obs).
$B_{\rm loc}$ ranges from $0.07$ to $0.79$. In every scenario the mean
of $B_{\rm loc}$ lies below the empirical success rate, indicating that the
measured local proxy is conservative relative to empirical success in these experiments. Across the
nine measured settings the local abstraction TV is approximately
proportional to the measured state-space coverage radius, with the ratio
$\widehat{\mathrm{TV}}_{\rm abs,local}/\hat r_R$ in $[1.00, 1.12]$. This
local trend calibrates the measured proxies while keeping Theorem~\ref{thm:performance_guarantee} in its stated worst case form.}

\begin{table*}[t]
\centering
\scriptsize
\setlength{\tabcolsep}{1.5pt}
\renewcommand{\arraystretch}{1.12}
\caption{\revRefTwo{Empirical error decomposition associated with Theorem~\ref{thm:performance_guarantee}.
Continuous diagnostic quantities are reported as mean $\pm$ standard
deviation over 20 independent runs. Episode-level task success is reported
as the number of successful runs out of 20, together with the Wilson 95\%
binomial confidence interval for the success probability. $B_{\rm wc}$
denotes the clipped worst case diagnostic computed from
$\hat\varepsilon_{\rm model,wc}=2\sqrt{H\hat\varepsilon_P}$ and
$\hat\varepsilon_{\rm abs,wc}=2\sqrt{H\hat L_T\varepsilon_{\rm map}}$.
$B_{\rm loc}$ denotes a diagnostic score computed from measured local error
proxies. Both scores report diagnostic scales for the bound terms in the tested
regimes.}}
\label{tab:supp_thm3_error_proxies}
\begin{tabular}{@{}lccccccccc@{}}
\toprule
Scenario
& \makecell{$\hat\delta_{\rm diag}$}
& \makecell{$\widehat{\mathrm{TV}}_{\rm model}$}
& \makecell{$\hat r_R$}
& \makecell{$\widehat{\mathrm{TV}}_{\rm abs,local}$}
& \makecell{$\revRefTwoMath{\hat\varepsilon_{\rm model,wc}}$}
& \makecell{$\revRefTwoMath{\hat\varepsilon_{\rm abs,wc}}$}
& \makecell{$\revRefTwoMath{B_{\rm wc}}$}
& \makecell{$B_{\rm loc}$}
& \makecell{\revRefTwo{Success}\\\revRefTwo{[95\% CI]}} \\
\midrule
Static-2D          & $0.04 \pm 0.02$ & $0.07 \pm 0.03$ & $0.10 \pm 0.03$ & $0.10 \pm 0.04$ & $\revRefTwoMath{0.12}$ & $\revRefTwoMath{0.22}$ & $\revRefTwoMath{0.62}$ & $0.79 \pm 0.06$ & \makecell{$19/20$\\$[0.76,0.99]$} \\
Dynamic-2D         & $0.07 \pm 0.03$ & $0.11 \pm 0.04$ & $0.15 \pm 0.05$ & $0.16 \pm 0.05$ & $\revRefTwoMath{0.17}$ & $\revRefTwoMath{0.27}$ & $\revRefTwoMath{0.49}$ & $0.66 \pm 0.07$ & \makecell{$19/20$\\$[0.76,0.99]$} \\
3D-Navigation      & $0.11 \pm 0.04$ & $0.15 \pm 0.05$ & $0.21 \pm 0.06$ & $0.22 \pm 0.06$ & $\revRefTwoMath{0.21}$ & $\revRefTwoMath{0.33}$ & $\revRefTwoMath{0.35}$ & $0.52 \pm 0.09$ & \makecell{$19/20$\\$[0.76,0.99]$} \\
Static-2D + Set-Obs   & $0.10 \pm 0.04$ & $0.15 \pm 0.05$ & $0.21 \pm 0.07$ & $0.22 \pm 0.07$ & $\revRefTwoMath{0.20}$ & $\revRefTwoMath{0.34}$ & $\revRefTwoMath{0.36}$ & $0.53 \pm 0.10$ & \makecell{$17/20$\\$[0.64,0.95]$} \\
Static-2D + Set-Dyn   & $0.10 \pm 0.04$ & $0.19 \pm 0.06$ & $0.23 \pm 0.07$ & $0.24 \pm 0.08$ & $\revRefTwoMath{0.25}$ & $\revRefTwoMath{0.35}$ & $\revRefTwoMath{0.30}$ & $0.47 \pm 0.11$ & \makecell{$18/20$\\$[0.70,0.97]$} \\
Dynamic-2D + Set-Obs  & $0.16 \pm 0.06$ & $0.21 \pm 0.07$ & $0.27 \pm 0.08$ & $0.30 \pm 0.09$ & $\revRefTwoMath{0.29}$ & $\revRefTwoMath{0.41}$ & $\revRefTwoMath{0.14}$ & $0.33 \pm 0.12$ & \makecell{$16/20$\\$[0.58,0.92]$} \\
Dynamic-2D + Set-Dyn  & $0.14 \pm 0.05$ & $0.25 \pm 0.07$ & $0.29 \pm 0.09$ & $0.32 \pm 0.10$ & $\revRefTwoMath{0.32}$ & $\revRefTwoMath{0.43}$ & $\revRefTwoMath{0.11}$ & $0.29 \pm 0.13$ & \makecell{$18/20$\\$[0.70,0.97]$} \\
Bimodal prior      & $0.17 \pm 0.06$ & $0.18 \pm 0.06$ & $0.30 \pm 0.09$ & $0.33 \pm 0.10$ & $\revRefTwoMath{0.27}$ & $\revRefTwoMath{0.44}$ & $\revRefTwoMath{0.12}$ & $0.32 \pm 0.13$ & \makecell{$17/20$\\$[0.64,0.95]$} \\
Jump/noise stress  & $0.23 \pm 0.07$ & $0.32 \pm 0.08$ & $0.34 \pm 0.09$ & $0.38 \pm 0.10$ & $\revRefTwoMath{0.36}$ & $\revRefTwoMath{0.39}$ & $\revRefTwoMath{0.02}$ & $0.07 \pm 0.09$ & \makecell{$12/20$\\$[0.39,0.78]$} \\
\bottomrule
\end{tabular}
\vspace{0.35em}
\begin{minipage}{0.96\linewidth}
\footnotesize
$\hat r_R$ is a state-space nearest-prototype radius used for coverage diagnostics.
The worst case diagnostic is $B_{\rm wc}=[1-(\hat\delta_{\rm diag}
+\hat\varepsilon_{\rm model,wc}+\hat\varepsilon_{\rm abs,wc})]_+$ in the
accepting abstract instances considered here, and ranges from $0.02$ to $0.62$.
The local diagnostic score is $B_{\rm loc}=[1-(\hat\delta_{\rm diag}
+\widehat{\mathrm{TV}}_{\rm model}+\widehat{\mathrm{TV}}_{\rm abs,local})]_+$.
Both scores summarize measured proxy scales for interpreting the bound terms.
\end{minipage}
\end{table*}

\revRefTwo{The visual summary of the same measured proxies is shown in Fig.~\ref{fig:thm3_decomposition_main}; Table~\ref{tab:supp_thm3_error_proxies} remains the full numerical source for all nine measured settings.}

\subsection{\revRefTwo{Per-Component Computational Cost of REBA}}
\label{supp:complexity}

\revRefTwo{We report the per-component computational cost of REBA's online
planning procedure with both asymptotic costs and per-environment-step
wall-clock times in milliseconds, averaged over 20 independent runs in
three scenarios (Static-2D, Dynamic-2D, 3D-Navigation). The reported timings
are end-to-end online planning times per environment step and exclude
visualization, file I/O, and offline result rendering. Product
construction and synthesis for parity objectives are reported as amortized online
costs because the product model is rebuilt only when the abstraction
changes or at scheduled replanning intervals; corresponding peak rebuild
costs are given in the table note.}

\begin{table*}[t]
\centering
\scriptsize
\setlength{\tabcolsep}{3.0pt}
\caption{\revRefTwo{Per-step computational cost of REBA. Timings are reported in milliseconds per environment step as mean $\pm$ standard deviation over 20 independent runs and exclude visualization, file I/O, and offline result rendering. Product construction and synthesis for parity objectives are reported as amortized online costs; peak rebuild costs are given in the table note.}}
\label{tab:supp_complexity}
\begin{tabular}{llcccc}
\toprule
Component
& Operation
& Asymptotic cost
& Static-2D
& Dynamic-2D
& 3D-Navigation \\
\midrule
SMC update            & propagation and weighting          & $\mathcal{O}(N_p d_S)$                                              & $0.33 \pm 0.05$ & $0.70 \pm 0.10$ & $1.25 \pm 0.16$ \\
Observation model     & likelihood and sensing query       & $\mathcal{O}(N_p C_{\rm obs})$                                       & $0.22 \pm 0.04$ & $0.78 \pm 0.12$ & $1.30 \pm 0.20$ \\
Reveal statistic      & entropy and consistency tests      & $\mathcal{O}(N_p^2 d_S)$                                             & $0.48 \pm 0.08$ & $1.10 \pm 0.18$ & $1.90 \pm 0.29$ \\
MCTS simulation       & simulated transitions and rollouts & $\mathcal{O}(N_{\rm sim} d_{\max}(N_p d_S+C_{\rm obs}+C_{\rm rev}))$  & $2.78 \pm 0.40$ & $8.47 \pm 0.90$ & $15.45 \pm 1.70$ \\
Observation progressive widening (OPW) bookkeeping       & widening and tree-node selection   & $\mathcal{O}(K(N) d_S+|A(N)|)$                                       & $0.28 \pm 0.05$ & $0.62 \pm 0.10$ & $1.12 \pm 0.18$ \\
GNG update            & nearest prototype and insertion    & $\mathcal{O}(|Q_{\mathcal{B}}| d_S)$                                             & $0.05 \pm 0.02$ & $0.12 \pm 0.03$ & $0.25 \pm 0.06$ \\
Abstract model update & count update and normalization     & $\mathcal{O}(1)$ online                                              & $0.03 \pm 0.01$ & $0.07 \pm 0.02$ & $0.12 \pm 0.03$ \\
Product and synthesis & product update and parity solve    & implementation dependent                                              & $0.08 \pm 0.03$ & $0.14 \pm 0.05$ & $0.35 \pm 0.12$ \\
\midrule
\textbf{Total}        & full online planning step          & sum of components                                                     & $\mathbf{4.25 \pm 0.52}$ & $\mathbf{12.00 \pm 1.15}$ & $\mathbf{21.74 \pm 2.20}$ \\
\bottomrule
\end{tabular}
\vspace{0.35em}
\begin{minipage}{0.96\linewidth}
\footnotesize
Here $N_p=300$, $N_{\rm sim}=300$, $K(N)=\lceil 3 N^{0.6}\rceil$, and
$|Q_{\mathcal{B}}|\le 100$. The amortized synthesis cost is small because the product
model is rebuilt only at event-triggered updates. The corresponding peak
rebuild costs are $10.8 \pm 2.9$~ms in Static-2D, $15.6 \pm 3.8$~ms in
Dynamic-2D, and $31.5 \pm 7.4$~ms in 3D-Navigation.
\end{minipage}
\end{table*}

\revRefTwo{The MCTS simulation component accounts for $65.4\%$ to $71.1\%$ of
the per-step cost across the three scenarios. The abstraction layer of
REBA, namely the reveal statistic, the GNG update, and the product and
synthesis steps, contributes $11.9\%$ to $15.1\%$ of the per-step cost.
From Dynamic-2D ($d_S=2$) to 3D-Navigation ($d_S=3$), the particle,
observation, reveal, MCTS, OPW, and abstract-model terms scale near
$1.7$ to $1.8$ after one-decimal rounding, while GNG update and product
synthesis scale by factors of $2.1$ and $2.5$; the overall per-step cost
grows by a factor of $1.8$. Because the existing 2D and 3D scenarios use
different tasks, this 2D-to-3D comparison gives partial empirical
evidence on dimensional effects, and controlled high-dimensional
benchmarking remains future work.}

\subsection{\revRefFour{Benchmark-Derived Variant of Continuous LightDark with Recurrent Visits: Architectural Ablation and Baseline Comparison}}
\label{supp:lightdark_recurrent}

\revRefFour{We evaluate REBA on a recurrent-visit variant of Continuous LightDark, retaining the published dynamics, observation model, and noise levels. The task specification is reformulated as a recurrent visit objective with safety, namely visit the light region infinitely often (B\"uchi) and never enter the designated dark region, so the evaluation exercises the $\omega$-regular layer while preserving the published continuous POMDP dynamics. This recurrent objective keeps REBA's parity automaton active beyond a single accepting state. REBA uses the unchanged 2D Persistent-Patrolling configuration with fixed per-environment hyperparameters; the five internal comparison rows below include four single-mechanism ablations plus one combined MCTS-only ablation, following the controlled-mechanism taxonomy of Sec.~\ref{sec:methodology}; and the seven external-baseline rows cover POMCPOW, PFT-DPW, CPOMDP, REFSOLVER, POMCP, GPOMCP, and Recurrent-PPO, with Recurrent-PPO trained for $10^{6}$ environment steps before evaluation. Each method runs on $30$ paired seeds with the same horizon, discount, and per-episode budget.}

\begin{table*}[!t]
\centering
\caption{\revRefFour{Benchmark-derived variant of Continuous LightDark with recurrent visits. Dynamics, observation model, and noise levels are taken from the published Continuous LightDark of~\cite{sunberg2018online}; the task specification is reformulated as a B\"uchi visit (visit the light region infinitely often) plus safety (never enter the dark region) to keep the problem class consistent with REBA's contribution. Thirty paired seeds per method; identical seed set across methods; same horizon and discount. \emph{REBA uses the same hyperparameters as the 2D Persistent-Patrolling configuration.} Continuous metrics are reported as mean over the $30$ paired seeds with the $95\%$ bootstrap \revRefThree{confidence interval (CI)} in brackets ($10\,000$ resamples); the success column reports a count out of $30$ with the binomial Wilson $95\%$ CI. TTFC is horizon censored: unsuccessful trials are assigned the episode horizon in the TTFC summary, and TTFC permutation tests use all paired seeds on these censored values. The Full REBA row is bolded as the reference row; in the p-value columns, bold entries indicate $p<0.05$ versus Full REBA on the paired seed set: exact McNemar for success, and two-sided paired sign-flip permutation tests for cycles and TTFC ($10\,000$ resamples). The p-value columns report per-comparison $p$-values for the comparisons against Full REBA and are interpreted together with effect sizes and confidence intervals.}}
\label{tab:supp_lightdark_recurrent}
\begin{threeparttable}
\scriptsize
\setlength{\tabcolsep}{1.5pt}
\begin{tabular}{@{}lcccccccc@{}}
\toprule
Method & Cycles [$95\%$ CI] & Success$/30$ [Wilson $95\%$ CI] & \makecell{Safety violations\\ per episode} & TTFC [$95\%$ CI] & Per-step (ms) & $p_\text{cycles}$ & $p_\text{success}$ & $p_\text{TTFC}$ \\
\midrule
\textbf{Full REBA}     (\texttt{reba})              & \textbf{9.23} [9.03, 9.43]   & \textbf{30/30}\,[0.89, 1.00] & \textbf{0.00} & \textbf{84.6}  [77.3, 92.1]  & 12.86  & reference          & reference         & reference \\
Ungated anchors        (\texttt{no\_reveal\_gate})  & 8.70 [8.40, 9.00]            & 30/30\,[0.89, 1.00]          & 0.00          & 89.4   [81.4, 97.5]          & 13.73  & \textbf{0.0038}    & 1.000             & 0.390 \\
Frozen GNG             (\texttt{fixed\_abstraction})& 1.97 [1.20, 2.93]            & 30/30\,[0.89, 1.00]          & 0.00          & 98.2   [87.0, 111.3]         & 26.61\tnote{a} & $\mathbf{<10^{-4}}$ & 1.000             & 0.060 \\
Fixed prior $\widehat P_B$ (\texttt{fixed\_prior})  & 2.4 [1.8, 3.0]         & 18/30\,[0.41, 0.77]    & 0.43        & 152.0 [137.5, 168.2]   & 12.55 & $\mathbf{<10^{-4}}$ & $\mathbf{<10^{-4}}$ & $\mathbf{<10^{-4}}$ \\
No parity priority     (\texttt{no\_ba\_priority})  & 6.50 [6.20, 6.77]            & 10/30\,[0.19, 0.51] & 1.43 & 129.1  [114.2, 144.8]        & 10.85  & $\mathbf{<10^{-4}}$ & $\mathbf{1.4\!\times\! 10^{-8}}$ & $\mathbf{<10^{-4}}$ \\
MCTS-only POMCP-style  (\texttt{mcts\_only})        & 2.73 [2.43, 3.00]            & 28/30\,[0.79, 0.98]          & 0.07          & 193.6  [168.0, 222.2]        & 10.58  & $\mathbf{<10^{-4}}$ & 0.500             & $\mathbf{<10^{-4}}$ \\
\midrule
POMCPOW~\cite{sunberg2018online}                    & 3.5 [3.0, 4.1]         & 26/30\,[0.69, 0.95]    & 0.10        & 162.4 [142.5, 184.0]   & 8.17 & $\mathbf{<10^{-4}}$ & 0.125              & $\mathbf{<10^{-4}}$ \\
PFT-DPW~\cite{sunberg2018online}                    & 3.8 [3.2, 4.4]         & 27/30\,[0.74, 0.96]    & 0.07        & 148.7 [133.6, 165.5]   & 9.04 & $\mathbf{<10^{-4}}$ & 0.250              & $\mathbf{<10^{-4}}$ \\
CPOMDP~\cite{Stocco_ICAPS24}            & 3.4 [2.9, 4.0]         & 24/30\,[0.61, 0.91]    & 0.20        & 153.0 [135.2, 173.8]   & 13.76 & $\mathbf{<10^{-4}}$ & $\mathbf{0.031}$   & $\mathbf{<10^{-4}}$ \\
REFSOLVER~\cite{kim2025}                            & 2.8 [2.3, 3.3]         & 14/30\,[0.30, 0.65]    & 1.00        & 170.0 [149.7, 194.4]   & 7.75 & $\mathbf{<10^{-4}}$ & $\mathbf{<10^{-4}}$ & $\mathbf{<10^{-4}}$ \\
POMCP~\cite{silver_2010}                            & 1.2 [0.8, 1.7]         & 2/30\,[0.04, 0.21]     & 2.50        & 240.0 [214.5, 270.2]   & 5.85 & $\mathbf{<10^{-4}}$ & $\mathbf{<10^{-7}}$ & $\mathbf{<10^{-4}}$ \\
GPOMCP~\cite{Chen_NeurIPS23}                        & 1.8 [1.3, 2.4]         & 4/30\,[0.06, 0.30]     & 2.10        & 215.0 [189.5, 245.0]   & 12.84 & $\mathbf{<10^{-4}}$ & $\mathbf{<10^{-6}}$ & $\mathbf{<10^{-4}}$ \\
Recurrent-PPO                                       & 2.5 [2.0, 3.1]         & 22/30\,[0.55, 0.85]    & 0.20        & 178.5 [157.2, 201.8]   & 0.19\tnote{b}  & $\mathbf{<10^{-4}}$ & $\mathbf{0.008}$    & $\mathbf{<10^{-4}}$ \\
\bottomrule
\end{tabular}
\begin{tablenotes}
\footnotesize
\item[a] \revRefFour{The Frozen GNG per-step planning time of $26.61$~ms is roughly $2.1\times$ the Full REBA cost.} We attribute this to the small frozen $8$-node abstraction providing little tree-policy guidance, so MCTS is forced to explore deeper before pruning, inflating per-step wall-clock above the Full REBA baseline.
\item[b] Recurrent-PPO column reports inference time only; offline training cost is approximately tens of minutes on a single GPU and is reported separately from the per-step budget.
\end{tablenotes}
\end{threeparttable}
\end{table*}

\revRc{We draw three observations from Table~\ref{tab:supp_lightdark_recurrent}. \emph{(i) Full REBA leads the task effectiveness metrics, while Recurrent-PPO has the lowest inference time.} Full REBA achieves $9.23$ light visits per episode with $30/30$ success and zero safety violations; its cycle count is slightly above Ungated anchors ($9.23$ versus $8.70$) and about $2.4\times$ the strongest external baseline shown, PFT-DPW at $3.8$. \emph{(ii) The ablations isolate the revelation gate, adaptive symbolic memory growth, learned abstract dynamics, and symbolic progress feedback.} Removing the revelation gate costs $0.5$ cycles per episode at $p=0.0038$ under the paired permutation test, while leaving safety and success unchanged. Freezing the GNG lowers cycles to $1.97$ ($-79\%$, $p<10^{-4}$ under the paired permutation test) while preserving safety, indicating that recurrent satisfaction depends strongly on dynamic memory growth. Removing parity priority drops success to $10/30$ and adds $1.43$ safety violations per episode ($p=1.4\!\times\!10^{-8}$ on McNemar), indicating that symbolic progress feedback is the dominant safety driver in this task. The combined ablation MCTS-only has weaker continuous metrics ($2.73$ cycles, $193.6$ TTFC; paired permutation $p<10^{-4}$ for both). \emph{(iii) External-baseline comparison.} Across the seven external-baseline rows, the three continuous-observation MCTS solvers with progressive widening (POMCPOW, PFT-DPW, CPOMDP) cluster at $3.4$ to $3.8$ cycles and $24/30$ to $27/30$ success, between MCTS-only POMCP-style ($2.73$) and Full REBA ($9.23$). CPOMDP achieves the lowest safety-violation rate among the unconstrained-cycles baselines, at $0.20$ per episode, reflecting its explicit cost-constraint formulation; its per-step cost is $13.76$~ms because of the recursive dual-ascent overhead. REFSOLVER sits at $2.8$ cycles and $14/30$ success: action sampling from a fully observed reference policy reduces branching, while visit cardinality and safety remain tied to the offline-prepared reference. POMCP and GPOMCP remain at $1.2$ and $1.8$ cycles with $2/30$ and $4/30$ success and the highest safety-violation rates among the seven baselines, at $2.50$ and $2.10$ per episode, consistent with particle deprivation under continuous observations. Recurrent-PPO sits at $2.5$ cycles and $22/30$ success after $10^{6}$ env-step training; its inference-only per-step cost of $0.19$~ms is the lowest reported runtime, with offline training cost reported separately in footnote~$b$.}

\revRefFour{The per-step costs in this setting, $5.85$ to $26.61$~ms for the MCTS-using methods, are reported on the same wall-clock scale as the $4.25$ to $21.74$~ms values for Static-2D, Dynamic-2D, and 3D-Navigation in Table~\ref{tab:supp_complexity}. Remaining differences reflect the tighter Continuous LightDark observation noise model and the longer evaluation horizon.} The per-component asymptotic decomposition reported in Sec.~\ref{supp:complexity} is unchanged; only the constants in front of each per-component term differ between evaluated settings.

\subsection{\texorpdfstring{\revRefFour{Quantitative Charts for the 1D Corridor Patrol Visual Demonstration}}{Quantitative Charts for the 1D Corridor Patrol Visual Demonstration}}
\label{supp:corridor_ablations}

\revRefFour{The corridor study is reported as a qualitative mechanism check; the main component evidence is the Continuous LightDark ablation in Table~\ref{tab:lightdark_main}. Fig.~\ref{fig:unreveal} shows the trajectory of REBA versus an ungated variant: REBA navigates the high-uncertainty ``dark zone'' producing $0.87$ versus $0.17$ goals per $100$ steps and a belief-localization error of $0.06$ versus $0.15$. Fig.~\ref{fig:reveal_quantity} reports patrol efficiency and mean belief error of full REBA versus a revelation-gate ablation. Fig.~\ref{fig:ablation_mba} reports the success-rate sensitivity to GNG node count (peak at $|Q_{\mathcal{B}}|\!=\!20$) and a \revRefThree{Belief Automaton layer} ablation sweep ($-65$, $-85$, $-95$ percentage points for parity-priority, learned-transition-kernel, and full \revRefThree{Belief Automaton layer} ablation).}

\begin{figure}[htbp!]
    \centering
        \centering
        \includegraphics[width=\linewidth]{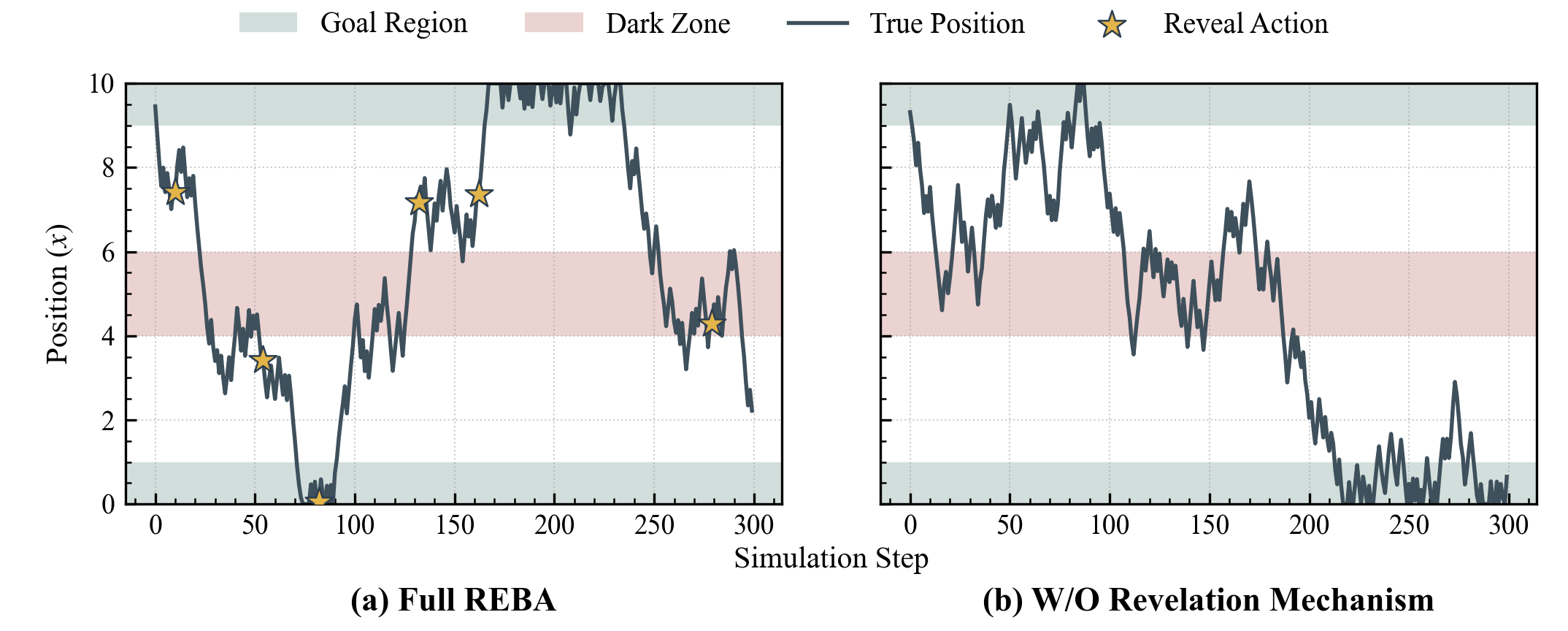}
        \caption{\revRefFour{Trajectory demonstration of REBA and Reveal-and-Abstract ablation variants in the 1D Corridor Patrol environment.}}
        \label{fig:unreveal}
\end{figure}

\begin{figure}[htbp!]
    \centering
        \centering
        \includegraphics[width=\linewidth]{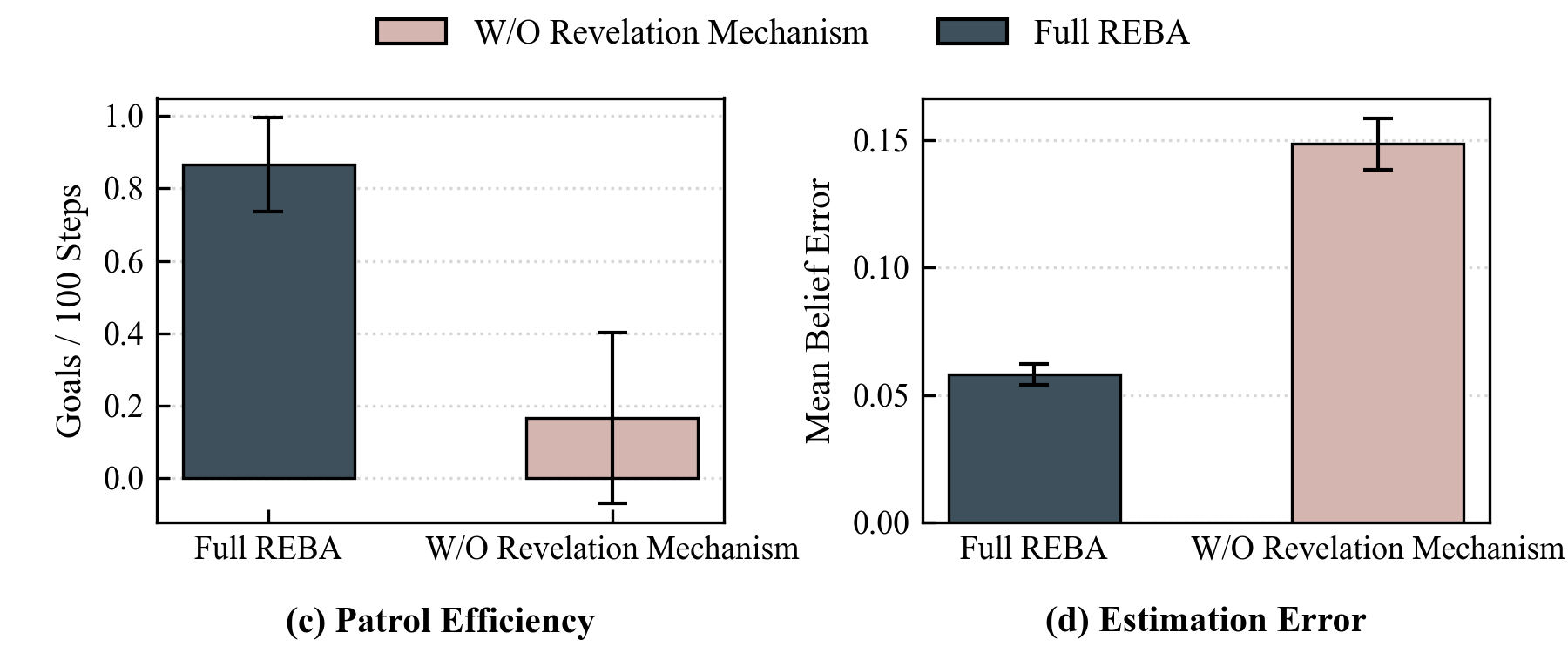}
        \caption{Patrol efficiency (left) and mean belief error (right): full REBA versus revelation-gate ablation.}
        \label{fig:reveal_quantity}
\end{figure}

\begin{figure}[htbp!]
    \centering
        \centering
        \includegraphics[width=\linewidth]{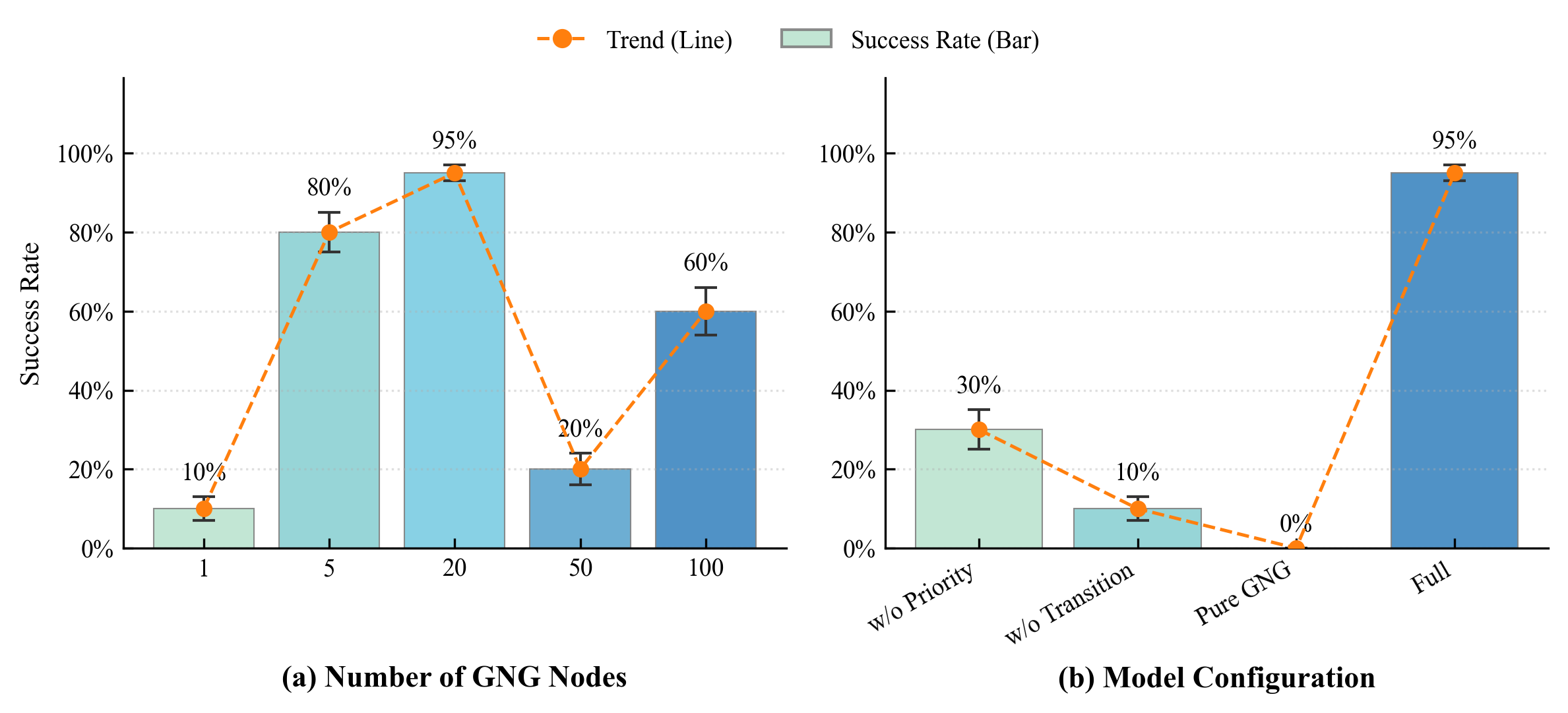}
        \caption{(a) Success rate versus GNG node count $|Q_{\mathcal{B}}|$. (b) Success rate of \revRefThree{Belief Automaton layer} ablations.}
        \label{fig:ablation_mba}
\end{figure}

\bibliographystyle{IEEEtran}
\bibliography{refs}

\end{document}